%% file: paper.tex
\documentclass{article}

\PassOptionsToPackage{numbers, sort&compress}{natbib}
\usepackage[final]{ewrl_2023}

\usepackage{algorithm,algorithmic}
\usepackage[utf8]{inputenc} 
\usepackage[T1]{fontenc}    
\usepackage{hyperref}       
\usepackage{url}            
\usepackage{booktabs}       
\usepackage{amsfonts}       
\usepackage{nicefrac}       
\usepackage{microtype}      
\usepackage{xcolor}         

\usepackage[english]{babel}

\usepackage[textsize=tiny]{todonotes}

\usepackage[normalem]{ulem}

\title{Reinforcement Learning in near-continuous time for continuous state-action spaces}

\author{%
        Lorenzo Croissant\textsuperscript{1,~2}\hspace{4em} Marc Abeille\textsuperscript{2}\hspace{4em} Bruno Bouchard\textsuperscript{1} \vspace{.35em}\\
        \textsuperscript{1}CEREMADE, Université Paris Dauphine-PSL, CNRS \hspace{2em} \textsuperscript{2}Criteo AI Lab\vspace{.4em}\\
        \texttt{\{croissant,bouchard\}@ceremade.dauphine.fr}\\
        \texttt{m.abeille@criteo.com}
}

\input{ICML_preamble.tex}

\input{commands.tex}

\newcommand{\mypar}[1]{{\textbf{#1.}}}

\begin{document}
\maketitle
\begin{abstract}
        We consider the Reinforcement Learning problem of controlling an unknown dynamical system to maximise the long-term average reward along a single trajectory. 
        Most of the literature considers system interactions that occur in discrete time and discrete state-action spaces. 
        Although this standpoint is suitable for games, it is often inadequate for mechanical or digital systems in which interactions occur at a high frequency, if not in continuous time, and whose state spaces are large if not inherently continuous. 
        Perhaps the only exception is the Linear Quadratic framework for which results exist both in discrete and continuous time.  
        However, its ability to handle continuous states comes with the drawback of a rigid dynamic and reward structure.
        This work aims to overcome these shortcomings by modelling interaction times with a Poisson clock of frequency $\varepsilon^{-1}$, which captures arbitrary time scales: from discrete ($\varepsilon=1$) to continuous time ($\varepsilon\downarrow0$). 
        In addition, we consider a generic reward function and model the state dynamics according to a jump process with an arbitrary transition kernel on $\mathbb{R}^d$. 
        We show that the celebrated optimism protocol applies when the sub-tasks (learning and planning) can be performed effectively. 
        We tackle learning within the eluder dimension framework and propose an approximate planning method based on a diffusive limit approximation of the jump process.
        Overall, our algorithm enjoys a regret of order $\tilde{\mathcal{O}}(\varepsilon^{1/2} T+\sqrt{T})$. 
        As the frequency of interactions blows up, the approximation error $\varepsilon^{1/2} T$ vanishes, showing that $\tilde{\mathcal{O}}(\sqrt{T})$ is attainable in near-continuous time.
\end{abstract}

\section{Introduction}

Controlling a dynamical system to drive it to optimal long-term average behaviour is a key challenge in many applications, ranging from mechanical engineering to econometrics. Reinforcement Learning (RL) aims to do so when the system is a priori unknown by tackling jointly both the control and the statistical inference of the system. This joint objective is even more important in the online version of the problem, in which one interacts with the system along a single trajectory (no resets or episodes). In the last decades, the insights of Bandit Theory (see e.g. \cite{lattimore2020bandit}) have been leveraged to tackle the RL problem, while addressing the inherent exploration-exploitation dilemma that naturally arises in sequential decision-making (see e.g. \cite[\S~4.2]{szepesvari2010algorithms}).
However, most literature considers interactions that occur in discrete time, which is not always applicable when events are triggered by a digital system. Such systems are pervasive in finance and advertising, for instance, and typically have interactions occurring at a very high frequency, with each interaction having only a marginal impact on the state of the system.

\mypar{Near-continuous time, continuous state-space} A natural approach to plan in such systems is to directly model the problem in continuous time. This is the common approach in finance, see for instance \cite{chen2001fundamentals, cont_financial_2004, obizhaeva2013optimal}.
However, the continuous-time approach conflicts with the sample-based nature of statistical learning theory that fundamentally takes place in discrete time. 
As such, learning requires careful modelling of the data-generating process and its arrival times.
We consider interactions governed by a Poisson clock, setting the expected inter-arrival time of the clock to a parameter $\ve\in(0,1)$. This allows us to model a continuum of situations: from discrete time $\ve=1$, to continuous time $\ve\downarrow0$. We are interested in the regime in which $\ve\ll 1$.

Concurrently, a prerequisite for real-world applicability is the ability to model complex dynamics and rich reward signals for continuous state variables. With this in mind, we focus on the model-based approach where the transition and the reward function belong to a parameterised class of functions operating on a continuous state-action space. This level of generality poses challenges regarding all three key sub-tasks of RL: which are planning, learning, and the explore-exploit trade-off.

\mypar{Discrete and continuous control} For discrete-time dynamics on finite state-action spaces, the planning problem falls under the umbrella of Markov Decision Processes (MDPs) which have been extensively reviewed in \cite{puterman_markov_2005}. The finite nature of MDPs is at the heart of their theoretical and computational success. Their extension to countable or even continuous state spaces is, however, non-trivial; see e.g. \cite[\S~4.6,~p.245]{bertsekas2011dynamic} for a review of the challenges. Perhaps the only exception which retains those nice theoretical and computational properties is the celebrated Linear Quadratic (LQ) framework \cite{kalman}. However, both frameworks are limited in their expressive power.
In contrast, the continuous-time theory of Stochastic Control has demonstrated how to effectively solve the control problem for arbitrary regular dynamics on continuous state-spaces. It enjoys a rich and mature literature \cite{arisawa1998ergodic,arapostathis2012ergodic,lions1983optimal}, both on the theoretical aspects as well as numerical solvers based on Partial Differential Equations (PDEs), another storied field \cite{kushner_numerical_2001,barles1991convergence,bonnans_consistency_2003}.
The near-continuous time framework lies between the two theories, and recent results of \cite{abeille_diffusive_2022} show how to navigate between them and approximately solve the planning problem in the high-frequency interactions regime by solving its diffusive counterpart.

\mypar{Learning non-linear systems} 
Similar to the planning problem, the natural way to move beyond finite Markov chain models and towards continuous state dynamics is through linear models.  The least-squares estimator enjoys strong theoretical guarantees including adaptive confidence sets that can be efficiently maintained online, see e.g. \cite{abbasi-yadkori_improved_2011}.
Extensions \cite{russo2013eluder,osband2014model} showed how to extend this approach to richer model classes through the use of Non-Linear Least Squares (NLLS). This framework subsumes standard least squares and has been successful in many dynamics by retaining its key properties regarding confidence sets. While providing a protocol for learning with NLLS, \citeauthor{russo2013eluder} characterised, in \cite{russo2013eluder}, the trade-off between the richness of the model and the hardness of its learning through two quantities of the model class: the log-covering number, and the eluder dimension which summarises the difficulty of turning the information from data into predictive power.


\mypar{Optimistic exploration} Optimism in the Face of Uncertainty (OFU) has proven highly successful in sequential decision-making from bandits to RL.
The works of \cite{jaksch_near-optimal_2010,auer_logarithmic_2006,bartlett_REGAL_2009} showed how to extend the celebrated UCB \cite{auer2002finite} algorithm from bandits to finite MDPs; later, extensions were made to continuous state in the LQ setting, see e.g. \cite{pmlr-v19-abbasi-yadkori11a,pmlr-v119-abeille20a,cohen2019learning} and references therein.
Extension from bandit to MDP and then to LQ raised new challenges that persist in our setting. First, the agent should not revise its behaviour too often to prevent dithering, which requires the design of a lazy update scheme. Second, generic continuous states-spaces models come with inherent unboundedness, and one must carefully address stability issues.

In this work, we consider the near-continuous time system interaction model and propose an optimistic algorithm for online reinforcement learning in the average reward setting\footnote{Also known as, \emph{average cost per stage}, \emph{long-run average}, or \emph{ergodic} setting.}. Our approach builds on the work of \cite{abeille_diffusive_2022} and the connection to the diffusive regime to address the planning sub-task, yielding $\ve^{1/2}$-optimal policies. Furthermore, we perform the learning with NLLS extending the work of \cite{russo2013eluder} to our near-continuous time and unbounded state setting. Underlying the extension of both these two approaches is a careful treatment of the state boundedness which we do with Lyapunov stability arguments. Overall, our algorithm enjoys near-optimal performance as its regret scales with $\tilde\Oc(\ve^{1/2} T + \sqrt{T})$. As the frequency of interactions increases ($\ve\downarrow 0$) the approximation error vanishes, showing that $\tilde\Oc(\sqrt{T})$ is attainable in near-continuous time.

\section{Setting}\label{sec: prelim} 


We consider an agent interacting with its environment to maximise a long-term average reward. At each interaction, it observes the current state of the system $x\in\RR^d$, takes action $a\in\Ab\subset\RR^{d_\Ab}$, and receives reward $r(x,a)$, for $r:\RR^d\x\Ab\to\RR$. The system then transitions to the state $x'$ according to 
\[x'=x+\mu_{\theta^*}(x,a)+\Sigma\xi \quad \mbox{ with } \quad \xi\sim \Nc(0,\idmat_d),\]
$\Sigma\in\RR^{d \x d}$, and in which $\mu_{\theta^*}:\RR^d\x\Ab\to\RR^d$ is the deterministic motion of the system\footnote{While the additive noise structure is a design choice that simplifies the analysis, the choice of parameterising the drift as $x + \mu_{\theta^*}(x,a)$ instead of $\mu_{\theta^*}(x,a)$ does not affect its generality and is made only for convenience.}. Contrasting with the standard setting, we consider here the interactions to occur in a random fashion, which we model by an independent Poisson process of intensity $\ve^{-1}$. As such, $\ve$ parameterises the mean wait time between events and gives us a direct control on the frequency of interactions.

\mypar{State dynamics} Let $\Omega:=\Db$ be the space of {\itshape c\`adl\`ag} functions from $[0,+\infty)$ to $\RR^d$, and let $\PP$ be a probability measure on $\Omega$. We formalise the interaction time and the noise process as a marked $\PP$-compound Poisson process $(N_t)_{t\in\RR_+}$ of intensity $\ve^{-1}\ge 1$. We denote by $(\tau_n)_{n\in\NN}$ its arrival (interaction) times, with $\tau_0:=0$, and by $(\xi_{n})_{n\in\NN}$ its marks, which are independent of everything else and drawn i.i.d. according to the centred standard Gaussian measure $\nu$ on $\RR^d$.
We encode the information available at time $t\in\RR_+$ in the $\sigma$-algebra $\Fc_t:= \sigma((\tau_n,\xi_n)_{\tau_n\le t})$ and with the filtration $\Fb$ defined as the completion of $(\Fc_t)_{t\in\RR_+}$.
Let $\Ac$ be the set of $\Fb$-adapted $\Ab$-valued processes, referred to as \textit{controls}. For any initial state $x_0\in\RR^d$ and $\alpha\in\Ac$, we let $X^{\alpha,\theta^*}$ denote the pathwise-unique solution of
\begin{align}
    \begin{cases}
        X_{\tau_{n}}^{\alpha,\theta^*}=X_{\tau_{n-1}}^{\alpha,\theta^*}+\mu_{\theta^*}(X_{\tau_{n-1}}^{\alpha,\theta^*},\alpha_{\tau_{n-1}}) + \Sigma\xi_{n}\\
        X_{\tau_0}^{\alpha,\theta^*}=x_0
    \end{cases}\,.
    \label{eq: intro def process}
\end{align}
In \eqref{eq: intro def process}, we model the dynamic according to a jump process and $X^{\alpha,\theta^*}$ is then defined at any time $t\in\RR_+$ by considering that it is piece-wise constant on each interval $[\tau_{n-1},\tau_{n})$, $n\in\NN^*$. Although involved, this definition allows us to define the state process at any time and feature the interplay of the Poisson and wall-time clocks. 

\mypar{Reinforcement learning problem} In our model based paradigm, ignorance about the system is condensed to a single parameter set $\Theta\subset\RR^{d_\Theta}$ containing the unknown nominal parameter $\theta^*$. To single out the RL challenges, we further assume that $\theta^*$ only affects the drift assuming other quantities (i.e.\ $\Sigma$, $\ve$, and $r$) are known to the agent. 
For any $x_0\in\RR^d$, we evaluate the performance of any strategy $\alpha\in\Ac$ with the long-term average reward criterion defined by
\begin{align}
    \rho^\alpha_{\theta^*}(x_0):=\liminf_{T\to\infty}\frac1{T}\EE\left[\sum_{n=\deb}^{N_T} r(X_{\tau_n}^{\alpha,\theta^*},\alpha_{\tau_n})\right]\,.
    \label{eq: intro def rho}
\end{align}
The goal of the agent is to accumulate as much reward as possible, i.e. to compete with the best an omniscient agent can achieve: $\rho^*_{\theta^*}(x_0):= \sup_{\alpha\in\Ac}\rho^\alpha_{\theta^*}(x_0)$. We evaluate the quality of a learning algorithm generating $\alpha$ according to its regret.
\begin{definition}\label{def: regret}
    For any $T\in\RR_+$, $x_0\in\RR^d$, and $\alpha\in\Ac$, the regret of $\alpha$ is
    \begin{align}
    \Rc_T(\alpha):= T\rho^*_{\theta^*}(x_0) - \sum_{n=\deb}^{N_T}r(X_{\tau_n}^{\alpha,\theta^*},\alpha_{\tau_n})\,.\label{eq: def regret}
    \end{align}
\end{definition}
Noticing that $N_T$ is the number of events up to time $T$, the definitions of the optimal performance \eqref{eq: intro def rho} and the regret \eqref{eq: def regret} highlight the interplay between the wall-clock ($T$) and Poisson clock ($N_T$). The agent's realised trajectory uses the Poisson clock, which governs interactions, while the ideal performance is understood per unit of wall-clock time.

\subsection{Working Assumptions}\label{subsec:asmps}

Of particular interest in our approach is the high-frequency regime in which $\ve\downarrow0$. In this framework, many interactions occur per unit of time, each of which is of negligible impact both in terms of dynamics and reward. This regime can be encoded by introducing, for any parameter $\theta\in\Theta$, rescaled coefficients $(\bar\mu_{\theta},\bar\Sigma,\bar r)$ connected to the original parametrisation by 
\begin{align*}
    \mu_{\theta}=\ve\bar\mu_{\theta}\,,\quad\Sigma=\ve^{\frac12}\bar\Sigma\,,\quad \mbox{ and } r=\ve\bar r\,.
\end{align*}
In this rescaled parametrisation, $\bar\mu_\theta$, $\bar\Sigma$, and $\bar r$ are understood as independent of $\ve$. To improve legibility, we will use both representations $(\mu_\theta,\Sigma,r)$ and $(\bar\mu_\theta, \bar\Sigma,\bar r)$.
While the scaling of $\mu_\theta$ and $r$
in $\ve$  arises naturally, the one of $\Sigma$ is a design choice: we consider the covariance $\Sigma\Sigma^\top$ to be linear in $\ve$. Known as the diffusive regime, this preserves stochasticity\footnote{Another common, but more rigid, regime is to consider $\Sigma=\ve\bar\Sigma$, whose limit regime is deterministic and known as the fluid limit, see \cite{fernandez-tapia_optimal_2016}.} as $\ve\downarrow0$. 

We now impose regularity assumptions on the drift and reward signal, uniformly over the possible parametrisations and controls $(\alpha,\theta)\in\Ac\x\Theta$. We take $\snorm{\cdot}$ to be the Euclidian norm on $\RR^d$ and $\snorm{\cdot}_{\op}$ for the operator norm on $\RR^{d\x d}$ associated to $\snorm{\cdot}$.

\begin{restatable}{assumption}{AsmpBasics}\label{asmp: basics}
    The map $(\bar\mu,\bar r)$ is continuous, and there is $L_0>0$ such that for all $(\theta,a)\in\Theta\x\Ab$
    \begin{align*}
        L_0\!>\!&\sup_{x\in\RR^d}\frac{\norm{\bar\mu_\theta(x,a)}}{1+\norm{x}} +\sup_{\substack{x\neq x'}}
        \frac{\norm{\bar\mu_\theta(x,a)-\bar\mu_\theta(x'\!,a)}}{\norm{x-x'}}
        +\sup_{x\in\RR^d} \norm{\bar r(x,a)} + \sup_{\substack{x\neq x'}}\frac{\norm{\bar r(x,a)-\bar r(x'\!,a)}}{\norm{x-x'}}.
    \end{align*}
     Furthermore, $L_0>\snorm{\bar\Sigma}_\op$ and $\bar\Sigma\bar\Sigma^\top\succeq \varsigma \idmat_d$ for some $\varsigma>0$, where $\succeq$ denotes the Loewner order.
\end{restatable}
\Cref{asmp: basics} mainly imposes regularity on both $\bar\mu_\theta$ and $\bar r$ through a Lipschitz condition. We also assume rewards to be bounded, which may be relaxed, but doing so is highly technical and involves trading off the growth of $r$ with the stability of the process (see \cref{asmp: both asmp joint}). Note that we do not assume boundedness of $\bar{\mu}_\theta$. Finally, we assume non-degeneracy of the noise by requiring $\bar\Sigma$ to be full rank.

We conclude with \cref{asmp: both asmp joint} to ensure the stability of the state process. Let $\RR_*^d:=\RR^d\setminus\{0\}$ and $\RR_+:=(0,+\infty)$. For $k\in\NN$, let $\Cc^k(\RR_*^d;\RR_+)$ denote the set of $k$-times continuously differentiable functions from $\RR^d_*$ to $\RR_+$. Let $\nabla$ and $\nabla^2$ denote the gradient and Hessian operator respectively.
\begin{restatable}{assumption}{AsmpJoint}\label{asmp: both asmp joint}
    There is $(\ell_\Vs,L_\Vs,\cf_\Vs,M_\Vs,M_\Vs')\in\RR_+^5$  and a Lyapunov function $\Vs\in\Cc^{2}(\RR^d_*;\RR_+)$ satisfying, for any $(x,x',a,\theta)\in\RR^d\x\RR^d\x\Ab\x\Theta$, $x\neq x'$, and $\ve \in (0,1)$:
\begin{align}
    &\text{({\rm{i}.})}\qquad\qquad\quad \ell_\Vs\norm{x-x'}\le \Vs(x-x')\le L_\Vs\norm{x-x'}\,,\notag\\
    &\text{({\rm{ii}.})}\quad\quad \sup_{x\in\RR^d_*}\norm{\nabla \Vs(x)}\le M_\Vs \mbox{ and } \sup_{x\in\RR^d_*}\norm{\nabla^2\Vs(x)}_\op\le M_\Vs'\,,\notag\\
    &\text{({\rm{iii}.})}\quad \Vs(x+ \ve\bar\mu(x,a)-x'-\ve\bar\mu(x',a))\le (1-\ve \cf_\Vs)\Vs(x-x')\,. \label{eq:lyapunov asmp joint on jump problem}
\end{align}
\end{restatable}
\cref{asmp: both asmp joint} is a Lyapunov-like condition through the function $\Vs$. The condition {({\rm{i}.})} requires that $\Vs$ behaves similarly to a norm, while {({\rm{ii}.})} asks that $\Vs$ be smoothly differentiable everywhere but at $0$ and {({\rm{iii}.})} imposes a contraction condition on the drifts.

\mypar{Connection to linear stability} Stability theory has been extensively studied the special case of linear dynamics. In this case, we recover \cref{asmp: both asmp joint} from the Continuous Algebraic Riccati Equation (CARE; see e.g. \cite[\S~4.4]{lancaster1995algebraic}). Considering linear dynamics $\bar\mu_\theta(x,a)= \bar A x+\bar B a$ (given matrices $(\bar A,\bar B)$ of appropriate dimensions), continuous stability is guaranteed when the eigenvalues of $\bar A$ have negative real-part or, equivalently, by 
 the existence of a positive semi-definite matrix $P$ solving the CARE $\bar A^\top P + P \bar A = -\idmat_d$. For this $P$, its associated norm $\Vs=\norm{\cdot}_P$ is the appropriate Lyapunov function for \cref{asmp: both asmp joint}. Indeed, conditions {({\rm{i}.})} and {({\rm{ii}.})} follow as $\Vs$ is a norm and, for $\ve \le 1/2\lambda_{{\rm max}}(P)$, we have
 \begin{align*}
     \Vs(x+ \ve\bar\mu(x,a)-x'-\ve\bar\mu(x',a))^2  &= (x - x')^{\top} ( P + \ve \bar A^{\top} P + \ve P \bar A + \ve^2 P)(x - x') \\
     &= (x - x')^{\top} ( P  -\ve\idmat_d + \ve^2 P)(x - x') \\
     &\leq (x - x')^{\top} (P -  \ve P / \lambda_{\max}(P) + \ve^2 P)(x - x') \\
     &\leq (1 - \ve /2\lambda_{\max}(P)) \Vs(x-x')^2\,.
 \end{align*}
 Taking the square-root and using $\sqrt{1 - \ve /2\lambda_{\max}(P)} \leq 1 - \ve /4\lambda_{\max}(P)$ leads to {({\rm{iii}.})} with $\cf_\Vs = 1/4 \lambda_{\max}(P)$.

\section{Main results}\label{sec: challenges}

Our main contribution is a demonstration of the OFU protocol in the near-continuous time continuous state-action RL problem. 
The ingredients of OFU are: learning from accumulated data to design confidence sets; lazy updates to trade off policy revision and learning guarantees; and planning amongst plausible parameterisations. We summarise this protocol in \cref{alg: RL 1}. 

\begin{algorithm}[htpb]
    \caption{\texttt{OFU-Diffusion} } 
    \label{alg: RL 1}
 \begin{algorithmic}
    \STATE {\bfseries Input:} confidence level $\delta$, initial state $x_0$, initial control $\varpi_0$
    \FOR{$n\in\NN^*$}
    \STATE At time $\tau_n$, receive $r(X_{\tau_{{n-1}}}^{\varpi,\theta^*},\varpi_{\tau_{n-1}})$ and $X_{\tau_{n}}^{\varpi,\theta^*}$.
    \IF{$n$ satisfies \eqref{eq: def lazsy update}}
        \STATE $n_{k}\gets n$, $k\gets k+1$,
        \STATE Compute $\hat\theta_{n_k}$ using \eqref{eq: def NLLS main text} and $\Cc_{n_k}(\delta/3)$ with \eqref{eq: def conf sets}.
        \STATE $\thetaopt_k\gets \argmax_{\theta\in\Cc_{n_k}(\delta/3)}\bar\rho^*_\theta$
        \STATE $\policy_k\gets \bar\pi_{\tilde\theta_k}^*$ using \eqref{eq: argsup diffusive}
    \ENDIF
    \STATE Play $\varpi_{\tau_n}:=\policy_k(X_{\tau_n}^{\varpi,\theta^*})$.
    \ENDFOR
 \end{algorithmic}
 \end{algorithm}

\mypar{Learning} Our algorithm proceeds by episodes, indexed by $k\in\NN$ with $n_k$ denoting the start of the $k\textsuperscript{th}$ episode. At each $n_k$,  \cref{alg: RL 1} revises its knowledge using the Non-Linear Least-Square fit and the associated confidence set $\Cc_{n_k}(\delta)$, defined (for $\beta_n(\delta)$ given in \eqref{eq: def beta_n main text} for all $n\in\NN$) by 
\begin{align}
    \hat\theta_{n_k}&\in\argmin_{\theta\in\Theta}\sum_{n=0}^{n_k-1}\norm{X_{\tau_{n+1}}^{\varpi,\theta^*}-X_{\tau_{n}}^{\varpi,\theta^*}-\mu_\theta(X_{\tau_{n}}^{\varpi,\theta^*},\varpi_{\tau_n})}^2 \label{eq: def NLLS main text}\,,\\
    \Cc_{n_k}(\delta)&:= \left\{\theta\in\parset: \sqrt{\sum_{n=0}^{n_k-1} \norm{\mu_\theta(X_{\tau_n}^{\varpi,\theta^*},\varpi_{\tau_n})-\mu_{\hat\theta_{n_k}}(X_{\tau_{n}}^{\varpi,\theta^*},\varpi_{\tau_{n}})}^2} \le \beta_{n_k}(\delta)\right\}\,.\label{eq: def conf sets}
\end{align}

\mypar{Lazy Updates} Our episodic scheme follows the same rationale as in \cite{jaksch_near-optimal_2010,pmlr-v19-abbasi-yadkori11a}, and triggers updates as soon as enough information is collected. Formally, it constructs a sequence of episodes $\{S_k\}_{k\in\NN}$ whose starting times are defined by $n_0:=0$ and, for any $k\in\NN$, $n_{k+1}$ is the first time $n>n_k$ satisfying \eqref{eq: def lazsy update}
\begin{align}
    \sqrt{\sup_{\theta\in\Cc_{n_k}(\delta)}\sum_{i=0}^{n}\norm{\mu_\theta(X_{\tau_i}^{\varpi,\theta^*},\varpi_{\tau_i}) - \mu_{\hat\theta_{n_k}}(X_{\tau_i}^{\varpi,\theta^*},\varpi_{\tau_i})}^2}>2\beta_n(\delta)\,.\label{eq: def lazsy update}
\end{align}

\mypar{Planning} At the heart of our proposal is the way in which we address the optimistic planning, detailed in \cref{subsec:results planning}. For a given parameter $\theta\in\Cc_{n_k}(\delta)$, we leverage the connection between our setting and its continuous-time counterpart. We consider  continuous-time controls $\bar\alpha\in\bar\Ac$ with diffusive average reward given by
\begin{align}
    \bar\rho^{\bar\alpha}_\theta(x_0):=\liminf_{T\to+\infty}\frac1T\EE\left[\int_0^T \bar r(\bar X^{\bar\alpha,\theta}_t,\bar\alpha_t)\de t\right]\,\mbox{ in which }\begin{cases}
        \de \bar X^{\bar\alpha,\theta}_t = \bar\mu_{\theta}(\bar X^{\bar\alpha,\theta}_t,\bar\alpha_t)\de t + \bar\Sigma\de W_t\\
        \bar X_0^{\bar\alpha,\theta}=x_0
    \end{cases}\!\!\!\!\label{eq: rho and diff process}
\end{align}
in which $W$ denotes a $\PP$-Brownian motion, $\bar\Fb$ its filtration, and $\bar\Ac$ the set $\Ab$-valued $\bar\Fb$-predictable processes.
This diffusive problem gives us an optimality criterion and associated optimal control\footnote{We will use the obvious notational confusion between the policy $\bar\pi^*_\theta$ and the control process it generates.}:
\begin{align}\bar\rho^*_{\theta}(x_0):=\sup_{\alpha\in\Ac}\bar\rho^{\bar\alpha}_\theta(x_0)\mbox{ and } \bar\pi^*_{\theta}\circ \bar X^{\bar\alpha,\theta} \in\argmax_{\bar\alpha\in\bar\Ac}\bar\rho^{\bar\alpha}_\theta(x_0)\label{eq: argsup diffusive}\end{align}
which approximates the original jump-process problem $\rho^*_\theta(x_0)$. This problem admits a Hamilton-Jacobi-Bellman (HJB) equation (given in \eqref{eq: prelim HJB diff} below) characterising an optimal policy $\bar\pi_\theta^*:\RR^d\to\Ab$ which yields a computable optimal Markov control for \eqref{eq: argsup diffusive}. 
\vfill

\begin{restatable}{theorem}{RegretAlgOne}\label{thm: regret alg 1}
    Under \cref{asmp: basics,asmp: both asmp joint}, for any $\delta\in(0,1)$, $x_0\in\RR^d$, and $\gamma\in(0,1)$, there is a pair $(C_\gamma,C)\in\RR_+^2$ of constants independent of $\ve$ such that \cref{alg: RL 1} achieves 
    \begin{align}
        R_T(\varpi) \le 2C_\gamma\ve^{\frac{\gamma}{2}}T + C\sqrt{\mathrm{d}_{E,T\ve^{-1}}\log(\Ns_{T\ve^{-1}}^\ve) T\log(T\delta^{-1})}
    \end{align}
    with probability at least $1-\delta$, in which $\mathrm{d}_{E,T\ve^{-1}}$ is the $2\ve /\sqrt{T}$-eluder dimension (see \cite[Def.~4.]{russo2013eluder} and \eqref{eq: def eluder dim in appendix} in \cref{app:widths}) of the class $\{\mu_\theta\}_{\theta\in\Theta}$ restricted to a ball of radius $\Oc(\sqrt{\log(T/\ve)})$, and $\log(\Ns_{T\ve^{-1}}^\ve)$ is the $\ve^{2}\snorm{\bar\Sigma}_\op^2/T$-log-covering number of this same restricted class.
\end{restatable}

\Cref{thm: regret alg 1} contains two terms of different nature. The linear term is inherited from the diffusive approximation planning method and scales with $C_\gamma\ve^{\gamma/2}$. The dependency of the constant in $\gamma$ is inherited from the analysis of \cite{abeille_diffusive_2022} and $C_\gamma<+\infty$ holds for $\gamma<1$. Quantifying the behaviour of $C_\gamma$ as $\gamma\uparrow 1$ is technically intricate. Nevertheless, our bound indicates that the long run approximation error vanishes as $\ve\downarrow0$ almost as a fast as $\sqrt{\ve}$. The second term quantifies all other sources of error, and exhibits
the expected scaling in the complexity measures of \cite{russo2013eluder}, in terms of both eluder dimension and log-covering numbers, as well as the $\sqrt{T}$ horizon dependency. 

\section{Ideas of the Proof}\label{sec: proofs and results}

\subsection{Stability}

Working with unbounded processes and generic drift requires us to prevent state blow-up, which could degrade regret regardless of learning. In \cref{prop: bounded state jump} we combine the Lyapunov stability of \eqref{eq:lyapunov asmp joint on jump problem} with concentration arguments to show that unstable trajectories can only happen with low probability. A detailed proof is given in \cref{app:stability}.
\begin{restatable}{proposition}{BoundedStateJump}\label{prop: bounded state jump}
        Under \cref{asmp: both asmp joint,asmp: basics}, there is a function $H_\delta(n)=\Oc\big(\sqrt{\log(n\delta^{-1}})\big)$ such that for any $\delta\in(0,1)$, $\alpha\in\Ac$, $x_0\in\RR^d$, and $\theta\in\parset$ we have
        \begin{align}
            \PP\left(\sup_{t\in\RR_+}\frac{\snorm{X_t^{\alpha,\theta}}}{H_\delta(N_t)}\ge 1 \right)\le \delta\,.\label{eq: sup bound process}
        \end{align}
\end{restatable}
Working on the high-probability event of \cref{prop: bounded state jump} allows us to handle the unbounded state in the learning, planning, and optimism.

\subsection{Learning}\label{subsec: results learning}

\mypar{Confidence Sets} The crux of our analysis is incorporating \cref{prop: bounded state jump} into the NLLS method of \cite{russo2013eluder} by refining it to be adaptive to the norm of the state process. For $R>0$, let $\Bc_2(R)\subset\RR^d$ denotes the Euclidean ball of radius $R$ at $0$. To adapt the log-covering number, we can work with $H_\delta$ by formally defining $\Ns_n^\ve$ as the size of the smallest cover $\Cs_n^\ve$ of $\Fs_\parset:=(\mu_\theta)_{\theta\in\Theta}$ such that
\begin{align}
    \sup_{\mu_1\in\Fs_\Theta}\min_{\mu_2\in\Cs_n^\ve} \sup_{x\in\Bc_2(H_\delta(n))}\norm{\mu_1(x)-\mu_2(x)}\le \frac{\ve\snorm{\bar\Sigma}_\op^2}n\,.\label{eq: def cover RvR intro}
\end{align}
Restricting the domain of $\Fs_{\Theta}$ allows us to handle the richness of unbounded models and states while following \cite{russo2013eluder} to define confidence sets. Let $\delta\in(0,1)$, set $\beta_0:=\ve^{\frac12}$, and let 
\begin{align}
    \!\!\!\!\beta_n(\delta)&:= \beta_0\vee 2\ve^{\frac12}\snorm{\bar\Sigma}_\op\left(\sqrt{1 + 2\left(\sqrt{2\log\left(\frac{4\pi^2n^3}{3\delta}\right)} \!+\! \sqrt{2\ve^{\frac12}\snorm{\bar\Sigma}_\op^{-1}\kappa_n(\delta) }\right)} +\sqrt{\kappa_n(\delta)} \right)\!\!\!\label{eq: def beta_n main text}
\end{align}
in which
    \[\kappa_n(\delta):= \log\left(\frac{2\pi^2n^2\ve\Ns_{n}^\ve}{3\delta}\left(\snorm{\bar\Sigma}_\op^2 + 8L_0^2(1+H_\delta(n))\right)\right)\,.\]
Using this choice $(\beta_n)_{n\in\NN}$ and replacing $n_k$ by $n$ in \eqref{eq: def conf sets} formally defines the confidence sets $(\Cc_n(\delta))_{n\in\NN}$. For any $\alpha\in\Ac$, the probability that the state process $X_t^{\alpha,\theta^*}$ outgrows $H_\delta(N_t)$ is small and, thus, this confidence set will hold with high probability as shown by \cref{prop: conf sets OvR}.
\begin{restatable}[{Adapted from \citep[Prop.~5]{osband2014model}}]{proposition}{ConfSetsVanRoy}\label{prop: conf sets OvR}
    Under \cref{asmp: basics,asmp: both asmp joint}, for any $x_0\in\RR^d$, and $\delta>0$,
    \begin{align}
        \PP\left(\left\{\theta^*\in\bigcap_{n=\deb}^\infty \Cc_{n}(\delta)\right\}\cap\left\{\sup_{n\in\NN^*}\frac{\norm{X_{\tau_n}^{\varpi,\theta^*}}}{H_\delta(n)}\le1\right\}\right)\ge 1-\delta\,, \label{eq: theta star in all confidence sets}
    \end{align}
\end{restatable}
Well-posed confidence sets are insufficient for low-regret approaches in the OFU paradigm. This high confidence (low fit error) of the NLLS estimator must be translated as low online prediction error. 

\mypar{Prediction error} To adapt the $\epsilon$-eluder dimension (defined for $\epsilon>0$ in \cite[Def.~3.]{osband2014model}), which we denote $\dime$, to our unbounded state we proceed on the trajectory. 
The relevant extension for us is given for $n\in\NN^*$ by the $2\sqrt{\ve/n}$-eluder dimension of the class $\{f\vert_B\}_{f\in\Fs_\Theta}$ of elements of $\Fs_\Theta$ restricted to the set $B_{n}:=\Bc_2(\sup_{t\le \tau_n}\snorm{X_t^{\varpi,\theta^*}})$, denoted by $\mathrm{d}_{\mathrm{E},n}:=\dime(\{f\vert_{B_{n}}\}_{f\in\Fs_\Theta},2\sqrt{\ve/n})$. In \cref{cor: conf set width no lazy updates} we obtain first and second order prediction error bounds from this eluder dimension. In \cref{cor: conf set width no lazy updates} the order notation $\tilde\Oc$ hides terms that are poly-logarithmic in $N_t$ and $\mathrm{d}_{\mathrm{E},N_t}$ whose the full details are given in \cref{app:widths}.
\begin{restatable}{proposition}{ConfSetWidthNoLazyUpdateMAINTEXT}\label{cor: conf set width no lazy updates}
    Under \cref{asmp: basics,asmp: both asmp joint}, for any $\delta\in(0,1)$, $\alpha\in\Ac$, $x_0\in\RR^d$, and $t\in\RR_+$, we have with probability at least $1-\delta$
    \begin{align}
        \sum_{n=\deb}^{N_t}\norm{\mu_{\hat\theta_n}(X_{\tau_n}^{\alpha,\theta^*},\alpha_{\tau_n}) - \mu_{\theta^*}(X_{\tau_n}^{\alpha,\theta^*},\alpha_{\tau_n})  } &\le \tilde\Oc\left(\sqrt{\ve\mathrm{d}_{\mathrm{E},{N_t}}\log(\Ns_{N_{t}}^\ve)N_{t}} +\mathrm{d}_{\mathrm{E},{N_t}}\right)\,,\label{eq: main text conf set width order 1}
        \intertext{and}
        \sum_{n=\deb}^{N_t} \norm{\mu_{\hat\theta_n}(X_{\tau_n}^{\alpha,\theta^*},\alpha_{\tau_n}) - \mu_{\theta^*}(X_{\tau_n}^{\alpha,\theta^*},\alpha_{\tau_n})    }^2&\le \tilde\Oc\left(\mathrm{d}_{\mathrm{E},{N_t}}\log(\Ns_{N_{t}}^\ve)\right)\,.\label{eq: main text conf set width order 2}
    \end{align}
\end{restatable}

\mypar{Lazy updates} We leverage the second order bound \eqref{eq: main text conf set width order 2} of \cref{cor: conf set width no lazy updates} to define our lazy-update scheme \eqref{eq: def lazsy update}. We show in \cref{app: regret bounds} that this scheme does not degrade the speed at which \cref{alg: RL 1} learns by more than a constant factor, while also ensuring that the policy is only updated logarithmically in the number of interactions up to any horizon. 

\subsection{Planning}\label{subsec:results planning}

\Cref{alg: RL 1} requires us to be able to plan using any $\theta\in\Theta$, and as such we will extend the definitions of $X^{\alpha,\theta}$, $\rho^{\alpha}_{\theta}(x_0)$, $\rho^*_\theta(x_0)$ to any $(\alpha,\theta)\in\Ac\x\Theta$ by replacing $\theta^*$ by $\theta$ in \eqref{eq: intro def process} and \eqref{eq: intro def rho}. Let $\As$ be the set of measurable maps from $\RR^d$ to $\Ab$.
For a given $\theta\in\Theta$, the well-posedness of the control problem $\rho^*_{\theta}(x_0)$ and its resolution are non-trivial.
\begin{restatable}[{Adapted from \citep[Thm.~2.3, Rem.~2.4.]{abeille_diffusive_2022}}]{proposition}{PropertiesRhoJump}\label{prop: PropertiesRhoJump}
    Under \cref{asmp: basics,asmp: both asmp joint}, there is $L_W\in\RR_+$, independent of $\ve$, such that for any $\theta\in\parset$ 
    \begin{enumerate}
    \item[{(i.)}] The map $x\mapsto\rho^*_\theta(x)$ is constant, taking only one value which we denote by $\rho^*_\theta\in\RR$; 
    \item[({ii.})] There is an $L_W$-Lipschitz function $W^*_\theta$ such that
        \begin{align}
            \ve \rho^*_\theta = \max_{a\in\actionset}\left\{  \EE[ W_\theta^*(x + \mu_{\theta}(x,a) + \Sigma\xi)] -W^*_\theta(x) + r(x,a) \right\}\;\forall x\in \RR^d\,;\label{eq: HJB jump}
        \end{align}
    \item[(iii.)] There is $\pi^*_\theta\in\As$, such that for all $x\in\RR^d$, $\pi_\theta^*(x)$ maximises the right hand side in~\eqref{eq: HJB jump}, and $\pi^*_\theta\circ X^{\pi^*_\theta,\theta}$ is an optimal Markov control, i.e. $\rho^{\pi^*_\theta}_\theta(\cdot)\equiv\rho^*_\theta$.
    \end{enumerate}
\end{restatable}
\cref{prop: PropertiesRhoJump}.({i.}) shows that the control problem $\rho^*_\theta$ is independent of the initial conditions and meaningfully ergodic, which follows from stability analysis of the process using \eqref{eq:lyapunov asmp joint on jump problem}. Points ({ii.}) and ({iii.}) show that there is an optimal policy, which can be computed by solving the HJB equation \eqref{eq: HJB jump}. As before, confusing policies in $\As$ and controls in $\Ac$, we will write $\rho^\pi_\theta$ and $X^{\pi,\theta}$ to simplify notation. Unfortunately \eqref{eq: HJB jump} is an integral equation with low regularity, owing to the non-local jumps of the system, which complicates its analysis and the construction of numerical solvers.

\mypar{Diffusion limit} In the limit regime of interest, i.e.\ as $\ve\downarrow0$, the non-local behaviour of \eqref{eq: HJB jump} vanishes and it becomes a diffusive HJB equation. The associated diffusive control problem $\bar\rho^*_\theta(x_0)$ has been extensively studied, see e.g. \citep{arisawa1998ergodic,arapostathis2012ergodic}. 
\begin{restatable}[{Adapted from \citep[Thm.~3.4.]{abeille_diffusive_2022}}]{proposition}{PropertiesRhoDiff}\label{prop: PropertiesRhoDiff}
    Under \cref{asmp: basics,asmp: both asmp joint}, for any $\theta\in\parset$,
    \begin{enumerate}
        \item[({i.})] The map $x\mapsto\bar\rho^*_\theta(x)$ is constant, taking only one value which we denote by $\bar\rho^*_\theta\in\RR$.  
        \item[({ii.})] There is an $L_W$-Lipschitz function $\bar W^*_\theta\in\Cc^2(\RR^d;\RR)$ such that
        \begin{align}
            \bar\rho^*_\theta&=\max_{a\in\actionset}\left\{\bar\mu_\theta(x, a)^\top\nabla \bar W_\theta^*(x)  + \bar r(x,a)\right\}+ \frac12\Trace[\bar\Sigma\bar\Sigma^\top\nabla^2\bar W_\theta^*(x)], \;\forall x\in\RR^d\,.\label{eq: prelim HJB diff}\end{align}
        \item[({iii.})]  There is $\bar\pi^*_\theta\in\As$ such that, for all $x\in\RR^d$,  $\bar\pi^*_\theta(x)$ maximises the right hand side in \eqref{eq: prelim HJB diff}, and $\bar\pi^*_\theta\circ \bar X^{\bar\pi^*_\theta,\theta}$ is an optimal Markov control, i.e. $\bar\rho^{\bar\pi^*_\theta}_\theta(\cdot) \equiv \bar\rho^*_\theta$.
    \end{enumerate}
\end{restatable}
\Cref{prop: PropertiesRhoDiff} ensures that the diffusive problem satisfies all the properties of \cref{prop: PropertiesRhoJump} (ergodicity, optimal policy, and HJB equation). However, the HJB \eqref{eq: prelim HJB diff} is now a second-order local PDE instead of a non-local integral equation. This local equation does not have cross-dependencies between points: the solution at $x$ depends only on its derivatives at $x$, which is fundamentally simpler than the non-local behaviour of \eqref{eq: HJB jump}. Moreover, this diffusive PDE belongs to a well-studied family, both from the points of view of theory \cite{GT, ladyzhenskaya1968linear} and of numerics \cite{knabner_numerical_2003,kushner_probability_1977}. These facts motivate the use of these tools to construct approximate planning methods for \eqref{eq: HJB jump} in the near-continuous time regime as $\ve\downarrow 0$.
\begin{restatable}[{Adapted from \citep[Thm.~3.6.]{abeille_diffusive_2022}}]{proposition}{PropRhoApprox}\label{prop: approx diff limit}
    Under \cref{asmp: both asmp joint,asmp: basics}, for any $\gamma\in(0,1)$, there is a constant $C_\gamma>0$, independent of $\ve$, such that, for any $\theta\in\parset$, 
    \begin{align}
        \abs{\bar\rho^*_\theta-\rho^*_\theta} \le C_\gamma \ve^{\frac\gamma2}\mbox{ and } \rho^*_\theta-\rho^{\bar\pi^*_\theta}_\theta(0) \le C_\gamma\ve^{\frac\gamma2}\,. \label{eq: approx diff limit for rho}
    \end{align}
    Moreover, there is a function $e_\theta:\RR^d\to\RR$ such that,
    \begin{align}
        \ve\rho^{\bar\pi^*_\theta}_\theta(0) &=\EE[ \bar W_\theta^*(x+ \mu_{\theta}(x,a) + \Sigma\xi)] -\bar W^*_\theta(x) +r(x,\bar\pi^*_\theta(x)) + e_\theta(x)\,,\,\forall x\in \RR^d \label{eq: approx diff limit evolution}
    \end{align}
    and there is $C_\gamma'>0$, independent of $\ve$, such that $\abs{e_\theta(x)}\le C_\gamma' \ve^{1+\frac\gamma2}(1+\norm{x}^3)$ for all $x\in\RR^d$.
\end{restatable}
\Cref{prop: approx diff limit}, combined with \eqref{eq: prelim HJB diff} provides a certifiable approximation for solving the control problem \eqref{eq: intro def rho} with off-the-shelf diffusive HJB solvers, at a cost independent of $\ve$. An example of this methodology is seen in \cite[\S~4]{abeille_diffusive_2022}, in which \cite[Fig.~1, p.~30]{abeille_diffusive_2022} shows the reduction in computational effort. \Cref{prop: approx diff limit} also provides in \eqref{eq: approx diff limit evolution} an HJB-like representation of the approximation, which provides a key with which to analyse the regret incurred when using this approximation. 

\subsection{Regret Decomposition}\label{subsec:regret}

To sketch the proof of \cref{thm: regret alg 1}, we work on the high-probability event of \cref{prop: conf sets OvR}, and omit martingale measurability issues this could cause. We will also ignore the randomness of jump times and consider $T \lesssim \ve N_T$, with $\lesssim$ denoting inequality up to a constant. \Cref{app: regret bounds} is dedicated to a complete proof.

\begin{proof}[Proof sketch of {\cref{thm: regret alg 1}}]  Let $k:\NN\to\NN$ map an event $n$ to the episode $k(n)$ to which it belongs and let $\theta_n:=\tilde\theta_{k(n)}$.
    We begin the regret decomposition by applying the HJB-like equation \eqref{eq: approx diff limit evolution} of \cref{prop: approx diff limit}.(iii.) to the rewards collected along the trajectory $r(X_{\tau_n}^{\varpi,\theta^*},\varpi_{\tau_n})$ in the definition of the regret. Conditioning as appropriate, this yields
    \begin{align}
    \Rc_T(\varpi) &= T\rho^*_{\theta^*} -\ve \sum_{n=\deb}^{N_T} \rho^{\bar\pi^*_{\theta_n}}_{\theta_n}(0)\tag{$R_1$}\label{eq: regret decomp R1}\\
    &+ \sum_{n=\deb}^{N_T} \EE[\bar W^{*}_{\theta_{n}}(\tilde X^{\varpi,\theta_n}_{\tau_{n+1}})\vert\Fc_{\tau_n}]-\bar W^{*}_{\theta_{n}}(X_{\tau_{n}}^{\varpi,\theta^*})\tag{$R_2$}\label{eq: regret decomp R2}\\
    & + \sum_{n=\deb}^{N_T} e_{\theta_n}(X_{\tau_n}^{\varpi,\theta^*})\tag{$R_3$} \label{eq: regret decomp sketch remainder hjb term}
    \end{align} 
    in which $\tilde X_{\tau_{n+1}}^{\varpi,\theta} := X_{\tau_n}^{\varpi,\theta^*}+\mu_\theta(X_{\tau_n}^{\varpi,\theta^*},\varpi_{\tau_n})+\Sigma\xi_{n+1}$, for $(n,\theta)\in\NN\x\Theta$, is a counterfactual one-step transition assuming parameter $\theta\in\Theta$. 
    
    On the event of \cref{prop: conf sets OvR}, $\theta^*$ is in $\cap_{n\in\NN}\Cc_{n}(\delta)$ and the optimism of \cref{alg: RL 1} ensures that $\bar\rho^*_{\theta^*}\le \bar\rho^*_{\theta_n}=\bar\rho^{\bar\pi^*_{\theta_n}}_{\theta_n}$ for all $n\in\NN$. Combining this with \cref{prop: approx diff limit}, show that \eqref{eq: regret decomp R1} decomposes into 
    \begin{align*}
        R_1&\lesssim \ve\left(\sum_{n=\deb}^{N_T}\left(\rho^*_{\theta^*}-\bar\rho^*_{\theta^*}\right) + \sum_{n=\deb}^{N_T} \left(\bar\rho_{\theta_n}^{*}-\rho_{\theta_n}^{\bar\pi^*_{\theta_n}}\right)\right)\le 4N_T C_\gamma \ve^{1+\frac{\gamma}{2}}\,.
    \end{align*}
    Also by \cref{prop: approx diff limit}, $R_3\le \ve^{1+\frac\gamma2}N_T(1+H_\delta(N_T)^3)$. Thus $R_1+R_3\lesssim C_\gamma \ve^{\frac\gamma2}T$.
    
 For \eqref{eq: regret decomp R2}, the identity
 \[\tilde X_{\tau_{n+1}}^{\varpi,\theta} = \tilde X_{\tau_{n+1}}^{\varpi,\theta^*}-\mu_{\theta^*}(X_{\tau_n}^{\varpi,\theta^*},\varpi_{\tau_n})+\mu_\theta(X_{\tau_n}^{\varpi,\theta^*},\varpi_{\tau_n}) \]
 combined with the Lipschitzness of $\bar W^*_\theta$ from \cref{prop: PropertiesRhoDiff}, yields
    \begin{align}
        R_2&\le L_{\bar W} \sum_{n=\deb}^{N_T} \norm{\mu_{\theta_n}(X_{\tau_n}^{\varpi,\theta^*},\varpi_{\tau_n})-\mu_{\theta^*}(X_{\tau_n}^{\varpi,\theta^*},\varpi_{\tau_n})}\tag{$R_4$}\label{eq: pf sketch, learning}\\
        &+\quad \sum_{n=\deb}^{N_T} \EE[\bar W_{\theta_n}^{*}( X_{\tau_{n+1}}^{\varpi,\theta^*}) - \bar W_{\theta_{n+1}}^{*}( X_{\tau_{n+1}}^{\varpi,\theta^*})\vert \Fc_{\tau_n}] \tag{$R_5$}\label{eq: pf sketch, lazy updates}\\
        &+\quad \sum_{n=\deb}^{N_T} \EE[\bar W_{\theta_{n+1}}^{*}( X_{\tau_{n+1}}^{\varpi,\theta^*})\vert\Fc_{\tau_n}]- \bar W_{\theta_n}^{*}(X_{\tau_n}^{\varpi,\theta^*})\tag{$R_6$}\label{eq: pf sketch, martingale}\,,
    \end{align}
    by adding and subtracting $\EE[\bar W^*_{\theta_{n+1}}(\tilde X^{\varpi,\theta^*}_{\tau_{n+1}})\vert \Fc_{\tau_n}]=\EE[\bar W^*_{\theta_{n+1}}(X^{\varpi,\theta^*}_{\tau_{n+1}})\vert \Fc_{\tau_n}]$. \eqref{eq: pf sketch, martingale} is a martingale term, which we can bound using concentration theory. Our lazy update-scheme ensures that $\theta_n\neq\theta_{n+1}$ only $\Oc(\log(N_T))$ times by time $T$, keeping 
    \eqref{eq: pf sketch, lazy updates} small. 
    
    It remains to show that the lazy update-scheme, does not degrade the learning of \eqref{eq: pf sketch, learning}, which is controlled by improvements to \cref{cor: conf set width no lazy updates} in \cref{app: learning} which yield
    \[ \sum_{n=\deb}^{N_T} \sup_{(\theta_1,\theta_2)\in\Cc_{k(n)}(\delta)^2} \norm{\mu_{\theta_1}(X_{\tau_{n}}^{\varpi,\theta^*},\varpi_{\tau_{n}})-\mu_{\theta_2}(X_{\tau_{n}}^{\varpi,\theta^*},\varpi_{\tau_{n}})} \lesssim \tilde\Oc(\sqrt{\dee(T\ve^{-1})\log(\Ns_{T\ve^{-1}}^\ve)T})\,.\]
\end{proof}

\section{Conclusion}

In this work we proposed a general framework for the Reinforcement Learning problem of controlling an unknown dynamical system, on a continuous state-action space, to maximise the long-term average reward along a single trajectory. In particular, we focused on the understudied high-frequency systems driven by many small movements. Modelling such systems as controlled jump processes, we provided an optimistic algorithm which leverages Non-Linear Least Squares for learning and the diffusive limit regime for approximate planning. This proof of concept calls for several further refinements to be implementable in practice.

\mypar{Optimism} The optimistic step of \cref{alg: RL 1} chooses $\tilde\theta_n$ in an inefficient manner. Like in UCRL2 \cite{jaksch_near-optimal_2010}, optimistic exploration can be performed at the same time as planning by solving an expanded HJB equation, i.e. \eqref{eq: prelim HJB diff} with the maximum now taken over $(a,\theta)\in{\Ab\x\Theta}$. Since our assumptions are uniform in $\theta$, this is possible up to a modified regret decomposition, as in \cite{jaksch_near-optimal_2010}.

\mypar{Lazy updates} The way we quantify learning progress to design the lazy update-scheme \eqref{eq: def lazsy update} remains fundamentally discrete. Computationally cheaper lazy update-schemes might be obtained through simpler heuristics. For instance, the scaling of the drift with $\ve$ suggests it could be possible to update periodically, directly in terms of the wall-clock time $T$.

\mypar{Case-by-case} As a proof of concept, we endeavoured to study the RL problem in high generality. However, practical applications must use all available model information to refine the method ad-hoc. This is true for the learning method (replace NLLS with a fit specialised to the model at hand and bound the eluder dimension and log-covering numbers), and for numerical schemes on the PDE \eqref{eq: prelim HJB diff} which are built on a case-by-case basis for $d>1$, see \cite{kushner_numerical_2001}.

\newpage
\bibliography{biblio.bib}


\newpage
\appendix
\part*{Appendices}

\section{Preliminaries}\label{app: notation}

\subsection{Organisation of Appendices}

We prove the results one by one, starting with stability, then learning, planning, and finally concluding with the regret proof of \cref{thm: regret alg 1}.

In \cref{app:stability}, we go over the probabilistic properties of our problem and show several bounds on the stability of the process, in the sense of high-probability and moment boundedness. In particular the main objective of this appendix is to prove \cref{prop: bounded state jump}.

In \cref{app: learning}, we show a generalisation of the existing theory of learning with NLLS to the case of unbounded functions on unbounded domains. The key results are \cref{prop: conf sets OvR,cor: conf set width no lazy updates}

In \cref{app:ctrl}, we provide a characterisation of the control part of the RL problem we analyse, including the diffusion limit approximation, namely \cref{prop: PropertiesRhoJump,prop: PropertiesRhoDiff,prop: approx diff limit}.

In \cref{app: regret bounds}, we perform regret analysis and collect the last few results used to prove the regret bound of \cref{thm: regret alg 1}. This includes treatement of the lazy update-scheme. 

The remainder of \cref{app: notation} is devoted to notations and short-hands used throughout, but each appendix is meant to be as notationally stand-alone as possible.

\subsection{General notation} 

The set of natural numbers including $0$ is denoted $\NN$, while $\NN^*:=\NN\setminus\{0\}$ denotes the set of (strictly) positive integers. For $n\in\NN^*$, we use $[n]$ to denote the set of positive integers up to and including $n$, i.e. $[n]:=\{1,\dots,n\}$. Let $\RR$ denote the set of real numbers and define $\RR_+:=(0,+\infty)$ and $\RR_*^d:=\RR^d\setminus\{0\}$. The space of sequences taking values in $S$ will be denoted by $S^{\NN}$. For $S\subset\RR^d$, we also denote the complement of $S$ by $S^{\rm c}:=\RR^d\setminus S$, we use the same notation for the complement of a probability event.

We denote by $\langle\cdot\vert\cdot\rangle$ the inner product on $\RR^d$, by $\snorm{\cdot}$ the Euclidean norm on $\RR^d$, and by $\snorm{\cdot}_\op$ the associated operator norm on $\RR^{d\x d}$. 
 For $R\in\RR_+$ and $x\in\RR^d$, we denote the Euclidean ball of radius $R$ centred at $x$ by $\Bc_2(x,R)$, and when $x=0$ we use the shorthand $\Bc_2(R)$ for $\Bc_2(0,R)$.

For $d\ge1$, $\Dc\subset\RR^d$ and $\Dc'\subset\RR$, we denote the space of continuous functions from $\Dc$ to $\Dc'$ by $\Cc^0(\Dc;\Dc')$. For any $k\in\NN^*$, we denote $\Cc^k(\Dc;\Dc')$ the subset of $\Cc^0(\Dc;\Dc')$ containing all functions which are continuously differentiable up to order $k$. 

\subsection{Problem dependent notation} 

 The space of {\itshape càdlàg} (rcll) functions from $[0,+\infty)$ to $\RR^d$, for $d\in\NN^*$, is denoted $\Db$ and $\PP$ is a probability measure on $\Omega:=\Db$. $(N_t)_{t\in\RR_+}$ denotes a marked $\PP$-compound Poisson process of intensity $\ve^{-1}>1$, $(\tau_n)_{n\in\NN}$ denotes the sequence of its arrival times, with $\tau_0:=0$, and $(\xi_{n})_{n\in\NN}$ denotes the sequence of its marks. Namely, the sequences $(\tau_n)_{n\in\NN}$ and $(\xi_{n})_{n\in\NN}$ are independent, $(\tau_{n+1}-\tau_n)_{n\in\NN}$ is i.i.d. with exponential distribution of parameter $\ve$ and $(\xi_{n})_{n\in\NN}$ is i.i.d. with standard Gaussian measure on $\RR^d$, which we denoted by $\nu$.  

 For $t\in[0,+\infty)$, $\Fc_t:=\sigma((\tau_n,\xi_n)_{\tau_n\le t})$ and the filtration $\Fb$ is the completion of $(\Fc_t)_{t\in\RR_+}$. The set of $\Fb$-adapted $\Ab$-valued processes, which we consider as admissible controls, is denoted $\Ac$. For any $(x_0,\alpha,\theta)\in\RR^d\x\Ac\x\Theta$, $X^{\alpha,\theta}$ is the solution of  
\begin{align}
    \begin{cases}
        X_{\tau_{n}}^{\alpha,\theta}=X_{\tau_{n-1}}^{\alpha,\theta}+\mu_{\theta}(X_{\tau_{n-1}}^{\alpha,\theta},\alpha_{\tau_{n-1}}) + \Sigma\xi_{n}\\
        X_{\tau_0}^{\alpha,\theta}=x_0
    \end{cases}\,.
    \label{eq: not def process}
\end{align}
When specifying the dependence on the initial condition $x_0\in\Rb^d$ is necessary, we write $X^{x_0,\alpha,\theta}$.
This process is defined for any $t\in[0,+\infty)$ by considering its trajectories as piece-wise constant on any interval of the form $[\tau_{n-1},\tau_n)$ for $n\in\NN^*$. For any $(x_0,\alpha,\theta)\in\RR^d\x\Ac\x\Theta$, the control problem is denoted by
\begin{align*}
    \rho^*_\theta(x_0) := \sup_{\alpha\in\Ac} \rho^\alpha_{\theta}(x_0) \mbox{ in which } \rho^\alpha_{\theta}(x_0):=\liminf_{T\to\infty}\frac1{T}\EE\left[\sum_{n=\deb}^{N_T} r(X_{\tau_n}^{x_0,\alpha,\theta},\alpha_{\tau_n})\right]\,.
\end{align*}

We denote by $W$ a $\PP$-Wiener process (a.k.a Brownian motion), by $\bar\Fb$ the $\PP$-augmentation of the filtration it generates, and by $\bar\Ac$ the collection of $\Ab$-valued and $\bar\Fb$-predictable processes. For any $(x_0,\bar\alpha,\theta)\in\RR^d\x\bar\Ac\x\Theta$, we denote by $\bar X^{\bar\alpha,\theta}$ (or $\bar X^{x_0,\bar\alpha,\theta}$ if specifying the initial condition) the solution of 
\begin{align}
    \begin{cases}
        \de \bar X_{t}^{\bar\alpha,\theta}=\bar\mu_\theta(\bar X_{t}^{\bar\alpha,\theta},\bar\alpha_t)\de t + \bar\Sigma\de W_t\\
        \bar X_{0}^{\bar\alpha,\theta}=x_0
    \end{cases}\,.
    \label{eq: def diffusion}
\end{align}
The associated control problem is denoted by 
\begin{align*}
    \bar\rho^*_\theta(x_0) := \sup_{\bar\alpha\in\bar\Ac} \bar\rho^{\bar\alpha}_{\theta}(x_0) \mbox{ in which } \bar\rho^{\bar\alpha}_{\theta}(x_0):=\liminf_{T\to\infty}\frac1{T}\EE\left[\int_0^T r(\bar X_{t}^{x_0,\bar\alpha,\theta},\bar\alpha_{t})\de t\right]\,.
\end{align*}
According to \cref{prop: PropertiesRhoDiff,prop: PropertiesRhoJump}, we defined the constants $\rho^*_\theta:=\rho^*_\theta(0)$ and $\bar\rho^*_\theta:=\bar\rho^*_\theta(0)$. For $\theta\in\Theta$, $\bar\pi^*_{\theta}$ denotes a policy in $\As$ ( the set of measurable maps from $\RR^d$ to $\Ab$) which maximises the right-hand side of the HJB equation \eqref{eq: HJB jump}  associated to $\bar\rho^*_\theta$ (see \cref{prop: PropertiesRhoDiff}). Throughout, we use the same notation for policies and the Markov controls they induce, provided there is no ambiguity.

We use $\varpi$ to denote the control process output of \cref{alg: RL 1} mathematically. For any $\omega\in\Omega$, the trajectory generated by \cref{alg: RL 1} is therefore defined as in \eqref{eq: not def process} by $X^{\varpi,\theta^*}_\cdot(\omega)$. By definition of \cref{alg: RL 1}, in its $k\textsuperscript{th}$ episode (i.e. for $t\in[\tau_{n_k},\tau_{n_k+1})$), $\varpi_t = \pi_k(X^{\varpi,\theta^*}_t)$, with $\pi_k:=\bar\pi_{\tilde\theta_k}^*$.

Throughout these appendices, we will use the shorthand $\pve_\theta(x,a):= x+\ve\bar\mu_\theta(x,a)$, for any $(x,a,\theta)\in\RR^d\x\Ab\x\theta$.

\newpage
\input{appendix_stability.tex}

\newpage
\input{appendix_learning.tex}

\newpage
\input{appendix_control.tex}

\newpage
\input{regret_proofs.tex}

\end{document}

%% file: ICML_preamble.tex
\usepackage{pgfplots}
\pgfplotsset{width=8cm,compat=1.15}

\usepackage{caption,subcaption,listings,multirow,booktabs}
\usepackage{microtype,stmaryrd,hyperref,url}

\usepackage[english]{babel}

\usepackage{amsmath,amssymb,amsfonts,amsopn,amsthm,mathrsfs}
\usepackage{bm,dsfont}
\usepackage{mathtools}
\usepackage{bigints}
\usepackage{latexsym}
\usepackage{thmtools}
\usepackage{thm-restate}

\usepackage[capitalize,noabbrev]{cleveref}

\theoremstyle{plain}
\newtheorem{theorem}{Theorem}[section]

\newtheorem{lemma}[theorem]{Lemma}
\newtheorem{corollary}[theorem]{Corollary}

\theoremstyle{definition}
\newtheorem{definition}[theorem]{Definition}
\newtheorem{assumption}{Assumption}

\crefname{assumption}{Assumption}{Assumptions}

\theoremstyle{remark}
\newtheorem{remark}[theorem]{Remark}

%% file: commands.tex
\def\x{\times}

\DeclareMathOperator*{\argmax}{argmax}
\DeclareMathOperator*{\argmin}{argmin}
\DeclareMathOperator{\Trace}{Tr}
\DeclareMathOperator{\Tr}{Tr}

\DeclareMathOperator{\idmat}{\boldmath{I}}

\def\mybigtimes{\mathop{\mathchoice{
   \vcenter{\hbox to10bp{\vrule height15bp width0pt \pdfliteral{
   q 1 J .8 w 0 1 m 10 14 l S 0 14 m 10 1 l S Q
}\hss}}}{
   \vcenter{\hbox to10bp{\kern1bp\vrule height10bp width0pt \pdfliteral{
   q 1 J .65 w 0 0 m 8 10 l S 0 10 m 8 0 l S Q
}\hss}}}{\times}{\times}
}}

\newcommand{\abs}[1]{\left\lvert #1 \right\rvert}

\newcommand{\norm}[1]{\left\lVert#1\right\rVert}
\newcommand{\snorm}[1]{\lVert#1\rVert}

\newcommand{\NN}{\mathbb{N}}

\newcommand{\RR}{\mathbb{R}}
\newcommand{\PP}{\mathbb{P}}
\newcommand{\EE}{\mathbb{E}}

\newcommand{\1}{\mathds{1}}

\def\op{{\rm op}}
\def\deb{1}
\newcommand{\de}{\mathrm{d}}

    \def\Ab{\mathbb{A}}

    \def\Db{\mathbb{D}}
    
    \def\Fb{\mathbb{F}}
    \def\Hb{\mathbb{H}}

    \def\Rb{\mathbb{R}}

    \def\Ac{\mathcal{A}}
    \def\Bc{\mathcal{B}}
    \def\Cc{\mathcal{C}}
    \def\Dc{\mathcal{D}}
    \def\Ec{\mathcal{E}}
    \def\Fc{\mathcal{F}}
    \def\Hc{\mathcal{H}}

    \def\Nc{\mathcal{N}}
    \def\Oc{\mathcal{O}}

    \def\Rc{\mathcal{R}}

    \def\Xc{\mathcal{X}}

    \def\Er{\mathrm{E}}

    \def\dr{\mathrm{d}}

    \def\As{\mathscr{A}}
    
    \def\Cs{\mathscr{C}}

    \def\Fs{\mathscr{F}}

    \def\Ns{\mathscr{N}}

    \def\Vs{\mathscr{V}}

    \def\cf{\mathfrak{c}}

\def\actionset{\Ab}
\def\parset{\Theta}

\def\thetaopt{\tilde\theta}

\def\policy{\pi}

\def\ve{\varepsilon}
\def\pve{\psi^\ve}

\DeclareMathOperator{\dimension}{dim}

\def\dime{\dimension_{\Er}}
\def\dee{\dr_{\Er}}

\makeatletter
\newcommand{\vvast}{\bBigg@{3}}
\newcommand{\vast}{\bBigg@{4}}
\newcommand{\Vast}{\bBigg@{5}}
\makeatother

%% file: appendix_stability.tex
\section{State Process Stability}\label{app:stability}

A key aspect of our setting is that both the state process $X^{\alpha,\theta}$, for any $(\alpha,\theta)\in\Ac\x\Theta$, and the drift $\mu$ itself are unbounded. This can lead to an exponential blow-up of the state process, which can be harmful to both the learning and control aspects. In order to avoid this difficulty we imposed \cref{asmp: both asmp joint}, which corresponds to a stochastic Lyapunov condition, and ensures that the state will not explode in expectation. We reinforce this result by leveraging concentration theory to obtain the high-probability bound of \cref{prop: bounded state jump}. \Cref{app: HP state bound} is dedicated to its proof, and it will be used in the proofs of learning results and high-probability regret bounds (\cref{app: learning,app: regret bounds}).

\BoundedStateJump*

Unlike learning and regret, the analysis of the control task is done in expectation via the HJB equation. Here the unbounded drift will materialise as higher moments of $X^{\alpha,\theta}$. The counterpart of \cref{prop: bounded state jump} in this case is a moment result, given by \cref{lemma: moment boundedness condition from paper}, which is proved in \cref{app: moment bounds state} and will then be used in \cref{app:ctrl}.

\begin{restatable*}{lemma}{BoundedMomentStateJump}\label{lemma: moment boundedness condition from paper}
    Under \cref{asmp: basics,asmp: both asmp joint}, for any $p\ge2$, there is a constant $\cf_p'>0$ independent of $\ve$ such that
    \[\EE\left[\snorm{X^{x_0,\alpha,\theta}_t}^p\right]\le \frac1{\ell_\Vs^p}\left(L_\Vs^p e^{-\frac{\cf_\Vs}4 t}\norm{x_0}^p + \frac{4\cf_p'}{\cf_\Vs}\left(1-e^{-\frac{\cf_\Vs}4 t}\right) \right)\,, \]
    for any $(x_0,\alpha,\theta)\in\RR^d\x\Ac\x\Theta$ and $t\in[0,+\infty)$.
\end{restatable*}

\subsection{Proof of \texorpdfstring{\cref{prop: bounded state jump}}{Proposition \ref{prop: bounded state jump}}}\label{app: HP state bound}

This appendix is dedicated to the proof of \cref{prop: bounded state jump} which is a high probability bound on the state process. This proof follows the Chernoff method. Thus, we will derive an exponential moment bound for the state process in \cref{lemma: MGF bound on process}. We will first obtain a stochastic stability condition in expectation in \cref{lemma: stoch lyap semi-contraction for x large}. In what follows, let $R_\ve:=\sqrt{8d\log(1/\ve)}$ and  $\xi\sim\nu$. 

\begin{lemma}\label{lemma: stoch lyap semi-contraction for x large}
    Under \cref{asmp: both asmp joint,asmp: basics}, 
    \begin{enumerate}
        \item[{\rm (i.)}] for any $(\eta,x,a,\theta)\in\RR^d\x\RR^d\x\Ab\x\Theta$, we have
        \begin{align}
             \Vs(\pve_\theta(x,a)-\sqrt\ve\eta)\le (1-\ve\cf_\Vs)\Vs(x-\sqrt\ve\eta) + \ve M_\Vs L_0 (1+\norm{\eta})\,;\label{eq: removing drift effect p=1} 
        \end{align}
        \item[{\rm (ii.)}] and, for any $(a,\theta)\in\Ab\x\Theta$, and any $x\not\in\Bc_2(\ve^{\frac12}\snorm{\bar\Sigma}_\op R_\ve)$ we have
        \[\EE[\Vs(\pve_\theta(x,a)+\Sigma\xi)]\le (1-\ve\cf_\Vs)\Vs(x) + \ve\cf_\Vs'\]
        in which $\cf_\Vs'$ is a constant independent of $\ve$.
    \end{enumerate}
\end{lemma}

\begin{proof}\hfill
    \begin{enumerate}
        \item[{(\rm i.)}] By Lipschitzness of $\Vs$ and \eqref{eq:lyapunov asmp joint on jump problem}, for any $(\eta,x,a,\theta)\in\RR^d\x\RR^d\x\Ab\x\Theta$, we have
        \begin{align*}
            \Vs(\pve_\theta(x,a)-\sqrt\ve\eta)&= \Vs(\pve_\theta(x,a)-\pve_\theta(\sqrt\ve\eta,a)+\ve\bar\mu(\sqrt\ve\eta,a))\\
            &\le  \Vs(\pve_\theta(x,a)-\pve_\theta(\sqrt\ve\eta,a)) + M_\Vs\ve\norm{\bar\mu(\sqrt{\ve}\eta,a)}\\
            &\le (1-\ve\cf_\Vs)\Vs(x-\sqrt\ve\eta)+M_\Vs\ve\norm{\bar\mu(\sqrt{\ve}\eta,a)} \,,
        \end{align*}
        from which \eqref{eq: removing drift effect p=1} follows by using \cref{asmp: basics}, which implies $\norm{\bar\mu(\sqrt{\ve}\eta,a)}\le L_0 (1+\sqrt\ve\norm{\eta})\le L_0 (1+\norm{\eta})$ since $\ve\in(0,1)$.

    \item[{(\rm ii.)}] For any $x\in\RR^d$, by the symmetry of the law of $\bar\Sigma\xi$, by \eqref{eq: removing drift effect p=1} applied for $\eta=\bar\Sigma\xi$, and by taking the expectation, we have
    \begin{align}
        \EE[\Vs(\pve_\theta(x,a)+\Sigma\xi)]&=\EE[\Vs(\pve_\theta(x,a)-\sqrt\ve\bar\Sigma\xi)]\notag\\
        &\le (1-\ve\cf_\Vs)\EE[\Vs(x-\sqrt\ve\bar\Sigma\xi)] + \ve M_\Vs L_0 (1+\snorm{\bar\Sigma}_\op\EE[\norm{\xi}])\,.\label{eq:proof of stoch lyap of order 1 in jump problem removing mu}
    \end{align}
    Since $\xi$ is a standard Gaussian, $\norm{\xi}^2$ is a random variable following a $\chi^2$ distribution with $d$ degrees of freedom, thus $\EE[\norm{\xi}^2]=d$, and by Jensen's inequality $\EE[\norm{\xi}]\le \sqrt{d}$. Thus the second term is bounded by $\ve M_\Vs L_0 (1+\norm{\bar\Sigma}_\op\sqrt{d})$.

    We now focus on bounding $\EE[\Vs(x-\Sigma\xi)]$. We would like to use a Taylor expansion, but care needs to be taken to handle the non-differentiability of $\Vs$ at $0$. Under the expectation, we distinguish two events: the event on which $\norm{\xi}<R_\ve$, which supports the main mass of $\nu$, and the event on which $\norm{\xi}\ge R_\ve$, corresponding to the tails. 

    \begin{enumerate}
        \item For the first event we consider (on which $\norm{\xi}<R_\ve$), for any $x\not\in\Bc_2(\norm{\Sigma}_\op R_\ve)$, we must have $0\not\in\Bc_2(x,\norm{\Sigma\xi})$, and thus $0\not\in (x+\Delta\Sigma\xi)_{\Delta\in[0,1]}$. Since this line segment doesn't contain $0$ (the only point at which $\Vs$ is not continuously differentiable), we can perform a second-order Taylor expansion of $\Vs$ to obtain
        \begin{align*}
            \EE[\Vs(x+&\Sigma\xi)\1_{\{\norm{\xi}<R_\ve\}}]\\
            &\le \EE\left[\left(\Vs(x)+ \xi^\top\Sigma^\top \nabla\Vs(x)+ \frac12\Tr[\Sigma\xi\xi^\top\Sigma^\top \nabla^2\Vs(\hat x)]\right)\1_{\{\norm{\xi}<R_\ve\}}\right]
        \end{align*}
        for some $\hat x\in (x+\Delta\Sigma\xi)_{\Delta\in[0,1]}$. By the Cauchy-Schwartz inequality and the derivative bounds of \cref{asmp: both asmp joint}, we obtain
        \begin{align*}
            \EE[\Vs(x+\Sigma\xi)\1_{\{\norm{\xi}_2<R_\ve\}}]&\le \Vs(x)+ \EE[ \xi^\top\1_{\{\norm{\xi}<R_\ve\}}]\Sigma^\top \nabla\Vs(x) + \frac\ve2 M_\Vs'\snorm{\bar\Sigma}_\op^2\\
            &\le \Vs(x)+ \frac\ve2 M_\Vs'\snorm{\bar\Sigma}_\op^2\,,
        \end{align*}
        since $\EE[ \xi^\top\1_{\{\norm{\xi}<R_\ve\}}]=0$ by the rotational invariance property of a truncated Gaussian.
        \item On the second event (on which $\norm{\xi} \ge R_\ve$), we cannot use a Taylor expansion. Instead, we use the Lipschitzness of $\Vs$ followed by the Cauchy-Schwartz inequality, and then apply a sub-Gaussian concentration inequality (see e.g.~\cite[(3.5)]{ledoux1991probability}):
    \begin{align*}
        \EE[\Vs(x+\Sigma\xi)\1_{\{\norm{\xi}\ge R_\ve\}}]&\le \Vs(x) + M_\Vs\norm{\Sigma}_\op\EE[\norm{\xi}\1_{\{\norm{\xi}\ge R_\ve\}}]\\
        &\le \Vs(x) + M_\Vs\norm{\Sigma}_\op\sqrt{\EE[\snorm{\xi}^2]\PP(\norm{\xi}\ge R_\ve)}\\
        &\le  \Vs(x) + M_\Vs\norm{\Sigma}_\op\sqrt{4de^{-\frac{R_\ve^2}{8d}}}\\
        &\le \Vs(x) + 2\ve M_\Vs\snorm{\bar\Sigma}_\op\sqrt{d}\,.
    \end{align*}
    \end{enumerate}
    To complete the proof, we combine both cases in \eqref{eq:proof of stoch lyap of order 1 in jump problem removing mu}, and let
    \[\cf_\Vs':= M_\Vs L_0(1+\snorm{\bar\Sigma}_\op\sqrt{d}) + 2M_\Vs\snorm{\bar\Sigma}_\op\sqrt{d} + \frac{M_\Vs'}2\snorm{\bar\Sigma}^2_\op . \]
    \end{enumerate}%
    \end{proof}

    \begin{lemma}\label{lemma: MGF bound on process}
        Under \cref{asmp: both asmp joint,asmp: basics}, for any $(x_0,\alpha,\theta)\in\RR^d\x\Ac\x\Theta$ and any $\lambda\in\RR_+$, we have
        \[\EE[e^{\lambda\Vs(X_{\tau_n}^{x_0,\alpha,\theta})}]\le (n+1)\exp\left(\lambda\left(\frac{\cf'_\Vs}{\cf_\Vs}+ L_\Vs(\ve^{\frac12}\snorm{\bar\Sigma}_\op R_\ve+ \norm{x_0})\right)+\frac{\lambda^2M_\Vs^2\snorm{\bar\Sigma}^2_\op}{2\cf_\Vs}\right)\,,\]
        for any $n\in\NN$.
    \end{lemma}

    \begin{proof}
        For $n\in\NN^*$, let us define the following events for $i<n$: $E_{i,n-1}:=\{i = \sup\{j\in\{0,\dots, n-1\}: \snorm{X^{\alpha,\theta}_{\tau_j}}\le \snorm{\Sigma}_\op R_\ve\}\}$ and $\bar E_{n-1}:=\{\min_{j\in\{0,\dots, n-1\}}\snorm{X^{\alpha,\theta}_{\tau_j}}>\snorm{\Sigma}_\op R_\ve \}$. Note that both these events are $\Fc_{\tau_{n-1}}$-measurable and that $\cup_{i\le n-1} E_{i,n-1} = \bar E_{n-1}^c$, so that $\{\bar E_{n-1},E_{0,n-1},\dots,E_{n-1,n-1}\}$ induces a partition of $\Omega$. We begin by working conditionally to each of these events, and in a second part we will collect them to bound $\EE[\exp(\lambda\Vs(X^{\alpha,\theta}_{\tau_n})]$.

         For any $0\le i<n$, by the tower rule and by adding and subtracting $\EE[\exp\left(\EE[\lambda\Vs(X^{\alpha,\theta}_{\tau_{n}})\vert \Fc_{\tau_{n-1}}]\right)\1_{E_{i,{n-1}}}]$, we have
        \begin{align}
            \EE[e^{\lambda \Vs(X_{\tau_n}^{\alpha,\theta})}\1_{E_{i,n-1}}]&=\EE[\EE[e^{\lambda \Vs(X_{\tau_{n}}^{\alpha,\theta})}\vert \Fc_{\tau_{n-1}}]\1_{E_{i,{n-1}}}]\notag\\
            &=  \EE\bigg[\exp\left(\EE[\lambda\Vs(X^{\alpha,\theta}_{\tau_{n}})\vert \Fc_{\tau_{n-1}}]\right)\1_{E_{i,{n-1}}}\notag\\
            &\qquad\x \EE\left[ \exp\left(\lambda\Vs(X^{\alpha,\theta}_{\tau_{n}})-\EE[\lambda\Vs(X^{\alpha,\theta}_{\tau_{n}})\vert \Fc_{\tau_{n-1}}]\right)\vert \Fc_{\tau_{n-1}}\right] \bigg].\notag
            \intertext{Using a result for Lipschitz functions of Gaussian random variables (see e.g. \citep[Thm~5.5]{boucheron_concentration_2013}) applied to $\Vs$ and $\xi$, we obtain}
            \EE[e^{\lambda \Vs(X_{\tau_n}^{\alpha,\theta})}\1_{E_{i,n-1}}]&\le  e^{\frac{ \lambda^2}{2}M_\Vs^2\norm{\Sigma}^2_\op }\EE\bigg[\exp\left(\EE[\lambda\Vs(X^{\alpha,\theta}_{\tau_{n}})\vert \Fc_{\tau_{n-1}}]\right)\1_{E_{i,{n-1}}}\bigg]\notag\\
            &= e^{\frac{ \lambda^2}{2} M_\Vs^2\norm{\Sigma}^2_\op }\EE\bigg[\exp\left(\EE[\lambda\Vs(\pve_\theta(X^{\alpha,\theta}_{\tau_{n-1}}, \alpha_{\tau_{n-1}}) +\Sigma\xi_{n} )\vert \Fc_{\tau_{n-1}}]\right)\1_{E_{i,{n-1}}}\bigg]\!.\label{eq: concentration bound process intermediate decomp}
        \end{align}

        If $i=n-1$, $\snorm{X^{\alpha,\theta}_{\tau_{n-1}}}\le \snorm{\Sigma}_\op R_\ve$ on the event $E_{i,n-1}$, and thus we have 
        \begin{align*}
            \EE\left[ \lambda \Vs(\pve_\theta(X^{\alpha,\theta}_{\tau_{n-1}}, \alpha_{\tau_{n-1}}) +\Sigma\xi_{n} )\big\vert \Fc_{\tau_{n-1}}\right]&\le \EE\left[\lambda L_\Vs \norm{X^{\alpha,\theta}_{\tau_{n-1}} + \mu(X^{\alpha,\theta}_{\tau_{n-1}}, \alpha_{\tau_{n-1}}) + \Sigma\xi} \big\vert \Fc_{\tau_{n-1}}\right]\\
            &\le \lambda L_\Vs \big((1+L_0)\snorm{\Sigma}_\op R_\ve+ 1 + \snorm{\Sigma}_\op\sqrt{d}\big)
        \end{align*}
        by using the fact that $\EE[\norm{\xi}]\le\sqrt{\EE[\snorm{\xi}^2]}=\sqrt{d}$, as $\xi\sim\nu$. Noticing that $\sup_{\ve\in(0,1)}\ve^{\frac12}R_\ve=\sqrt{8de^{-1}}$, let us introduce 
        \begin{align}\label{eq: def CH}
            C_H:=L_\Vs \big((1+L_0)\snorm{\bar\Sigma}_\op \sqrt{8de^{-1}}+ 1 + \snorm{\bar \Sigma}_\op\sqrt{d}\big)\,.
        \end{align}  
        
        Combining this with \eqref{eq: concentration bound process intermediate decomp} yields
        \begin{align}
            \EE[e^{\lambda \Vs(X_{\tau_n}^{\alpha,\theta})}\1_{E_{i,n-1}}]&\le \exp\left(\frac{ \lambda^2}{2}M_\Vs^2\norm{\Sigma}^2_\op +\lambda C_H \right)\,,\label{eq: concentration bound process intermediate case 1}
        \end{align}
        in the case $i=n-1$.
        
        If $i<n-1$, we can apply the same methodology, and continuing from \eqref{eq: concentration bound process intermediate decomp} apply  \cref{lemma: stoch lyap semi-contraction for x large} to obtain
        \begin{align}
            \EE[e^{\lambda \Vs(X_{\tau_n}^{\alpha,\theta})}\1_{E_{i,n-1}}]&\le e^{\frac{ \lambda^2}{2}M_\Vs^2\norm{\Sigma}^2_\op }\EE\bigg[\!\exp\left(\EE\!\left[\lambda\Vs(\pve_{\theta}(X^{\alpha,\theta}_{\tau_{n-1}},\alpha_{\tau_{n-1}})+\Sigma\xi_{n})\vert \Fc_{\tau_{n-1}}\right]\right)\nonumber\\
        &\qquad\qquad\qquad\qquad \x\1_{\{X^{\alpha,\theta}_{\tau_{n-1}}>\snorm{\Sigma}_\op R_\ve\}}\1_{E_{i,{n-2}}}\bigg]\,,\label{eq: bbd recu}\\
        &\le  e^{\frac{ \lambda^2}{2}M_\Vs^2\norm{\Sigma}_\op^2+ \lambda\ve\cf_\Vs'}\EE[\exp((1-\ve\cf_\Vs)\lambda\Vs(X^{\alpha,\theta}_{\tau_{n-1}}))\1_{E_{i,{n-2}}}]\,.\nonumber
        \end{align} 
        It remains to use an induction argument in $n$ down to $n=i+1$ and use the fact that $\snorm{X^{\alpha,\theta}_{\tau_{i}}}\le \snorm{\Sigma}_\op R_\ve$ on $E_{i,i}$, to obtain
        \begin{align}
            &\EE[e^{\lambda \Vs(X_{\tau_n}^{\alpha,\theta})}\1_{E_{i,n-1}}]\notag\\
            &\quad\le \exp\left(\lambda C_H+ \lambda\ve\cf_\Vs'\!\sum_{k=0}^{n-1-i}(1-\ve\cf_\Vs)^k +\frac{\lambda^2 M_\Vs^2\snorm{\Sigma}_\op^2}{2}\sum_{k=0}^{n-1-i} \!(1-\ve\cf_\Vs)^{2k}\right)\notag\\
            &\quad\le \exp\left(\lambda C_H  +\lambda\frac{\cf_\Vs'}{\cf_\Vs} +\frac{\lambda^2 M_\Vs^2\snorm{\bar\Sigma}_\op^2}{2 \cf_\Vs}\right)\,.\label{eq: concentration bound process intermediate case 2}
        \end{align}

        On the event $\bar E_{n-1}$, that is if the process is never in the ball $\Bc_2(\snorm{\Sigma}_\op R_\ve)$ before time $\tau_n$, we use the fact that \eqref{eq: bbd recu} is valid with $\bar E_{n-1}$ and $\bar E_{n-2}$ in place of $ E_{i,n-1}$ and $ E_{i,n-2}$. Applying the induction, we obtain 
        \begin{align}
            \EE[e^{\lambda \Vs(X_{\tau_n}^{\alpha,\theta})}\1_{\bar E_{n-1}}]
            &\le \exp\left(\lambda L_\Vs\norm{x_0} +\lambda\frac{\cf_\Vs'}{\cf_\Vs} +\frac{\lambda^2 M_\Vs^2\snorm{\bar\Sigma}_\op^2}{2 \cf_\Vs}\right)\,.\label{eq: concentration bound process intermediate case 3}
        \end{align}
            
         Using our partition and combining \eqref{eq: concentration bound process intermediate case 1}, \eqref{eq: concentration bound process intermediate case 2}, and \eqref{eq: concentration bound process intermediate case 3} we can thus write, for any $n\in\NN$
        \begin{align*}
            \EE\left[e^{\lambda\Vs(X^{\alpha,\theta}_{\tau_n})}\right]&\le \EE\left[e^{\lambda\Vs(X^{\alpha,\theta}_{\tau_n})}\left(\1_{\bar E_{n-1}}+\sum_{i=0}^{n-1}\1_{E_{i,n-1}}\right)\right]\\
            &\le (n+1)\exp\left(\lambda\left(\frac{\cf'_\Vs}{\cf_\Vs}+ C_H+L_\Vs\norm{x_0}\right)+\frac{\lambda^2M_\Vs^2\snorm{\bar\Sigma}^2_\op }{2\cf_\Vs}\right)
        \end{align*}
        which concludes the proof.
    \end{proof}

With these two lemmas, we can now prove \cref{prop: bounded state jump}, the main result of this section. First, let us give the exact definition of $H_\delta(n)$:
\begin{align}
H_\delta(n):= \frac{1}{\ell_\Vs}\left(C_H+L_\Vs\norm{x_0}\right) + \frac{\cf'_\Vs}{\ell_\Vs\cf_\Vs} + \frac{M_\Vs}{\ell_\Vs}\snorm{\bar\Sigma}_\op \sqrt{\frac{2}{\cf_\Vs}\log\left(\frac{\pi^2(n+1)^3}{6\delta}\right)}\,\label{eq: def H delta}
\end{align}
in which $C_H$ is defined in \eqref{eq: def CH}, so that $H_\delta(n)=\Oc\big(\sqrt{\log(n\delta^{-1})}\big)$.

\BoundedStateJump*

\begin{proof}
    Fix $n\in\NN$, by Markov's inequality and \cref{asmp: both asmp joint}, for any $u>0$, we have
        \begin{align*}
            \PP\left(\snorm{X_{\tau_n}^{\alpha,\theta}}>u\right)&\le \EE\left[e^{\lambda\ell_\Vs \snorm{X_{\tau_n}^{\alpha,\theta}}}\right]e^{-\lambda\ell_\Vs u}\le \EE\left[e^{\lambda\Vs(X_{\tau_n}^{\alpha,\theta})}\right]e^{-\lambda\ell_\Vs u}\,,
        \end{align*}
        which implies that
        \begin{align*}
            &\PP\left(\snorm{X_{\tau_n}^{\alpha,\theta}}-\frac{\cf'_\Vs}{\ell_\Vs\cf_\Vs}- \frac{C_H}{\ell_\Vs}- \frac{L_\Vs}{\ell_\Vs}\norm{x_0}>u\right)\\
            &\qquad\qquad\le \EE\left[e^{\lambda\Vs(X_{\tau_n}^{\alpha,\theta})}\right]\exp\left(-\lambda\ell_\Vs\left(u + \frac{\cf'_\Vs}{\ell_\Vs\cf_\Vs}+ \frac{C_H}{\ell_\Vs}+ \frac{L_\Vs}{\ell_\Vs}\norm{x_0}\right) \right).
        \end{align*}
        Applying \cref{lemma: MGF bound on process}, and taking  $\lambda =\cf_\Vs\ell_\Vs u/(M_\Vs^2\snorm{\bar\Sigma}_{op}^{2})$, we obtain
        \begin{align*}
            \PP\bigg(\snorm{X_{\tau_n}^{\alpha,\theta}}>u +\frac{\cf'_\Vs}{\ell_\Vs\cf_\Vs}+ \ve^{\frac12}\frac{L_\Vs}{\ell_\Vs}\snorm{\bar\Sigma}_\op R_\ve+&\frac{L_\Vs}{\ell_\Vs}\norm{x_0}\bigg)\\
            &\le (n+1)\exp\left(-\lambda \ell_\Vs u + \lambda^2\frac{M_\Vs^2\snorm{\bar\Sigma}_\op ^2}{2\cf_\Vs} \right)\\
            &= (n+1)\exp\left( -\frac{\cf_\Vs\ell_\Vs^2}{2M_\Vs^2\snorm{\bar\Sigma}^2_\op } u^2 \right)\,.
        \end{align*}
        Letting $u=M_\Vs\snorm{\bar\Sigma}_\op \ell_\Vs^{-1}\sqrt{2\cf_\Vs^{-1}\log((n+1)/{\delta'})}$, yields
    \[\PP\left(\snorm{X_{\tau_n}^{\alpha,\theta}}\ge \frac{C_H}{\ell_\Vs}+ \frac{L_\Vs}{\ell_\Vs}\norm{x_0}+ \frac{\cf'_\Vs}{\ell_\Vs\cf_\Vs} + \frac{M_\Vs}{\ell_\Vs}\snorm{\bar\Sigma}_\op \sqrt{\frac{2}{\cf_\Vs}\log\left(\frac{n+1}{\delta'}\right)}\right)\le \delta'\!.\]
    Setting $\delta'= 6\delta/\pi^2(n+1)^2$, and taking a union bound over $n\in\NN$ yields 
    \[\PP\left(\sup_{t\in\RR_+} \frac{X^{\alpha,\theta}_t}{H_\delta(N_t)} \ge 1\right) = \PP\left({\bigcup}_{n\in\NN}\{\snorm{X_{\tau_n}^{\alpha,\theta}}\ge H_\delta(n)\}\right)\le \delta\,,\]
    which implies the result since $\delta\in(0,1)$ implies $\log(n^3/\delta)\le \log(n^3/\delta^3)= 3\log(n/\delta)$.
\end{proof}
 
\subsection{Expectation Bounds of Higher Orders}\label{app: moment bounds state}

In this appendix, we will focus on higher moment conditions of the state process, which will be used in the control results of \cref{app:ctrl}. In \cref{lemma: small intermediate lyapunov lemma p,cor: small intermediate lyapunov lemma} we work to raise the stochastic stability condition from \cref{lemma: stoch lyap semi-contraction for x large} to a power $p\ge 2$. \Cref{lemma: moment boundedness condition from paper}, the main result of this section, will follow from this by arguments of \cite{abeille_diffusive_2022}.

\begin{lemma}\label{lemma: small intermediate lyapunov lemma p}
    Under \cref{asmp: both asmp joint,asmp: basics}, for $p\ge2$, there is a function $g:\RR^d\x\RR^d\to\RR_+$ and a constant $C_p>0$ independent of $\ve$ satisfying 
    \[g(x,\eta)\le \ve C_p\left(1+\Vs(x-\sqrt\ve \eta)^{p-1}\right)\left(1+\norm{\eta}^{p}\right)\,,\] 
    for any $(\eta,x)\in\RR^d\x\RR^d$, such that
    \begin{align}
        \Vs(\pve_\theta(x,a)-\sqrt\ve\eta)^p\le (1-\ve\cf_\Vs)\Vs(x-\sqrt\ve\eta)^p + g(x,\eta)\,.\label{eq: removing drift effect p>1}
    \end{align}%
    for any $(\eta,x,a,\theta)\in\RR^d\x\RR^d\x\Ab\x\Theta$.
\end{lemma}

\begin{proof}
    We first raise both sides of
    \eqref{eq: removing drift effect p=1} to the power $p$
    \begin{align*}
        \Vs(\pve_\theta(x,a)-\sqrt\ve\eta)^p&\le \bigg((1-\ve\cf_\Vs)\Vs(x-\sqrt\ve\eta)+\ve M_\Vs L_0(1+\norm{\eta})\bigg)^p\,.
    \end{align*}
    We will now expand the right hand side. Let $a=(1-\ve\cf_\Vs)\Vs(x-\sqrt\ve\eta)$ and $b=\ve M_\Vs L_0(1+\norm{\eta})$, by the binomial theorem we have 
    \begin{align*}
        (a+b)^p&=\sum_{k=0}^p \binom{p}{k} a^k b^{p-k}= a^p + b\sum_{k=0}^{p-1} \binom{p}{k} a^k b^{p-1-k}\\
        &\le a^p + b(1+b)^{p-1} (1+a)^{p-1}\sum_{k=0}^{p-1}\binom pk\,.
    \end{align*}
    Since $(1-\ve\cf_\Vs)\in(0,1)$, $\ve\le 1$, $b\le 1+b$, and $\sum_{k=0}^{p-1}\binom pk\le 2^{p}$, by using the binomial identity $(1+a)^q\le 2^{q-1}(1+a^q)$ for $(a,q)\in[0,+\infty)\x[1,+\infty)$, we have
    \begin{align}
        \Vs(\pve_\theta(x,a)-\sqrt\ve\eta)^p &\le (1-\ve\cf_\Vs)\Vs(x-\sqrt\ve\eta)^p \notag\\
        &\qquad+ \ve (1+M_\Vs L_0(1+\norm{\eta}))^p(1+\Vs(x-\sqrt\ve \eta)^{p-1})2^{p-2+p}\,.\label{eq: expansion of g proof big}
    \end{align}
    Finally, we have 
    \begin{align}
        (1+M_\Vs L_0(1+\norm{\eta}))^p&= (1+ M_\Vs L_0 + M_\Vs L_0\norm{\eta})^p\notag\\
        &\le(1+ M_\Vs L_0 + (1+M_\Vs L_0)\norm{\eta})^p\notag\\
        &= (1+M_\Vs L_0)^p (1+\norm{\eta})^p\notag\\
        &\le (1+M_\Vs L_0)^{p}(1+\norm{\eta}^p) 2^{p-1}\,.\label{eq: expansion for of first term of g}
    \end{align}

    Combining \eqref{eq: expansion of g proof big} and \eqref{eq: expansion for of first term of g}, leads to the required result.
\end{proof}


Recall that $\xi\sim\nu$ is a centred standard Gaussian random variable.

\begin{corollary}\label{cor: small intermediate lyapunov lemma}
    Under \cref{asmp: basics,asmp: both asmp joint}, for any $p\ge 2$, there is a constant $\cf_p>0$ independent of $\ve$ such that
    \[\EE\left[\Vs(\pve_\theta(x,a)+\Sigma\xi)^p\right]\le\left(1-\ve\frac{\cf_\Vs}2\right)\EE[\Vs(x-\sqrt\ve\xi)^p] + \ve\cf_p\, \]
    for any $(x,a,\theta)\in\RR^d\x\Ab\x\Theta$.
\end{corollary}

\begin{proof}\hfill
    \begin{enumerate}
        \item[{\rm i.}] Taking the expectation of the bound on $g$ from \cref{lemma: small intermediate lyapunov lemma p} and applying Hölder's inequality yields
        \begin{align*}
            \EE[g(x,\xi)]&\le \ve C_p \EE\left[(1+\Vs(x-\sqrt\ve\xi)^{p-1})(1+\norm{\xi}^{p})\right]\\
            &\le \ve C_p\EE\left[ (1+\Vs(x-\sqrt\ve\xi)^{p-1})^{\frac{(p+1)}{p}}\right]^{\frac p{p+1}}\EE\left[ (1+\norm{\xi}^{p})^{p+1}\right]^{\frac1{p+1}}\\
            &\le 4\ve C_p \EE\left[1+\Vs(x-\sqrt\ve\xi)^{\frac{(p-1)(p+1)}p}\right] \EE\left[ (1+\norm{\xi}^{p})^{p+1}\right]^{\frac1{p+1}}\,,
        \end{align*}
        by using the identities: for $(u,v)\in\RR_+^2$, $(1+u)^{(p+1)/p}\le 4(1+ u^{(p+1)/p})$ and $(1+v)^{p/(p+1)}\le 1+v$.
        Since $\xi$ has bounded moments of any order,
        \[ C'_p:= 4C_p\EE\left[ (1+\norm{\xi}^{p})^{p+1}\right]^{\frac1{p+1}}\]
        is a finite constant and we have 
        \begin{align*}
            \EE\left[g(x,\xi)\right]&\le \ve C'_p\EE\left[ 1+\Vs(x-\sqrt\ve\xi)^{p-\frac{1}{p}}\right]\,.
        \end{align*}
        \item[{\rm ii.}] Recalling \cref{lemma: small intermediate lyapunov lemma p}, we have 
        \begin{align}
            \EE\left[\Vs(\pve_\theta(x,a)+\Sigma\xi)^p\right]&\le\left(1-\ve\cf_\Vs\right)\EE[\Vs(x-\sqrt\ve\xi)^p] +\EE[g(x,\xi)]\notag\\
            &\le \left(1-\ve\frac{\cf_\Vs}2\right)\EE[\Vs(x-\sqrt\ve\xi)^p] \notag\\
            &\quad +\ve \EE\left[C'_p(1+\Vs(x-\sqrt\ve\xi)^{p-\frac{1}{p}})-\frac{\cf_\Vs}2\Vs(x-\sqrt{\ve}\xi)^p\right]
            \,.\label{eq: corr stoch lyap order p proof eq}
        \end{align}
        \item[{\rm iii.}] 
        Note that, for any $p\ge 2$, the function 
        \[ z\in\RR^d\mapsto \frac{\norm{z}^{p-\frac1p}}{1+\norm{z}^p}\in\RR_+\]
        is bounded, so there exists a constant $C_p''>0$ such that, for any $z\in\RR^d$, 
        \[ C'_p\Vs(z)^{p-\frac1p}-\frac{\cf_\Vs}2\Vs(z)^p\le C_p''\,.\]
        Applying this to the expectation in \eqref{eq: corr stoch lyap order p proof eq}, we have 
        \begin{align*}
            \EE\left[\Vs(\pve_\theta(x,a)+\Sigma\xi)^p\right]&\le\left(1-\ve\frac{\cf_\Vs}2\right)\EE[\Vs(x+\sqrt\ve\xi)^p] +\ve(C_p''+C_p')\,.
        \end{align*}
        Letting $\cf_p:=C_p'+C_p''$ completes the proof.
    \end{enumerate}
\end{proof}

\BoundedMomentStateJump

\begin{proof}
    Recall from \cref{cor: small intermediate lyapunov lemma} that we have 
    \begin{align}
        \EE[\Vs(\pve_\theta(x,a)+\Sigma\xi)^p]\le \left(1-\ve\frac{\cf_\Vs}2\right)\EE[\Vs(x+\Sigma\xi)^p]+\ve\cf_p\label{eq: applying cor B4 to PJ order p expectation bound}
    \end{align}
    for any $(x,a,\theta)\in\RR^d\x\Ab\x\Theta$.
    We begin by eliminating the $\Sigma\xi$ from the right-hand side so that we have a proper Lyapunov contraction property on the generator. We expand $\Vs^p\in\Cc^2(\RR^d;[0,+\infty))$ and use the fact that $\EE[\xi]=0$  to obtain
    \begin{align*}
        \EE[\Vs(x+\Sigma\xi)^p]&=\Vs(x)^p+\ve p\EE[\Vs(x+\Delta\Sigma\xi)^{p-1}\Tr[\xi\bar\Sigma\bar\Sigma^\top\xi^\top\nabla^2\Vs(x+\Delta\Sigma\xi)]] \\
        &\,\,\,\,\,\, + \ve p(p-1)\EE[\Vs(x+\Delta\Sigma\xi)^{p-2}\Tr[\xi\bar\Sigma\bar\Sigma^\top\xi^\top\nabla\Vs(x+\Delta\Sigma\xi)\nabla\Vs^\top(x+\Delta\Sigma\xi)]]
    \end{align*}
    for some random variable $\Delta$ taking value in $[0,1]$. This is now upper-bounded by using the Lipschitzness of $\Vs$ and the Cauchy-Schwartz inequality
    \begin{align*}
        \EE[\Vs(x+\Sigma\xi)^p]&\le \Vs(x)^p+\ve p M_\Vs'\snorm{\bar\Sigma}_\op^2\EE[\left(\Vs(x)+M_\Vs\Delta\Sigma\norm{\xi}\right)^{p-1}\norm{\xi}^2] \\
        &\qquad + \ve p(p-1)(M_\Vs)^2\snorm{\bar\Sigma}_\op^2\EE[\left(\Vs(x)+M_\Vs\Delta\Sigma\norm{\xi}\right)^{p-2}\norm{\xi}^2]\,.
    \end{align*}
    By the binomial theorem as in the proof of \cref{lemma: small intermediate lyapunov lemma p}, and as $\abs{\Delta}\le 1$, we have 
    \begin{align*}
        \EE[\Vs(x+\Sigma\xi)^p]&\le \Vs(x)^p + \ve\bigg(pM'_\Vs\snorm{\bar\Sigma}_\op^2\EE\left[\norm{\xi}^2\sum_{k=0}^{p-1}\binom{p-1}{k}\Vs(x)^{k}(M_\Vs\norm{\Sigma}_\op\norm{\xi})^{p-1-k}\right] \\
        &\qquad + p(p-1)(M_\Vs\snorm{\bar\Sigma}_\op)^2\EE\left[\sum_{k=0}^{p-2}\binom{p-2}{k}\Vs(x)^k(M_\Vs\norm{\Sigma}_\op\norm{\xi})^{p-2-k}\right]\bigg).
    \end{align*}
    Since $\norm{\xi}$ is a sub-Gaussian random variable it has moments of all orders, and we can express the interior of the bracket above as a polynomial in $\Vs(x)$ of order $p-1$ with finite coefficients $\{a_{k}\}_{k=0}^{p-1}\subset\RR_+$. Recalling \eqref{eq: applying cor B4 to PJ order p expectation bound}, we thus have
    \begin{align*}
        \EE[\Vs(\pve_\theta(x,a)+\Sigma\xi)^p]&\le (1-\ve\cf_\Vs)\left(\Vs(x)^p + \ve\sum_{k=0}^{p-1} a_k \Vs(x)^{k}\right) + \ve\cf_p\\
        &\le \left(1-\ve\frac{\cf_\Vs}4\right)\Vs(x)^p + \ve\left(\cf_p -\frac{\cf_\Vs}{4}\Vs(x)^p + \sum_{k=0}^{p-1} a_k \Vs(x)^{k}\right) 
    \end{align*} 
    As in part {\rm iii.} of the proof of \cref{{cor: small intermediate lyapunov lemma}}, the interior of the second bracket is a continuous function which goes to $-\infty$ as $\norm{x}\to+\infty$, so there must be a constant $\cf'_p\in\RR_+$ (independent of $\ve$) such that 
    \[ \cf_p + \sup_{x\in\RR^d}\left(-\frac{\cf_\Vs}{4}\Vs(x)^p + \sum_{k=0}^{p-1} a_k \Vs(x)^{k}\right) \le \cf'_p <+\infty\,.\]
    Therefore, we have the desired Lyapunov generator condition 
    \begin{align*}
        \EE[\Vs(\pve_\theta(x,a)+\Sigma\xi)^p]&\le \left(1-\ve\frac{\cf_\Vs}{4}\right)\Vs(x)^p+\ve\cf_p'\,,
    \end{align*}
    which is equivalently written for any $(x,a)\in\RR^d\x\Ab$ as 
    \begin{align}
        \frac1\ve\int (\Vs(\pve_\theta(x,a)+\Sigma e)^p -\Vs(x)^p)\nu(\de e) \le -\frac{\cf_\Vs}4\Vs(x)^p + \cf'_p\,.\label{eq: lyapunov contraction condition order p in integral form} 
    \end{align}
    \noindent By Itô's Lemma, \eqref{eq: lyapunov contraction condition order p in integral form}, and a localisation argument, we have, for any $t\ge t_0\ge0$, that
    \begin{align*}
        \EE[\Vs(X^{x_0,\alpha,\theta}_t)^p]&=\EE[\Vs(X^{x_0,\alpha,\theta}_{t_0})^p]\\
        &\qquad+ \EE\left[\int_{t_0}^t\frac1\ve\int (\Vs(\pve_\theta(X^{x_0,\alpha,\theta}_{s}, \alpha_{s})+\Sigma e)^p-\Vs(X^{x_0,\alpha,\theta}_{s})^p)\nu(\de e)\de s\right]\\
        &\le \EE[\Vs(X^{x_0,\alpha,\theta}_{t_0})^p] -\frac{\cf_\Vs}4 \int_{t_0}^t \EE\left[\Vs(X^{x_0,\alpha,\theta}_{s})^p\right]\de s + (t-t_0)\cf_p'\,.
        \intertext{By a simple comparison argument for ODEs, we then obtain}
        \EE[\Vs(X^{x_0,\alpha,\theta}_t)^p]&\le e^{-\frac{\cf_\Vs}4 t}\Vs(x_0)^p + \frac{4\cf_p'}{\cf_\Vs}\left(1-e^{-\frac{\cf_\Vs}4t}\right).
    \end{align*}
    Using now \cref{asmp: both asmp joint}, we obtain
    \[\EE[\snorm{X^{x_0,\alpha,\theta}_t}^p]\le \frac1{\ell_\Vs^{p}}\left(L_\Vs^{p} e^{-\frac{\cf_\Vs}4 t}\norm{x_0}^p + \frac{4\cf_p'}{\cf_\Vs}\left(1-e^{-\frac{\cf_\Vs}4 t}\right)\right)\,.\]
\end{proof}

        

%% file: appendix_learning.tex
\section{Concentration Inequality and Online Prediction Error}\label{app: learning}


The key result of this section, \cref{prop: conf sets OvR}, builds heavily on \citep[Prop.~5]{russo2013eluder}.  \Cref{prop: conf sets OvR} differs from this existing result in three ways. First, it is \emph{any-time} i.e. does not require \textit{a priori} knowledge of a time horizon. This is a minor technical refinement, but it is of practical importance. Second, it applies to a pure-jump process defined on $\RR_+$. This apparent complexity vanishes when the filtration of the pure-jump process is chosen correctly, as the state process is piece-wise constant. Third, and most important, it applies to learning in a function class ($\Fs_\Theta$) of unbounded drifts for an unbounded process $X^{\alpha,\theta}$, which is an inherent difficulty in handling continuous state RL problems.  

This third extension is non-trivial and leads us to significantly reshuffle the proof structure of \citep{russo2013eluder}, and to incorporate some self-normalised inequality arguments as well as high-probability bounds on the state from \cref{app:stability}. While many of the original ideas are still used, the way they link together has changed and thus we will include, in \cref{app: conf sets}, a complete derivation for the sake of clarity. In this spirit, we will prove a generic result (\cref{thm: eluder learning bound}), which itself implies \cref{prop: conf sets OvR}.

\ConfSetsVanRoy*

\Cref{prop: conf sets OvR} ensures that the sets $(\Cc_n(\delta))_{n\in\NN}$ defined in \eqref{eq: def conf sets} are valid confidence sets. 
In order to bound the regret, we need to go further and to bound the online prediction error of functions within these confidence sets along the trajectory (see. \eqref{eq: conf set width order 1}). 

For any $n\in\RR$, let $\mathrm{d}_{\mathrm{E},n}$ denotes the $2\sqrt{\ve/n}$-eluder dimension of the model class restricted to the set $B_n:=\Bc_2(\sup_{s\le \tau_{n}}\snorm{X_s^{\varpi,\theta^*}})$, i.e. $\mathrm{d}_{\mathrm{E},n}:=\dime(\{f\vert_{B_n}\}_{f\in\Fs_\Theta},2\sqrt{\ve/n})$. In \cref{app:widths}, we derive a general result (\cref{prop: conf set width no lazy updates}) from which \cref{prop: conf sets OvR} follows.

\ConfSetWidthNoLazyUpdateMAINTEXT*

\subsection{Confidence sets}\label{app: conf sets}

In this section, we work in a generic online learning framework, so that our results can be more easily compared and contrasted with \citep{russo2013eluder,osband2014model} and others. We, therefore, introduce some dedicated notation and a stand-alone assumption for this section. 

Consider a set of functions $\Fs$ from $\RR^d\to\RR^d$, and fix $f^*\in\Fs$. We will study pairs of (random) $\RR^d$-valued sequences $((X_i)_{i\in\NN},(Y_{i})_{i\in\NN})$ generated as
\[ Y_i = f^*(X_{i})+ \xi_{i}\]
for $(\xi_i)_{i\in\NN}$ a stochastic process in some filtered probability space $(\Omega',\Hc_\infty,\Hb,\PP)$, with each $\xi_i$ independent of everything else up to time $i$. We take $\Hc_i$ as the completion of $\sigma(\{\xi_j\}_{j\le i})$, for $i\in\NN$, and we let $\Hb=(\Hc_i)_{i\ge 0}$.

Given some $\RR^d$-valued and $\Hb$-adapted sequences $(Z_i)_{i\in\NN}$ and $(Z'_i)_{i\in\NN}$, and some $n\in\NN^*$, let us define
\[ \langle Z \vert Z'\rangle_n := \sum_{i=0}^{n-1} \langle Z_i \vert Z'_i \rangle\mbox{ and } \snorm{Z}_n := \sqrt{\langle Z \vert Z \rangle_n}\,.\]
While $\norm{\cdot}_n$ is not a norm, it plays this role and we follow here the notational convention of \citep{russo2013eluder}. We will extend the definitions of $\langle\cdot\vert\cdot\rangle_n$ and $\norm{\cdot}_n$ to $n=0$ by simply taking the empty sum to be $0$, i.e. $\langle Z,Z'\rangle_0:=0$.

To simplify notation, we will drop the sequence $(X_i)_{i\in\NN}$ when it is an argument to a function inside $\norm{\cdot}_n$ or $\langle\cdot\vert\cdot\rangle_n$: i.e. $\norm{f}_n$ stands for $\norm{(f(X_i))_{i\in\NN}}_n$. With this notation in mind, for any $n\in\NN$, we define $\hat f_n$ as an arbitrary element of
\[  \argmin_{f\in\Fs} \norm{Y-f}^2_n\,.\]
In other words $\hat f_n$ is a non-linear least-square fit in $\Fs$ using the first $n$ points of $(X_i,Y_i)_{i\in\NN}$. 
In this generic setting, we introduce \cref{asmp: eluder}, which in our end-goal application subsumes \cref{asmp: both asmp joint,asmp: basics} and \cref{prop: bounded state jump}.

\begin{assumption}\label{asmp: eluder}
    There is $(L,\Gamma)\in\RR_+^2$ and a function $H_\delta:\NN\to\RR_+$ such that
    \[ \sup_{f\in\Fs}\sup_{x\in\RR^d}\frac{\norm{f(x)}}{1+\norm{x}}\le L\,,\]
    and for all $i\in\NN^*$, $\xi_i$ is an $\Hc_{i-1}$-conditionally $\Gamma^2$-sub-Gaussian random variable, $\xi_0$ is  $\Gamma^2$-sub-Gaussian, and the sequence $(X_i)_{i\in\NN}$ satisfies 
    \[\PP\left(\sup_{n\in\NN}\frac{\norm{X_n}}{H_\delta(n)}>1\right)<\delta\, \]
    for all $\delta\in(0,1)$.
\end{assumption}
Let $(\Cs_n^\Gamma)_{n\in\NN^*}$ denote a deterministic sequence of finite covers of $\Fs$, whose cardinalities are respectively given by $(\Ns_n^\Gamma)_{n\in\NN^*}$, such that for all $n\in\NN^*$
\[\sup_{f\in\Fs}\min_{g\in\Cs_n^\Gamma} \sup_{x\in\Bc(H_\delta(n))}\norm{f(x)-g(x)}\le \frac{\Gamma^2}n\,.\] 
The definition of this cover corresponds to one used in \cite{russo2013eluder} with a domain restricted to lie in the high-probability region of the state process instead of the whole domain. This ensures the cover remains finite for all $n\in\NN^*$.

For any $\delta\in(0,1)$, $n\in\NN^*$, and $f\in\Fs$ let us define the quantities
\begin{align*}
L^1_n(\delta)&:= \log((\Gamma^2 + 8L^2(1+\sup_{i\le n}\snorm{X_i}^2_2))\Ns_n^\Gamma\delta^{-1})\,,\\
L^0_n(\delta)&:= L^1_n(6\delta \pi^{-2}n^{-2})\,,\\
C^1_n(f)&:= \Gamma^2 + \norm{f-f^*}_n^2\\
C^2_n(f)&:=\sup_{i\le n}\snorm{f(X_i)-\hat f_n(X_i)}\,,
\end{align*} 
and the event
\begin{align}
    \Ec^0_n(\delta):=\Bigg\{ &\norm{\hat f_n -f^*}_n\le 2\Gamma \sqrt{L_n^1\left(\frac{3\delta}{\pi^2n^2}\right)} \notag\\
    &+2\sqrt{\Gamma^2+ 2\Gamma\left(n\sup_{g\in\Cs_n^\Gamma} C^2_n(g)\sqrt{2\log\left(  \frac{4\pi^2n^3}{3\delta}\right)}+\! \sqrt{2n\sup_{g\in\Cs_n^\Gamma}C^2_n(g)L^1_n\left(\frac{3\delta}{\pi^2n^2}\right)}\right)}\Bigg\}\!.\label{eq: def event Ec 0 conf set}
\end{align}

Building upon the proof method of \cite{russo2013eluder}, the cornerstone of this section is \cref{lemma: eluder learning big bound lemma}, which shows that, with high-probability, $f^*$ is contained in all the elements of a sequence of confidence sets, each centred at $\hat f_n$ in the $\norm{\cdot}_n$ norm.   

\begin{restatable*}{lemma}{LemmaEluderConfSetProba}\label{lemma: eluder learning big bound lemma}
    Under \cref{asmp: eluder}, for $n\in\NN^*$ and $\delta\in(0,1)$, we have
    \[ \PP\left(\bigcap_{n\in\NN^*} \Ec^0_n(\delta)\right)\ge 1-\delta\,.\]
\end{restatable*}

We begin the proof of \cref{lemma: eluder learning big bound lemma} by giving the concentration inequality of \cref{lemma: eluder self normalised bd on IP}.

\begin{lemma}\label{lemma: eluder self normalised bd on IP}
    Under \cref{asmp: eluder}, for all $n\in\NN^*$, $\delta\in(0,1)$, and $f\in\Fs$
    \[\PP\left( \abs{\langle \xi\vert f-f^* \rangle_n} \ge \Gamma\sqrt{2(\Gamma^2+\norm{f-f^*}_n)\log\left( \frac{\Gamma^2 + \norm{f-f^*}_n}{\delta}\right)}  \right)\le \delta\,.\]
\end{lemma}

\begin{proof}
    This proof relies on extensively studied arguments for self-normalised inequalities, but we include it for completeness because it uses non standard constants. Let us begin by fixing $f\in\Fs$. For all $n\in\NN$, let 
    \[ Z_n(f) := \langle \xi\vert f-f^*\rangle_n\,.\]
    For any $\lambda\in\RR$, let us define the process $(M_n^\lambda(f))_{n\in\NN}$ defined by
    \[M_n^\lambda(f):=\exp\left(\lambda Z_n(f)-\frac{\lambda^2\Gamma^2}{2}\norm{f-f^*}_n^{2}\right)\,.\]
    Let us show that $M_n^\lambda(f)$ is a conditional supermartingale. For any $n\in\NN$, we have
    \begin{align}
        \EE\left[ M_{n+1}^\lambda(f)\vert \Hc_n\right] &= M_n^\lambda(f)\EE\left[ \exp\big( \lambda \langle \xi_{n+1}\vert f(X_n) -f^*(X_n)\rangle_{n}\big)\bigg\vert \Hc_{n}\right] e^{- \frac{\lambda^2\Gamma^2}{2}\norm{f(X_n)-f^*(X_n)}^2_{n}}.\label{eq: supermartingale learning term}
    \end{align}
    By the Cauchy-Schwartz inequality 
    \[\abs{\langle\xi_{n}\vert f(X_n)-f^*(X_n)\rangle_{n}}\le \norm{\xi_{n}}_{n}\norm{f(X_n)-f^*(X_n)}_{n}\]
    and thus, since $\xi_{n}$ is conditionally $\Gamma^2$-subgaussian with variance $\Gamma^2$, $\norm{\xi_{n}}$ is $\Gamma^2$-subgaussian. Therefore
    \[\EE\left[ \exp\big( \lambda \langle \xi_{n}\vert f(X_n) -f^*(X_n)\rangle_{n} - \frac{\lambda^2\Gamma^2}{2}\norm{f(X_n)-f^*(X_n)}^2_{n} \big)\vert \Hc_{n}\right] \le 1\]
    and thus, by \eqref{eq: supermartingale learning term}, $M_n^\lambda(f)$ is a supermartingale. By definition of  $\langle\cdot\vert\cdot\rangle_0$ and $\norm{\cdot}_0$, $M_0^\lambda(f)=1$, so that $\EE[M_n^\lambda(f)]\le 1$ for all $n\in\NN$.

    We now perform a Laplace trick. Let $\Phi$ be the Gaussian measure of mean $0$ and variance $\Gamma^{-4}$ on $\RR$, and let us define the process $(M_n(f))_{n\in\NN}$ by
    \begin{align*}
        M_n(f) :&= \int M^\lambda_n(f) \Phi(\de \lambda)\\
        &= \int \exp\left(\lambda Z_n(f)-\frac{\lambda^2\Gamma^2}2\norm{f-f^*}_n^2\right)\Phi(\de \lambda)\\
        &= \frac{1}{\Gamma^2 + \norm{f-f^*}_n^2}\exp\left\{\frac{Z_n^2(f)}{2\Gamma^2(\Gamma^2+\norm{f-f^*}_n^2)}\right\}\,.
    \end{align*}

    By Markov's inequality, $\PP(M_n(f)\ge \delta^{-1})\le \delta$, and thus
    \begin{align*}
         \PP\left( Z_n(f) \ge \Gamma\sqrt{2(\Gamma^2+\norm{f-f^*}_n^2)\log\left( \frac{\Gamma^2 + \norm{f-f^*}_n^2}{\delta}\right)}\right)\le \delta\,.
    \end{align*}
\end{proof}

We will turn to the proof of \cref{lemma: eluder learning big bound lemma}.
Recall \eqref{eq: def event Ec 0 conf set}, which defined for $\delta\in(0,1)$ and $n\in\NN^*$, the event
\begin{align*}
    \Ec^0_n(\delta):=\Bigg\{& \norm{\hat f_n -f^*}_n\le 2\Gamma \sqrt{L_n^1\left(\frac{3\delta}{\pi^2n^2}\right)}\, \\
    &\!+2\sqrt{\Gamma^2+ 2\Gamma\left(n\sup_{g\in\Cs_n^\Gamma} C^2_n(g)\sqrt{2\log\left(  \frac{4\pi^2n^3}{3\delta}\right)}+ \sqrt{2n\sup_{g\in\Cs_n^\Gamma}C^2_n(g)L^1_n\left(\frac{3\delta}{\pi^2n^2}\right)}\right)}
    \Bigg\}\!.
\end{align*}

\LemmaEluderConfSetProba

\begin{proof}
    The proof builds on elements of \cite{russo2013eluder}. We begin by giving two small auxiliary results which we will use.
    \begin{enumerate}
    \item[i.] Let $n\in\NN^*$, and $\delta\in(0,1)$,  by a union bound over the family of conditionally sub-Gaussian random variables $(\norm{\xi_i})_{i\in[n]}$, we have
    \begin{align}
        \PP\left(\sup_{i\le n}\norm{\xi_i}\ge \Gamma\sqrt{2\log\left(\frac{2n}{\delta}\right)}\right) \le \delta\label{eq: eluder bound subgaussian sup}
    \end{align}
    \item[ii.] For any $f\in\Fs$, and $n\in\NN^*$ we have 
    \begin{align}
        \norm{f^*-Y}^2_n -\norm{f-Y}^2_n &= \langle f^*-Y \vert f^*-Y \rangle_n -\langle f-f^*+f^*\!-Y \vert f-f^*+f^*\!-Y\rangle_n \notag\\
        &=  \langle f^*-Y \vert f^*-Y \rangle_n -\langle f-f^*\vert f-f^*\rangle_n \notag\\
        &\qquad+2\langle Y-f^* \vert f-f^*\rangle_n -\langle Y-f^*\vert Y-f^* \rangle_n\notag\\
        &= -\norm{f-f^*}_n^2 + 2\langle \xi\vert f-f^*\rangle_n\,.\label{eq: russo square decomposition}
    \end{align}
    Applying \eqref{eq: russo square decomposition} with $f:=\hat f_n$, the $n$-point non-linear least-square fit, leads to a non positive left hand side and thus
    \[ \norm{\hat f_n -f^*}^2_n\le 2\abs{\langle\xi \vert f-f^* \rangle_n}\,.\] 
    At the same time, for all $n\in\NN^*$, by definition of $\Cs_n^\Gamma$, it holds that for all $g\in\Cs_n^\Gamma$
    \begin{align}
        \norm{\hat f_{n} -f^*}^2_n &\le 2\abs{\langle\xi\vert g-f^*\rangle_n} + 2\abs{\langle \xi \vert \hat f_n-g\rangle_n}\notag\\
        &\le  2\abs{\langle \xi\vert g-f^*\rangle_n} + 2n\sup_{i\le n} \norm{\xi_i}_2C^2_n(g)\,.\label{eq: eluder bound square developement}
    \end{align}
    \end{enumerate}
    Combining  \eqref{eq: eluder bound subgaussian sup} and \eqref{eq: eluder bound square developement}, we obtain, for all $\delta\in(0,1)$, $n\in\NN^*$, and $g\in\Cs_n^\Gamma$,  that
    \begin{align}
    \PP\left(\norm{ \hat f_n -f^*}^2_n \ge 2\abs{\langle\xi \vert g-f^*\rangle_n} + 2nC_n^2(g)\Gamma\sqrt{2\log\left(  \frac{2n}{\delta}\right)}\right)\le \delta \label{eq: eluder cover union bound} 
    \end{align}

    Let us now provide two bounds on $C^1_n(g)$ we will use. For all $n\in\NN^*$, $\delta\in(0,1)$ and $g\in\Cs_n^\Gamma$, let
    \begin{align}
        C^1_n(g) &\le \Gamma^2 + 8L^2(1+\sup_{i\le n}\norm{X_i}^2)\,. \label{eq: Eluder C2 bound 2}\\
        C^1_n(g) &\le \Gamma^2 + \norm{\hat f_n -f^*}^2_n + \norm{g-\hat f_n}^2_n\le C^1_n(\hat f_n) +nC^2_n(g)\,,\label{eq: Eluder C1 bound 1}
    \end{align}

Applying \cref{lemma: eluder self normalised bd on IP} for each $g\in\Cs_n^\Gamma$, by a union bound over $g\in\Cs_n^\Gamma$, we have for any $\delta_0(n)\in(0,1)$ (to be fixed at the end), that
    \begin{align}
        \delta_0(n)&\ge\PP\left(\sup_{g\in\Cs_n^\Gamma} \abs{\langle \xi \vert g- f^*\rangle_n }\ge \Gamma\sqrt{ 2\sup_{g\in\Cs_n^\Gamma}C^1_n(g)\log\left(\frac{\sup_{g\in\Cs_n^\Gamma}C^1_n(g)\Ns_n^\Gamma}{\delta_0(n)}\right) }\right)\,.\notag\\
        \intertext{Applying \eqref{eq: Eluder C2 bound 2} and \eqref{eq: Eluder C1 bound 1} this becomes }
        \delta_0(n)&\ge\PP\vast(\sup_{g\in\Cs_n^\Gamma} \abs{\langle \xi\vert g- f^*\rangle_n }\notag\\
        &\qquad\quad\ge \Gamma\sqrt{ 2(C^1_n(\hat f_n)+n\sup_{g\in\Cs_n^\Gamma}C^2_n(g))\log\left(\frac{(\Gamma^2+8L^2(1+\sup_{i\le n}\norm{X_i}^2))\Ns_n^\Gamma}{\delta_0(n)}\right) }\vast) \notag
        \intertext{and thus}
        \delta_0(n)&\ge\PP\left(\sup_{g\in\Cs_n^\Gamma} \abs{\langle \xi\vert g- f^*\rangle_n }\ge \Gamma\sqrt{2L^1_n(\delta_0(n))}\left(\sqrt{ C^1_n(\hat f_n)}+\sqrt{n\sup_{g\in\Cs_n^\Gamma}C^2_n(g)}\right)\right) \,.\label{eq: eluder sup cover CI}
    \end{align}
    Combining \eqref{eq: eluder cover union bound}  and \eqref{eq: eluder sup cover CI} by a union bound gives us
    \begin{align*}
        \delta_0(n)\ge\PP\Bigg(\norm{ \hat f_n -f^*}^2_n &\ge 2\Gamma\sqrt{2L^1_n\left(\frac{\delta_0(n)}2\right)}\left(\sqrt{ C^1_n(\hat f_n)}+\sqrt{n\sup_{g\in\Cs_n^\Gamma}C^2_n(g)}\right) \\
        &\qquad+ 2nC_n^2(g)\Gamma\sqrt{2\log\left(  \frac{4n}{\delta_0(n)}\right)} \Bigg)\,.
    \end{align*}
    For all $n\in\NN^*$, on the complement of this event (whose probability is at least $1-\delta_0(n)$) we have 
    \begin{align}
        C_n^1(\hat f_n)\le \Gamma^2 + \Gamma\sqrt{2C^1_n(\hat f_n)L_n^1(\delta_0(n)/2)} +h_n^\Gamma\,,\label{eq: eluder dim polynomial}
    \end{align}
    in which
    \begin{align*}
        h_n^\Gamma:= 2\Gamma\left(n\sup_{g\in\Cs_n^\Gamma} C^2_n(g)\sqrt{2\log\left(  \frac{4n}{\delta_0(n)}\right)}+ \sqrt{2n\sup_{g\in\Cs_n^\Gamma}C^2_n(g)L^1_n\left(\frac{\delta_0(n)}2\right)}\right)
        \,.
    \end{align*} 
    Viewing \eqref{eq: eluder dim polynomial} as a second order polynomial in $\sqrt{C^1_n(\hat f_n)}$, we obtain via its roots that
    \begin{align*}
        \sqrt{C_n^1(\hat f_n)}&\le \Gamma \sqrt{L_n^1(\delta_0(n)/2)} + \sqrt{\left(\Gamma \sqrt{L_n^1(\delta_0(n)/2)}\right)^2 +4(\Gamma^2+h_n^\Gamma)}\\
        &\le 2\Gamma \sqrt{L_n^1(\delta_0(n)/2)} +2\sqrt{\Gamma^2+h_n^\Gamma}\,.
    \end{align*}
    Since $\norm{\hat f_n -f^*}_n\le \sqrt{C^1_n(\hat f_n)}$ by definition of $C_n^1(\hat f_n)$, we have 
    \begin{align*}
        \norm{\hat f_n -f^*}_n&\le 2\sqrt{\Gamma^2+ 2\Gamma\left(n\sup_{g\in\Cs_n^\Gamma} C^2_n(g)\sqrt{2\log\left(  \frac{4n}{\delta_0(n)}\right)}+ \sqrt{2n\sup_{g\in\Cs_n^\Gamma}C^2_n(g)L^1_n\left(\frac{\delta_0(n)}2\right)}\right)}\\
        &\quad+2\Gamma \sqrt{L_n^1(\delta_0(n)/2)}\,.
    \end{align*} 
    Therefore, letting
    \begin{align*}
         \Ec^1_n(\delta):=\bigg\{ &\norm{\hat f_n -f^*}_n\le2\Gamma \sqrt{L_n^1\left(\frac\delta 2\right)}  \\
         &\qquad +2\sqrt{\Gamma^2+ 2\Gamma\left(n\sup_{g\in\Cs_n^\Gamma} C^2_n(g)\sqrt{2\log\left(  \frac{4n}{\delta}\right)}+ \sqrt{2n\sup_{g\in\Cs_n^\Gamma}C^2_n(g)L^1_n\left(\frac{\delta}2\right)}\right)}
         \bigg\},
    \end{align*}
    we have, for all $n\in\NN^*$, that $\PP(\Ec^1_n(\delta_0(n)))\ge \delta_0(n)$. Letting $\delta_0(n)= \frac{6}{\pi^2n^2}\delta$, by a union bound we obtain 
    \[\PP\left( \bigcap_{n\in\NN^*}\Ec^1_n(\delta_0(n)) \right)\ge 1 - \delta \frac{6}{\pi^2}\sum_{n=1}^\infty \frac1{n^2}= 1-\delta\,.\]
    Noting that $\Ec^0_n(\delta)=\Ec^1_n(\delta_0(n))$ for all $\delta\in(0,1)$ and $n\in\NN^*$ completes the proof.
\end{proof}

In the proof of \cref{lemma: eluder learning big bound lemma}, we used self-normalised inequalities to generalise the results of \cite{russo2013eluder} to unbounded states. We now incorporate the high probability bound of \cref{asmp: eluder} and formalise confidence sets, which will prove \cref{thm: eluder learning bound}. \Cref{thm: eluder learning bound} can then be specified for our setting by merging it with the results of \cref{app:stability} in \cref{prop: conf sets OvR}.

For $\delta\in(0,1)$, let $\beta_0\in\RR_+$ and let us define the sequence $(\Cc_n(\delta))_{n\in\NN}$ in which 
\begin{align}
    \Cc_n(\delta):= \left\{f\in\Fs: \snorm{f-\hat f_n}_n \le \beta_n\right\}\label{eq: eluder def conf sets}
\end{align}
with 
\begin{align}
    \beta_n(\delta):= \beta_0\vee 2\Gamma\left(\sqrt{1+ 2\left(\sqrt{2\Gamma\log\left(  \frac{8n}{\delta}\right)}+ \sqrt{2L^0_n\left(\frac{\delta}4\right)}\right)}+\sqrt{L_n^0\left(\frac\delta 4\right)}\right)\,.\label{eq: def beta_n appendix}
\end{align}

\begin{theorem}\label{thm: eluder learning bound}
    Under \cref{asmp: eluder}, we have for all $\delta\in(0,1)$
    \[\PP\left(\left\{\bigcap_{n\in\NN^*}\{ f^*\in\Cc_n(\delta)\}\right\}\bigcap \left\{\sup_{n\in\NN^*}\frac{\norm{X_n}}{H_\delta(n)}\le 1\right\}\right)\le \delta\]
\end{theorem}

\begin{proof}
    Fix $\delta\in(0,1)$, and assume $\omega\in\{\omega'\in\Omega: \sup_{n\in\NN^*}\snorm{X_n{(\omega')}}_2/H_\delta(n)\le 1\}$. In this case we have the following bound, for all $n\in\NN^*$
    \begin{align}
        2n\min_{g\in\Cs_n^\Gamma}C^2_n(g) &\le 2\Gamma^2\notag
    \end{align}
    by definition of $\Cs_n^\Gamma$ as a $\Gamma^2 n^{-1}$ cover on $\Bc_2(H_\delta(n))$. Therefore, the event 
    \[\left\{\bigcap_{n\in\NN^*}\Ec_n^0(\delta)\right\}\bigcap\left\{\sup_{n\in\NN^*}\frac{\snorm{X_n}_2}{H_{\delta}(n)}\le 1\right\}\]
    is contained in the event 
    \[ \Ec^0(\delta):= \left\{\bigcap_{n\in\NN^*}\left\{\norm{f^*-\hat f_n}_n\le \beta_n(2\delta)\right\}\right\}\bigcap\left\{\sup_{n\in\NN^*}\frac{\norm{X_n}_2}{H_\delta(n)}\le 1\right\} \,.\]
    By \cref{lemma: eluder learning big bound lemma}, \cref{asmp: eluder}, and a union bound,
    $\PP\left(\Ec^0(\delta) \right)\ge 1-2\delta$, and we obtain the result by \eqref{eq: eluder def conf sets} and \eqref{eq: def beta_n appendix}, i.e. by definition of $\Cc_n(\delta)$.
\end{proof}

\ConfSetsVanRoy*

\begin{proof}
    The proof follows by applying \cref{thm: eluder learning bound} to this setting. Where $(X_i)_{i\in\NN}:= ((X_{\tau_i}^{\varpi,\theta^*},\varpi_{\tau_i}))_{i\in\NN}$, $(Y_i)_{i\in\NN}:=(X_{\tau_{i+1}}^{\varpi,\theta^*}-X_{\tau_{i}}^{\varpi,\theta^*})_{i\in\NN}$, $\Fs:=\Fs_\Theta$ and with $(\xi_{n+1})_{n\in\NN}$ and $(\beta_{n}(\delta))_{n\in\NN^*}$ as defined in \cref{sec: prelim} and \eqref{eq: def beta_n main text} respectively. This sets $\Gamma=\snorm{\Sigma}_\op=\ve^{\frac12}\snorm{\bar\Sigma}_\op$. The only subtlety is that the process $X^{\varpi,\theta^*}$ is measured at random times, but since these times are independent of anything else, and the process is almost surely constant between them, they do not affect the proof.
\end{proof}

\subsection{Widths of confidence sets}\label{app:widths}

In \cref{app: conf sets}, we showed how to design confidence sets along a trajectory of $X^{\alpha,\theta}$ for learning $\mu$ by using NLLS to minimise a fit error of the form
\[ \sum_{n=\deb}^{N} \norm{\mu_1(X_{\tau_n}^{\alpha,\theta^*},\alpha_{\tau_n}) -\mu_2(X_{\tau_n}^{\alpha,\theta^*},\alpha_{\tau_n}) }\,,\]
for $(\mu_1,\mu_2)\in\Cc_N(\delta)$ and $N\in\NN^*$. 
When analysing the regret of such a learning algorithm this is not sufficient: instead of the fit error, we need to control a prediction error of the form 
\[ \sum_{n=\deb}^{N} \norm{\mu_{\theta_n}(X_{\tau_n}^{\alpha,\theta^*},\alpha_{\tau_n}) -\mu_{\theta^*}(X_{\tau_n}^{\alpha,\theta^*},\alpha_{\tau_n}) }\,,\]
for $(\mu_{\theta_n})_{n\in\NN}\subset\Fs_\theta$ such that $\mu_{\theta_n}\in\Cc_n(\delta)$ for all $n\in\NN$. The difference is that $\mu_{\theta_n}$ changes over time, so that the sum counts the errors in predicting the next state made by the sequence $(\mu_{\theta_n})_{n\in\NN}$.

In fact, since we will want to implement lazy-updates, we will need a more general result where the $\mu_{\theta_n}$ are not all in their respective $\Cc_n(\delta)$ but rather are from a piece-wise constant sequence with $\mu_{\theta_n}:=\mu_{\theta_{k(n)}}\in\Cc_{k(n)}(\delta)$, where $k(n)\le n$ for all $n\in\NN$. 
Therefore, as in \cref{app: conf sets}, we begin by showing a general result in the learning framework of \cite{russo2013eluder} (\cref{prop: conf set width generic}), then apply it to our setting to prove \cref{cor: conf set width no lazy updates}. Using the notation of \cref{app: conf sets}, let 
$\Fs$ be a function class of functions from $\RR^d\to\RR^d$, and recall the arbitrary sequence $(X_n)_{n\in\NN}\subset\RR^d$.



The $\epsilon$-eluder dimension of a function class $\Fs$, for $\epsilon\in\RR_+$, introduced in \cite{russo2013eluder} is a notion of dimension which is perfectly tailored to converting fit errors into prediction errors. We defer to \cite{russo2013eluder} for its technical definition.
Unlike \cite{russo2013eluder}, we must adapt our eluder dimension to work with unbounded functions on unbounded processes. Failing to do so would lead our results to be largely vacuous since the eluder dimension of $\Fs$ might be infinite for any $\epsilon$. 

We work with a modified eluder dimension, which takes three arguments: a function class $\Fs$ whose elements have for domain a set $\Xc\subset\RR^d$; a set $S\subset\Xc$; and $\epsilon\in\RR_+$. Our modified eluder dimension is the $\epsilon$-eluder dimension of $\{f\vert_S: f\in \Fs\}$, the class containing the restrictions to $S$ of elements of $\Fs$, which we denote by $\dime^S(\Fs,\epsilon)$. In this way, the eluder dimension of \cite{russo2013eluder} is $\dime^\Xc(\Fs,\epsilon)$. For $n\in\NN^*$, let $B_n:=\Bc_2(\sup_{i\in[n]}\norm{X_i})$ and, for any $u\in\RR_+$, let us define the sequence $(\mathrm{d}_{\mathrm{E},n}^{\Fs}(u))_{n\in\NN^*}$, in which 
\[\mathrm{d}_{\mathrm{E},n}^{\Fs}(u):=\dime^{B_n}\left(\Fs, \frac {2u}{\sqrt{n}}\right)\]
for all $n\in\NN^*$ and $u\in\RR_+$.

\begin{restatable}{proposition}{ConfSetWidthGeneric}\label{prop: conf set width generic}
    Let $(\tilde\beta_i)_{i\in\NN}$ be a non-decreasing positive real-valued sequence, $(\tilde f_i)_{i\in\NN}$, and $(\Fs_{i})_{i\in\NN}$ be a sequence of subsets of $\Fs$ of the form $\Fs_{i}:=\{f\in\Fs : \snorm{f-\tilde f_i}_i\le \tilde\beta_i\}$. Then, for any $n\in\NN$, we have
    \begin{align}
        \sum_{i=\deb}^n \sup_{(f,f')\in\Fs_{n}^2} \norm{f(X_i) - f'(X_i) } &\le 2\tilde\beta_n\sqrt{\mathrm{d}_{\mathrm{E},n}^{\Fs}(\tilde\beta_0)n} +\mathrm{d}_{\mathrm{E},n}^{\Fs}(\tilde\beta_0)\sup_{i\in[n]}\norm{X_i}\,,\label{eq: generic conf set width order 1}
        \intertext{and}
        \sum_{i=\deb}^n  \sup_{(f,f')\in\Fs_{n}^2} \norm{f(X_i) - f'(X_i) }^2 &\le 4\tilde\beta_n^2\mathrm{d}_{\mathrm{E},n}^{\Fs}(\tilde\beta_0)\left(3 + \log\left(\frac{n\sup_{i\in[n]}\norm{X_i}}{16\tilde\beta_n^4(\mathrm{d}_{\mathrm{E},n}^{\Fs}(\tilde\beta_0))^2}\right)\right)\notag\\
        &\qquad + 2\mathrm{d}_{\mathrm{E},n}^{\Fs}(1+2\tilde\beta_n^2\mathrm{d}_{\mathrm{E},n}^{\Fs}(\tilde\beta_0))(1+\sup_{i\in[n]}\norm{X_i})\,.\label{eq: generic conf set width order 2}
    \end{align}
\end{restatable}

To prove \cref{prop: conf set width generic}, the key result of \cite{russo2013eluder} we leverage is \cref{lemma: eluder dimension russo} which we combine with  two functional inequalities given in \cref{lemma: widths functional inequality}.

For a function class $\Fs$ with domain $\Xc\subset\RR^d$, and any $x\in\Xc$, let us define
\[ \Lambda(\Fs; x) = \sup_{(f_1,f_2)\in\Fs^2}\norm{f_1(x)-f_2(x)}\,.\]
The quantity $\Lambda(\Fs,x)$ is the maximal prediction gap at $x$ between two functions in $\Fs$. Bounding the prediction error along $(X_i)_{i\in\NN}$ of a sequence of function classes $(\Fs_i)_{i\in\NN}\subset\Fs$ means bounding $\sum_{i=1}^n\Lambda(\Fs_i,X_i)$ in terms of $n\in\NN$. 

\begin{restatable}{lemma}{BoundedSumWidths}[{\citep[Prop.3]{russo2013eluder}}]\label{lemma: eluder dimension russo}
    Let $(\tilde f_i)_{i\in\NN}$ be a sequence of elements of $\Fs$, $(\Fs_{i})_{i\in\NN}$ be a sequence of subsets of $\Fs$ of the form $\Fs_{i}:=\{f\in\Fs : \snorm{f-\tilde f_i}_i\le \tilde\beta_i\}$.
    For any $\epsilon\in(0,1)$ and $n\in\NN$, one has  
    \[ \sum_{i=1}^n \1_{\{\Lambda(\Fs_{i};X_i)>\epsilon\}} \le \left(\frac{4\tilde\beta_n^2}{\epsilon^2}+1\right)\dime^{B_n}\left(\Fs,\epsilon\right)\,.\]
\end{restatable}

\begin{proof}
    Following the proof of \citep[Prop.3]{russo2013eluder}, the only modification involves the bound $\snorm{\overline{f} -\underline{f}}_n \le \tilde\beta_n$, for any $(\overline{f}, \underline{f})\in\Fs_{n}^2$, which holds by assumption.  
\end{proof}

\begin{lemma}\label{lemma: widths functional inequality}
    Let $(x_i)_{i\in\NN}\in\RR_+^{\NN^*}$. Assume there is a family of positive sequences $((\zeta^\epsilon_n)_{n\in\NN})_{\epsilon\in\RR_+}$ and a family of positive constants $(\chi^\epsilon)_{\epsilon\in\RR_+}$ such that, for any $n\in\NN^*$ and $\epsilon >0$,
    \begin{align}
        \sum_{i=1}^n \1_{\{x_i>\epsilon\}}\le \frac{\zeta^\epsilon_n}{\epsilon^{2}} +\chi^\epsilon\, \label{eq: asmp of widths functional inequality}
    \end{align}
    then the following two inequalities hold
    \begin{align}
        \sum_{i=1}^n x_i &\le 2\sqrt{n\zeta^\epsilon_n} + \chi^\epsilon\sup_{i\in[n]}x_i \label{eq: order 1 function ineq}\\
        \sum_{i=1}^n x_i^2&\le\zeta_n^\epsilon \left(3 + \log\left(\frac{n\sup_{i\in[n]}x_i^2}{(\zeta_n^\epsilon)^2}\right)\right) + \chi^\epsilon(2+\zeta_n^\epsilon)(1+\sup_{i\in[n]}x_i^2)\,.\label{eq: order 2 function ineq}
    \end{align}
\end{lemma}

\begin{proof}\hfill
\begin{enumerate}
    \item[{\rm i.}] For $\epsilon>0$, we have by \eqref{eq: asmp of widths functional inequality}
    \begin{align*}
        \sum_{i=1}^n (x_i-\epsilon)\1_{\{x_i>\epsilon\}} &= \sum_{i=1}^n\int_\epsilon^{x_i}\1_{\{x_i>u\}}\de u\\
        &\le \int_\epsilon^{\sup_{i\in[n]}x_i}\sum_{i=1}^n \1_{\{x_i>u\}}\de u\\
        &\le \int_\epsilon^{\sup_{i\in[n]}x_i} \frac{\zeta^\epsilon_n}{u^2}+\chi^\epsilon \de u\\
        &= \chi\sup_{i\in[n]}x_i -\frac{\zeta^\epsilon_n}{ \sup_{i\in[n]}x_i} - \chi^\epsilon\epsilon + \frac{\zeta^\epsilon_n}{\epsilon}\,,
    \end{align*}
    and thus 
    \begin{align}
        \sum_{i=1}^n (x_i-\epsilon)\1_{\{x_i>\epsilon\}} \le \frac{\zeta^\epsilon_n}{\epsilon} + \chi^\epsilon\sup_{i\in[n]}x_i\,. \label{eq: function ineq 1 intermediate}
    \end{align}
    Combining \eqref{eq: function ineq 1 intermediate} with  
    \[ \sum_{i=1}^n (x_i-\epsilon) \le \sum_{i=1}^n (x_i-\epsilon)\1_{\{x_i>\epsilon\}}\]
    yields
    \[ \sum_{i=1}^n x_i \le n\epsilon + \frac{\zeta^\epsilon_n}{\epsilon} + \chi^\epsilon\sup_{i\in[n]}x_i\,.\]
    Setting $\epsilon = \sqrt{\zeta^\epsilon_n/n}$ yields \eqref{eq: order 1 function ineq}.
    \item[{\rm ii.}]  To prove \eqref{eq: order 2 function ineq}, we iterate the bound \eqref{eq: function ineq 1 intermediate}
    \begin{align*}
        \sum_{i=1}^n (x_i-\epsilon)^2\1_{\{x_i>\epsilon\}} &= 2\sum_{i=1}^n \int_\ve^{x_i}(x_i-u)\1_{\{x_i>u\}}\de u\\
        &\le 2\sum_{i=1}^n\int_\epsilon^{\sup_{i\in[n]}x_i}(x_i-u)\1_{\{x_i>u\}} \de u \\
        &\le 2\int_\epsilon^{\sup_{i\in[n]}x_i} \frac{\zeta^\epsilon_n}{\epsilon} + \chi^\epsilon\sup_{i\in[n]}x_i \de u \\
        &\le 2\left(\chi(\sup_{i\in[n]}x_i^2- \sup_{i\in[n]}x_i\epsilon) + \zeta_n^\epsilon \log\left(\frac{\sup_{i\in[n]}x_i}{\epsilon}\right)\right)\\
        &\le 2\zeta_n^\epsilon \log\left(\frac{\sup_{i\in[n]}x_i}{\epsilon}\right) +  2\chi^\epsilon\sup_{i\in[n]}x_i^2\,.
    \end{align*}
        Now, by some algebraic manipulations of $\sum_{i=1}^n x_i^2$, completing the square, discarding negative terms, and using \eqref{eq: function ineq 1 intermediate} in the third step, we get
    \begin{align*}
        \sum_{i=1}^n x_i^2 &\le \sum_{i=1}^n x_i^2\1_{\{x_i>\epsilon\}} + \epsilon^2\sum_{i=1}^n\1_{\{x_i>\epsilon\}}\\
        &\le \sum_{i=1}^n (x_i-\epsilon)^2\1_{\{x_i>\epsilon\}} +2\epsilon\sum_{i=1}^nx_i\1_{\{x_i>\epsilon\}} + n\epsilon^2 \\
        &\le 2\zeta_n^\epsilon \log\left(\frac{\sup_{i\in[n]}x_i}{\epsilon}\right) +  2\chi^\epsilon\sup_{i\in[n]}x_i^2 +\epsilon\left(\frac{\zeta_n^\epsilon}{\epsilon} + \chi^\epsilon\sup_{i\in[n]}x_i + \epsilon n\right) + n\epsilon^2\,.
    \end{align*}
    Taking $\epsilon=\zeta^\epsilon_n/\sqrt{n}$ and factoring, using also $u\le 1+u^2$ for $u\in\RR$, yields
    \[\sum_{i=1}^n x_i^2\le\zeta_n^\epsilon \left(3 + \log\left(\frac{n\sup_{i\in[n]}x_i^2}{(\zeta_n^\epsilon)^2}\right)\right) + \chi^\epsilon(2+\zeta_n^\epsilon)(1+\sup_{i\in[n]}x_i^2)\,.\]
\end{enumerate}\end{proof}

\ConfSetWidthGeneric*

\begin{proof}
    The proof consists in applying \cref{lemma: widths functional inequality} to \cref{lemma: eluder dimension russo}, with
    $x_i = \Lambda(\Fs_i,X_i)$, $\zeta_n^\epsilon=4\tilde\beta^2\dime^{B_n}(\Fs,\epsilon)$ ($B_n:=\Bc_2(\sup_{i\in[n]}\norm{X_i})$), and $\chi^\epsilon= \dime^{B_n}(\Fs,\epsilon)$. When we set the value of $\epsilon$ in the proof of \cref{lemma: widths functional inequality}, $\chi^\epsilon$ becomes
    \[\dime^{B_n}\left(\Fs,\sqrt{frac{4\tilde\beta^2_n}{n}}\right)\le \dime^{B_n}\left(\Fs,\sqrt{\frac{4\tilde\beta_0^2}{n}}\right)\]
    as $(\tilde\beta_n)_{n\in\NN}$ is non-decreasing and the eluder dimension is decreasing in its third argument. An analogue remark holds for $\zeta^\epsilon_n$.
    We can thus substitute $\zeta_n^\epsilon = 4\tilde\beta^2_n\mathrm{d}_{\mathrm{E},n}^{\Fs}(\tilde\beta_0)$ and $\chi^\epsilon =\mathrm{d}_{\mathrm{E},n}^{\Fs}(\tilde\beta_0)$ in \eqref{eq: order 1 function ineq} and \eqref{eq: order 2 function ineq}, which gives the result.
\end{proof}

We now apply \cref{prop: conf set width generic} to our setting. For $n\in\NN^*$, let us recall the shorthand notation
\begin{align} 
\mathrm{d}_{\mathrm{E},n}:=\dime^{B_n}\left(\Fs_\Theta, 2\sqrt{\frac {\ve}{n}}\right) \label{eq: def eluder dim in appendix}
\end{align}
in which we extended the notation from $(X_i)_{i\in\NN}$ to $X^{\alpha,\theta}$ in the evident manner. 

 \begin{restatable}{proposition}{ConfSetWidthNoLazyUpdate}\label{prop: conf set width no lazy updates}
    Under \cref{asmp: basics,asmp: both asmp joint}, for any $(\alpha,\theta)\in\Ac\x\Theta$ and $t\in\RR_+$, any non-decreasing positive real-valued sequence $(\tilde\beta_n)_{n\in\NN}$, any $(\tilde\mu_n)_{n\in\NN}\subset\Fs_\Theta$, and any sequence $(\Fs_n)_{n\in\NN}$ of subsets of $\Fs_\theta$ of the form 
    \[ \Fs_n=\left\{\mu\in\Fs_\Theta: \sqrt{\sum_{i=0}^{n-1} \norm{\mu_n(X^{\alpha,\theta}_{\tau_{i}},\alpha_{\tau_i}) - \tilde\mu_n(X^{\alpha,\theta}_{\tau_{i}},\alpha_{\tau_i})}_2^2}\le \tilde \beta_n   \right\}\,,\]
    we have
    \begin{align}
        \sum_{n=\deb}^{N_t} \sup_{(\mu_1,\mu_2)\in\Fs_n} \norm{\mu_1(X_{\tau_n}^{\alpha,\theta},\alpha_{\tau_n}) - \mu_2(X_{\tau_n}^{\alpha,\theta},\alpha_{\tau_n})  } &\le 2\beta_{N_t}\sqrt{\mathrm{d}_{\mathrm{E},{N_t}}} +\mathrm{d}_{\mathrm{E},{N_t}}\sup_{s\le t}\norm{X_s^{\alpha,\theta}}\,,\label{eq: conf set width order 1}
    \end{align}
    and
    \begin{align}
        &\sum_{n=\deb}^{N_t} \sup_{(\mu_1,\mu_2)\in\Fs_n} \norm{\mu_1(X_{\tau_n}^{\alpha,\theta},\alpha_{\tau_n}) - \mu_2(X_{\tau_n}^{\alpha,\theta},\alpha_{\tau_n})  }^2\notag\\
        &\le 4\beta_{N_T}^2\mathrm{d}_{\mathrm{E},{N_t}}\left(3 + \log\left(\frac{N_t\sup_{s\le t}\norm{X_s^{\alpha,\theta}}}{16\beta_{N_t}^4\mathrm{d}_{\mathrm{E},{N_t}}^2}\right)\right)+ 2\mathrm{d}_{\mathrm{E},N_t}(1+2\beta_{N_t}^2\mathrm{d}_{\mathrm{E},N_t})(1+\sup_{s\le t}\norm{X_s^{\alpha,\theta}}^2).\label{eq: conf set width order 2}
    \end{align}
\end{restatable}


\begin{proof}
    Immediate by applying \cref{prop: conf set width generic} to our setting, as we did in the proof of \cref{prop: conf sets OvR}.
\end{proof}

Under the event of \cref{prop: conf sets OvR}, which ensures that $\theta^*\in\cap_{n\in\NN}\Cc_n(\delta)$, we can derive from \cref{prop: conf set width no lazy updates} a bound on the prediction error relative to the true dynamics $X^{\alpha,\theta^*}$ generated by the control $\alpha\in\Ac$, in particular we are interested in $\alpha=\varpi$.

\ConfSetWidthNoLazyUpdateMAINTEXT*
\begin{proof}
    This follows from \cref{prop: conf set width no lazy updates} by choosing $(\tilde\beta_n)_{n\in\NN}=(\beta_n(\delta))_{n\in\NN}$ and $(\Fs_n)_{n\in\NN}=(\Cc_n(\delta))_{n\in\NN}$, i.e. choosing $(\tilde\mu_n)_{n\in\NN}=(\mu_{\hat\theta_n})_{n\in\NN}$, the NLLS fit on $n$ points. It is key to notice that these choices of $(\tilde\beta_n)_{n\in\NN}$, $(\Fs_n)_{n\in\NN}$, and $(\tilde\mu_n)_{n\in\NN}$ are adapted to $\Fb$, and therefore we can apply \cref{prop: conf set width no lazy updates} on the event of \cref{prop: conf sets OvR} without issues. This yields
        \begin{align*}
        \sum_{n=\deb}^{N_t}\norm{\mu_{\hat\theta_n}(X_{\tau_n}^{\alpha,\theta^*},\alpha_{\tau_n}) - \mu_{\theta^*}(X_{\tau_n}^{\alpha,\theta^*},\alpha_{\tau_n})  } &\le 2\beta_{N_t}(\delta)\sqrt{\mathrm{d}_{\mathrm{E},{N_t}}} +\mathrm{d}_{\mathrm{E},{N_t}}H_\delta(N_T)\,,
        \intertext{and}
        \sum_{n=\deb}^{N_t} \norm{\mu_{\hat\theta_n}(X_{\tau_n}^{\alpha,\theta^*},\alpha_{\tau_n}) - \mu_{\theta^*}(X_{\tau_n}^{\alpha,\theta^*},\alpha_{\tau_n})    }^2&\le 4\beta_{N_T}(\delta)^2\mathrm{d}_{\mathrm{E},{N_t}}\left(3 + \log\left(\frac{N_t H_\delta(N_t)}{16\beta_{N_t}(\delta)^4\mathrm{d}_{\mathrm{E},{N_t}}^2}\right)\right)\notag\\
        &\qquad + 2\mathrm{d}_{\mathrm{E},N_t}(1+2\beta_{N_t}(\delta)^2\mathrm{d}_{\mathrm{E},N_t})(1+H_\delta^2(N_t))\,.
    \end{align*}
    To obtain the estimates of \eqref{eq: main text conf set width order 1}--\eqref{eq: main text conf set width order 2}, it suffices to recall the definitions of $\beta_n(\delta)$ (i.e. \eqref{eq: def beta_n main text}) and $H_\delta(n)$ (i.e. \eqref{eq: def H delta}).
\end{proof}

%% file: appendix_control.tex
\section{Planning and Diffusive Limit Approximation}\label{app:ctrl}

Our work builds upon \citep{abeille_diffusive_2022}, but with specialised results for our setting. This paper recovers the key results of this section (\cref{prop: PropertiesRhoJump,prop: PropertiesRhoDiff,prop: approx diff limit}) under a stronger and more abstract set of assumptions. For the comfort of the reader we thus present the necessary steps to extend their results to our assumptions. Since our assumptions do not directly subsume theirs, we exhibit in each case from \cref{asmp: basics,asmp: both asmp joint} how to recover the keystone results which underpin the technical arguments of \citep{abeille_diffusive_2022}.

We begin by the well-posedness results for the pure jump case (\cref{prop: PropertiesRhoJump}) and the diffusive limit case (\cref{prop: PropertiesRhoDiff}), and then focus on the approximation result linking the two regimes (\cref{prop: approx diff limit}). In \cite{abeille_diffusive_2022}, \cref{prop: PropertiesRhoJump} corresponds to Theorem 2.3. and Remark 2.4. In \cref{app: proof of PJ HJB}, we show how it follows from \cref{asmp: both asmp joint,asmp: basics} by proving the two intermediary results used in \cite{abeille_diffusive_2022} to prove the result. 

\PropertiesRhoJump*

In \cite{abeille_diffusive_2022}, \cref{prop: PropertiesRhoDiff} corresponds to Theorem 3.4. In \cref{app: proof diff HJB}, we show that it also follows from \cref{asmp: both asmp joint,asmp: basics} by proving that \cite[Assumption 5]{abeille_diffusive_2022} holds under \cref{asmp: basics,asmp: both asmp joint}.

\PropertiesRhoDiff*

\begin{remark}
    \cref{prop: PropertiesRhoDiff}.({\rm iii.}) is not stated as is in \cite[Thm.~3.4]{abeille_diffusive_2022}, but it follows from it by the same arguments as \cite[Remark 2.4]{abeille_diffusive_2022}. 
\end{remark}

\Cref{prop: PropertiesRhoJump,prop: PropertiesRhoDiff} together ensure that both the prelimit and limit regimes are well posed, while \cref{prop: approx diff limit} gives the rate of convergence of the control problems along this limit. This result is essentially contained in the proof of \cite[Thm. 3.6]{abeille_diffusive_2022}, but since its statement is different, we include a proof for completeness in \cref{subsec:approx diff limit}.

\PropRhoApprox*

\subsection{Proof of \texorpdfstring{\cref{prop: PropertiesRhoJump}}{Proposition \ref{prop: PropertiesRhoJump}}}\label{app: proof of PJ HJB}

In \cite{abeille_diffusive_2022}, Theorem 2.3 and Remark 2.4 follow from Lemmas A.1 and A.2, which respectively give a mixing condition and a moment bound for $X^{\alpha,\theta}$. We already proved \cite[Lemma A.2]{abeille_diffusive_2022} in \cref{lemma: moment boundedness condition from paper}. Moreover, \cref{lemma: intermediate lipschitzness jump} which reproduced \cite[Lemmas A.1]{abeille_diffusive_2022} holds with only minor modifications of the proof from \cite{abeille_diffusive_2022}. 

\BoundedMomentStateJump*

\begin{restatable}{lemma}{LemmaLipschitzProcess}\label{lemma: intermediate lipschitzness jump}
    For any $(x,x')\in\RR^d\x\RR^d$, $\theta\in\parset$, and $\alpha\in\Ac$, 
        \[ \EE\left[\snorm{X_t^{x,\alpha,\theta}-X_t^{x',\alpha,\theta}}\right]\le \frac{L_\Vs}{\ell_\Vs}\norm{x-x'}e^{-\cf_\Vs t}\,\]
    for any $t\in[0,+\infty)$.
\end{restatable}

\begin{proof}
    We can follow the proof of \cite{abeille_diffusive_2022} using \cref{asmp: both asmp joint} directly without resorting to the higher order Lyapunov function $\zeta$ which they use.
\end{proof}

\subsection{Proof of \texorpdfstring{\cref{prop: PropertiesRhoDiff}}{Proposition \ref{prop: PropertiesRhoDiff}}}\label{app: proof diff HJB}

\Cref{prop: PropertiesRhoDiff}, such as it is stated in \cite[Thn~3.4.]{abeille_diffusive_2022} relies on their Assumption 5. This assumption contains two conditions, which we will show respectively in \cref{lemma: intermediate lipschitzness diffusion,lemma: contraction in p power}.

As detailed in \cite[Remark 3.2.(i)]{abeille_diffusive_2022}, the first condition can be shown by proving an analogue of \cite[Lemma A.1]{abeille_diffusive_2022} for the diffusive limit process \eqref{eq: def diffusion}. In terms of arguments of the proof, this analogue requires only a change in the stochastic generator used in Itô's Lemma\footnote{For a general overview of this sort of stability results and of Stochastic Lyapunov conditions in the diffusive case, see e.g. \citep[{$\S$~5.7}]{khasminskii_stochastic_2012}.}. In the proof of \cref{lemma: intermediate lipschitzness diffusion}, we, therefore, show how to adapt \cite[Lemma A.1]{abeille_diffusive_2022} to the generator of the diffusion under \cref{asmp: both asmp joint,asmp: basics}.

In the proof of \cite[Lemma A.1]{abeille_diffusive_2022}, there are two key steps. First, study the discounted version of the control problem, and show that it is equi-Lipschitz continuous in the discount, which rests on the result in \cref{lemma: intermediate lipschitzness diffusion}. Then one takes the vanishing discount limit in the HJB equation using the theory of viscosity solutions to complete the proof.

\begin{lemma}\label{lemma: intermediate lipschitzness diffusion}
    For any $(x_0,x_0')\in\RR^d\x\RR^d$, $\theta\in\parset$, $\alpha\in\Ac$,  
        \[ \EE\left[\norm{\bar X_t^{x,\alpha,\theta}-\bar X_t^{x',\alpha,\theta}}\right]\le \frac{L_\Vs}{\ell_\Vs}\norm{x-x'}e^{-\cf_\Vs t}\,\]
        for any $t\in[0,+\infty)$.
\end{lemma}

\def\erf{\mathrm{erf}}
\begin{proof} If $x_0=x_0'$, this is trivially true by pathwise-uniqueness, so we suppose $x_0\neq x_0'$. Let us consider $(x_1,x_2)\in\RR^d\x\RR^d$ with $x_1\neq x_2$. By a Taylor expansion in \eqref{eq:lyapunov asmp joint on jump problem}, we obtain as $\ve\to0$
        \begin{align}
             (\bar\mu(x_1,a)-\bar\mu(x_2,a))^\top\nabla \Vs(x_1-x_2) \le -\cf_\Vs\Vs(x_1-x_2)\,.\label{eq: diffusion contraction condition}
        \end{align}
        The Lyapunov function $\Vs$ is not differentiable at $0$, so we will construct an approximating sequence for it. Let $\erf$ denote the error function and let $\Vs_\iota:= \Vs\erf(\iota \Vs)$ for $\iota>0$. Note that $\Vs_\iota\in\Cc^1(\RR^d;\RR_+)$ and $\Vs_\iota$ is Lipschitz, let us show that it satisfies \eqref{eq: diffusion contraction condition} everywhere.

        Let $z:= x_1-x_2$. Since $z\neq0$ we have
        \[\nabla\Vs_\iota(z)=\nabla\Vs(z)\left(\erf(\iota \Vs(z)) + \frac{2\iota}{\sqrt{\pi}}\Vs(z)e^{-\iota^2\Vs^2(z)}\right)\,.\]
        By \cref{asmp: both asmp joint}, this implies that 
        \begin{align}
            (\bar\mu_\theta(x_1,a) -\bar\mu_\theta(x_2,a))^\top\nabla\Vs_\iota(z) &\le -\cf_\Vs\Vs(z)\erf(\iota\Vs(z))-\frac{2\iota}{\sqrt{\pi}}\cf_\Vs\Vs(z)^2e^{-\iota^2\Vs^2(z)}\notag\\
            &\le -\cf_\Vs\Vs_\iota(z)\,.\label{eq: contraction v iota}
        \end{align}
        Since $\nabla\Vs_\iota$ is continuous in $z$, and so is the right-hand side, we can let $\norm{z}\to0$ and conclude the bound also holds for $x_1=x_2$. 

        We now apply It\^o's lemma for the process $\bar X^{x,\alpha,\theta}-\bar X^{x',\alpha,\theta}$ to $\Vs_\iota$. Using \eqref{eq: contraction v iota}, this yields, for $t\ge t_0\ge 0$,
        \begin{align*}
            \EE\bigg[&\Vs_\iota\Big(\bar X_t^{x_0,\alpha,\theta} - \bar X_t^{x_0',\alpha,\theta} \Big)\bigg]\\
            &\le  \EE\bigg[\Vs_\iota\Big(\bar X_{t_0}^{x_0,\alpha,\theta} - \bar X_{t_0}^{x_0',\alpha,\theta}  \Big)\bigg]\\
            &+\EE\left[\int_{t_0}^t \!\!\bigg(\!\bar\mu_\theta\Big(\bar X_s^{x_0,\alpha,\theta} ,  \alpha_s\Big) - \bar\mu_\theta\Big(\bar X_s^{x_0',\alpha,\theta} ,\alpha_s\Big)\!\!\bigg)^{\!\!\top}\!\!\nabla \Vs_\iota\Big(\bar X_s^{x_0,\alpha,\theta}  - \bar X_s^{x_0',\alpha,\theta} \Big)\de s\right] \\
            &\le  \EE\bigg[\Vs_\iota\Big(\bar X_{t_0}^{x_0,\alpha,\theta} - \bar X_{t_0}^{x_0',\alpha,\theta}  \Big)\bigg]  - \int_{t_0}^t \cf_\Vs\EE\left[\Vs_\iota\Big(X_s^{x_0,\alpha,\theta} -X_s^{x_0',\alpha,\theta} \Big)\right] \de s\,. 
        \end{align*}
        We conclude by the same ODE comparison argument as in the proof of \cref{lemma: moment boundedness condition from paper} and then pass to the limit as $\iota\to0$ to obtain the claimed result using \cref{asmp: both asmp joint}.(i). 
\end{proof}

While \cref{lemma: intermediate lipschitzness diffusion} showed that \citep[Assumption 5.({\rm i})]{abeille_diffusive_2022} is implied by \cref{asmp: both asmp joint,asmp: basics}. It remains now to verify their Assumption 5.({\rm ii}). Note that by \citep[Remark 3.2.({\rm ii})]{abeille_diffusive_2022}, an equation of the form of their (3.3) is sufficient to do so. \Cref{lemma: contraction in p power} gives exactly this result with \eqref{eq: near contraction diff power p}, by noting that \citep[(3.4)]{abeille_diffusive_2022} holds by \cref{asmp: both asmp joint}.

\def\bcvs{\bar\cf_p}
\begin{lemma}\label{lemma: contraction in p power}
    Under \cref{asmp: both asmp joint,asmp: basics}, for any $p\ge 2$ there are $(\bcvs,\bcvs')\in\RR_+^2$ such that
    \begin{align}
         \bar\mu_\theta(x,a)^\top\nabla\Vs(x)^p+\Tr[\bar\Sigma\bar\Sigma^\top\nabla^2\Vs(x)^p]\le -\bcvs\Vs(x)^p+\bcvs'\,\label{eq: near contraction diff power p}
    \end{align}
    for any $(x,a,\theta)\in\RR^d\x\Ab\x\Theta$.
\end{lemma}

\begin{proof}
    Let us take $(x,x')\in\RR^d\x\RR^d$ such that $\norm{x-x'}\ge \ve/(1-\ve L_0)$, which implies $\norm{x-x'+\Delta(\mu_\theta(x,a)-\mu_\theta(x',a))}>0$ for any $\Delta\in[0,1]$ and for all $(a,\theta)\in\Ab\x\Theta$ and we can expand \eqref{eq:lyapunov asmp joint on jump problem}, which gives 
    \begin{align*}
        -\ve\cf_\Vs \Vs(x-x')&\ge (\mu_\theta(x,a)-\mu_\theta(x',a))^\top\nabla\Vs(x-x')\\
        &\qquad + \frac12(\mu_\theta(x,a)-\mu_\theta(x',a))^\top\nabla^2\Vs(\hat x)(\mu_\theta(x,a)-\mu_\theta(x',a))\,, 
    \end{align*} 
    in which $\hat x = x+\hat \Delta (x'-x)$ for some $\hat \Delta\in[0,1]$. Thus 
    \begin{align*}
        (\bar\mu_\theta(x,a)-\bar\mu_\theta &(x',a))^\top\nabla\Vs(x-x')  \\
        &\le -\cf_\Vs\Vs(x-x')-\frac{\ve}2(\bar\mu_\theta(x,a)-\bar\mu_\theta(x',a))^\top\nabla^2\Vs(\hat x)(\bar\mu_\theta(x,a)-\bar\mu_\theta(x',a))\,.
    \end{align*} 
    Letting $\ve\to0$, the constraint on $(x,x')$ vanishes as well as the second term (on compact sets), and we recover
    \[(\bar\mu_\theta(x,a)-\bar\mu_\theta(x',a))^\top\nabla \Vs(x-x')+\frac12\Tr[\bar\Sigma\bar\Sigma^\top \nabla^2\Vs(x-x')]\le -\cf_\Vs\Vs(x-x') +\frac d2\snorm{\bar\Sigma}^2_\op M_\Vs'\,.\]
    Taking $x'=0$ implies that 
    \[ \bar\mu_\theta(x,a)^\top\nabla\Vs(x) +\frac12\Tr[\bar\Sigma\bar\Sigma^\top\nabla^2\Vs(x)]\le -\cf_\Vs\Vs(x) + C \]
    for all $(x,a)\in\RR_*^d\x\Ab$, in which $C:=d\snorm{\bar\Sigma}_\op^2 M_\Vs'/2 + L_0M_\Vs$.
    
    Notice that, since $\Vs\in\Cc^2(\RR^d_*;\RR_+)$ and vanishes at $0$ (see \cref{asmp: basics}), $\Vs(\cdot)^p$ can be extended by continuity at $0$ so that $\Vs(\cdot)^p\in\Cc^2(\RR^d;\RR_+)$. For any $(x,a,\theta)\in\RR^d\x\Ab\x\parset$, let 
    \begin{align*}
        k(x,a):&=\bar\mu_\theta(x,a)^\top\nabla\Vs(x)^p+\frac12\Tr\left[\bar\Sigma\bar\Sigma^\top\nabla^2\Vs(x)^p\right]\\
        &= p\bar\mu_\theta(x,a)^\top\nabla\Vs(x)\Vs(x)^{p-1}\\
        &\qquad+ \frac12\Tr\left[\bar\Sigma\bar\Sigma^\top \left(p\Vs(x)^{p-1}\nabla^2\Vs(x) + p(p-1)\Vs(x)^{p-2}\nabla\Vs(x)\nabla^\top\Vs(x)\right) \right]\\
        &=p\Vs^{p-1}(x)\left(\bar\mu_\theta(x,a)^\top\nabla\Vs(x)+\frac12\Tr[\bar\Sigma\bar\Sigma^\top\nabla^2\Vs(x)]\right) \\
        &\qquad + \frac{p(p-1)}2 \Vs(x)^{p-2}\Tr[\bar\Sigma\bar\Sigma^\top\nabla\Vs(x)\nabla^\top\Vs(x)]\\
        &\le -p\cf_\Vs\Vs(x)^p + Cp\Vs(x)^{p-1} + \frac{dp(p-1)}2(\snorm{\bar\Sigma}_\op M_\Vs)^2\Vs(x)^{p-2}
    \end{align*}
    and we can now choose $\bcvs=-p\cf_\Vs/2$, for which there exists a constant $\bcvs'$ such that 
    \[-\bcvs \Vs^p(x)+ Cp\Vs^{p-1}(x) + \frac{dp(p-1)}2(\snorm{\bar\Sigma}_\op M_\Vs)^2\Vs^{p-2}(x)  \le \bcvs'\]
    for all $x\in\RR^d$.
\end{proof}

\subsection{Proof of \texorpdfstring{\cref{prop: approx diff limit}}{Proposition \ref{prop: approx diff limit}}}\label{subsec:approx diff limit}

The rest of this section is dedicated to showing \cref{prop: approx diff limit} using modifications of the proof of \citep[Thm.~3.6.]{abeille_diffusive_2022} to which it corresponds. Here we produce a self-contained proof in order to clarify how \eqref{eq: approx diff limit evolution} is derived from the proof.

\PropRhoApprox*

\begin{proof}
    The first part of \cref{prop: approx diff limit}, i.e. \eqref{eq: approx diff limit for rho}, corresponds to \citep[Thm.~3.6.]{abeille_diffusive_2022}, which we previously showed holds in our setting by verifying its assumptions. 
    We now prove the second claim. Let
    \[\delta r_\theta^\ve (x,a) := \bar\mu_\theta(x,a)^\top\nabla\bar W_\theta^*(x)+\frac12\Tr[\bar\Sigma\bar\Sigma^\top\nabla^2\bar W^*_\theta(x)]-\frac1\ve\left(\EE\left[ \bar W_\theta^*(\pve_\theta(x,a)+\Sigma\xi)\right] -\bar W^*_\theta(x)\right).\]
    From \eqref{eq: prelim HJB diff},  and \cref{prop: PropertiesRhoDiff}.(iii.) we have 
    \begin{align*}
        \bar\rho^*_\theta &= \max_{a\in\Ab}\left\{\bar\mu_\theta(x,a)^\top\nabla\bar W^*_\theta +\frac12\Tr[\bar\Sigma\bar\Sigma^\top\nabla^2\bar W^*_\theta(x)] + \bar r(x,a)\right\}\\
        &=\bar\mu_\theta(x,\bar\pi^*_\theta(x))^\top
        \nabla \bar W^*_\theta(x) + \frac12\Tr[\bar\Sigma\bar\Sigma^\top\nabla^2\bar W^*_\theta(x)] + \bar r(x,\bar\pi^*_\theta(x))
        \intertext{which implies}
        \ve\rho_\theta^{\bar\pi^*_\theta}(0)&=\EE[ \bar W_\theta^*(\pve_\theta(x,\bar\pi^*_\theta(x))\!+\!\Sigma\xi)]\! -\!\bar W^*_\theta(x) + r(x,\bar\pi^*_\theta(x)) + \ve(\delta r_\theta^\ve(x,\bar\pi^*_\theta(x))+ \bar\rho^*_\theta-\rho_\theta^{\bar\pi^*_\theta}(0)).
    \end{align*}
    Note that $\abs{\delta r_\theta^\ve(x,\bar\pi_\theta^*(x))}\le\sup_{a\in\Ab}\abs{\delta r_\theta^\ve(x,a)}$, which by \citep[(3.10)]{abeille_diffusive_2022} is bounded by $c_\gamma\ve^{\frac{\gamma}{2}}(1+\norm{x}^3)$ for some constant $c_\gamma>0$. An application of \eqref{eq: approx diff limit for rho} yields
    \[ \bar\rho_\theta^*- \rho^{\bar\pi_\theta^*}_\theta(0) = \bar\rho_\theta^*-\rho^*_\theta +\rho^*_\theta -\rho^{\bar\pi_\theta^*}_\theta(0)\le 2C_\gamma \ve^{\frac \gamma2} \]
    and, at the same time, $\bar\rho^*_\theta-\rho_\theta^{\bar\pi^*_\theta}(0)\ge \bar\rho^*_\theta -\rho^*_\theta\ge -C_\gamma\ve^{\frac\gamma2}$.
    Therefore, there is a function $e_\theta:\RR^d\to\RR$ such that \eqref{eq: approx diff limit evolution} holds, which also satisfies  
    \[\abs{e_\theta(x)}\le (2C_\gamma +c_\gamma)\ve^{1+\frac\gamma2}(1+\norm{x}^3) \,. \]
\end{proof}

%% file: regret_proofs.tex
\section{Regret Analysis}\label{app: regret bounds}

In this final appendix, we complete the analysis of the regret of \cref{alg: RL 1} and prove \cref{thm: regret alg 1}. First, we will give the regret decomposition, and then in the later sections we will bound terms one by one calling upon the results of the previous appendices.

\RegretAlgOne*

\subsection{Regret Decomposition}\label{subsec: regret decom}

Recall that we defined $k:n\in\NN\mapsto k(n)$ as the map associating to each event $n$ the episode of \cref{alg: RL 1} in which they occur. Like in \cref{subsec:regret}, let us define $\theta_n=\tilde\theta_{k(n)}$ for all $n\in\NN$. The regret of \cref{alg: RL 1}, which generates the control $\varpi\in\Ac$, is
\begin{align*}
    \Rc_T(\varpi) &:= T\rho^*_{\theta^*} - \sum_{n=\deb}^{N_T}r(X_{\tau_n}^{\varpi,\theta^*},\varpi_{\tau_n})
    \intertext{By definition of $\varpi$ in \cref{alg: RL 1}, $\varpi_{\tau_n} = \bar\pi^*_{\theta_n}(X_{\tau_n}^{\varpi,\theta^*})$, so that }
    &\Rc_T(\varpi) := T\rho^*_{\theta^*} - \sum_{n=\deb}^{N_T}r(X_{\tau_n}^{\varpi,\theta^*},\bar\pi_{\theta_n}^*(X_{\tau_n}^{\varpi,\theta^*}))
\end{align*}

At the heart of the decomposition is the use of the HJB-type equation \eqref{eq: approx diff limit evolution} applied for each $n$ at the point $X_{\tau_n}^{\varpi,\theta^*}$. For clarity, let us introduce for all $n\in\NN$ the random variable $\tilde X^{\varpi,\theta_n}_{\tau_{n+1}}$ equal in distribution, conditionally on $\Fc_{\tau_n}$, to the random variable $\pve_{\theta_n}(X^{\varpi,\theta^*}_{\tau_n},\varpi_{\tau_n})+\Sigma\xi_{n+1}$. With this notation \eqref{eq: approx diff limit evolution} becomes
\begin{align}
    \ve\rho_{\theta_n}^{\bar\pi^*_{\theta_n}}(0) &=\EE[ \bar W_{\theta_n}^*(\tilde X^{\varpi,\theta_n}_{\tau_{n+1}})\vert \Fc_{\tau_n}] -\bar W^*_{\theta_n}(X_{\tau_n}^{\varpi,\theta^*}) +r(X_{\tau_n}^{\varpi,\theta^*},\bar\pi_{\theta_n}^*(X_{\tau_n}^{\varpi,\theta^*})) + e_{\theta_n}(X_{\tau_n}^{\varpi,\theta^*})\,.\label{eq: approx diff limit evolution 2}
\end{align}

This \textit{imagined} evolution of the system represents the counterfactual induced by a single step transition at time $\tau_{n+1}$, according to the belief in $\theta_n$. With this notation, applying \eqref{eq: approx diff limit evolution 2} yields
\begin{align}
        \Rc_T(\varpi)&= T\rho^*_{\theta^*}- \sum_{n=\deb}^{N_T} \ve\rho_{\theta_n}^{\bar\pi^*_{\theta_n}}(0)+ \sum_{n=\deb}^{N_T} e_{\theta^*}(X_{\tau_n}^{\varpi,\theta^*})+\sum_{n=\deb}^{N_T}\EE[\bar W_{\theta_n}^*(\tilde X_{\tau_{n+1}}^{\varpi,\theta_n})\vert \Fc_{\tau_n}] - \bar W_{\theta_n}^*(X_{\tau_n}^{\varpi,\theta^*})\notag\,.\\
        &=(T-\ve N_T)\rho^*_{\theta^*} \tag{$R_1$}\label{eq: regret pf R1} \\
        &\quad + \ve \sum_{n=\deb}^{N_T}(\rho^*_{\theta^*} - \rho_{\theta_n}^{\bar\pi^*_{\theta_n}}(0) )+ \sum_{n=\deb}^{N_T} e_{\theta^*}(X_{\tau_n}^{\varpi,\theta^*}) \tag{$R_2$}\label{eq: regret pf R2} \\
        &\quad + \sum_{n=\deb}^{N_T}\EE[\bar W_{\theta_n}^*(\tilde X_{\tau_{n+1}}^{\varpi,\theta_n})\vert \Fc_{\tau_n}] - \bar W_{\theta_n}^*(X_{\tau_n}^{\varpi,\theta^*}).\label{eq: regret remainder of first decomp}
    \end{align}
    The first term, \eqref{eq: regret pf R1}, quantifies the deviation of the Poisson clock from its mean. On the other hand, \eqref{eq: regret pf R2} quantifies both the optimistic nature of \cref{alg: RL 1} and the approximation error of its approximate planning. The third term, \eqref{eq: regret remainder of first decomp}, resembles a martingale (up to reordering), but it fails to be one on two key counts. First, the element from the family of functions $(\bar W^*_{\theta_n})_{n\in\NN}$ used at each step $n$ changes. Second, the expectation terms are with respect to the counterfactual transitions $(\tilde X^{\varpi,\theta^*}_{\tau_{n+1}})_{n\in\NN}$ while the random terms use the real transitions $(X^{\varpi,\theta^*}_{\tau_{n+1}})_{n\in\NN}$. 

    Note that we can control the difference between the counterfactual and the real trajectory at a one-step time horizon, by using 
    \begin{align}
        \tilde X^{\varpi,\theta}_{\tau_{n+1}}\overset{\mathrm{d}}{=}X^{\varpi,\theta^*}_{\tau_{n+1}}-\mu_{\theta^*}(X^{\varpi,\theta^*}_{\tau_n},\varpi_{\tau_n}) + \mu_{\theta}(X^{\varpi,\theta^*}_{\tau_n},\varpi_{\tau_n}) \,,\label{eq: counterfactual removal bound}
    \end{align}
    in which $\overset{\mathrm{d}}{=}$ denotes equality in the same conditionally distributional sense as above.
    By adding and subtracting relevant terms to exhibit the key quantities we get:
        \begin{align}
        \sum_{n=\deb}^{N_T}\EE[\bar W_{\theta_n}^*(\tilde X_{\tau_{n+1}}^{\varpi,\theta_n})\vert \Fc_{\tau_n}] - \bar W_{\theta_n}^*(X_{\tau_n}^{\varpi,\theta^*})\notag
        &\le \sum_{n=\deb}^{N_T}\EE[\bar W_{\theta_n}^*(\tilde X_{\tau_{n+1}}^{\varpi,\theta_n})\vert \Fc_{\tau_n}] - \EE[\bar W_{\theta_n}^*( X_{\tau_{n+1}}^{\varpi,\theta^*})\vert \Fc_{\tau_n}] \notag\\
        &\quad + \sum_{n=\deb}^{N_T} \EE[\bar W_{\theta_n}^*( X_{\tau_{n+1}}^{\varpi,\theta^*})\vert \Fc_{\tau_n}]-\EE[\bar W_{\theta_{n+1}}^*( X_{\tau_{n+1}}^{\varpi,\theta^*})\vert \Fc_{\tau_n}] \notag\\
        &\quad + \sum_{n=\deb}^{N_T} \EE[\bar W_{\theta_{n+1}}^*( X_{\tau_{n+1}}^{\varpi,\theta^*})\vert \Fc_{\tau_n}] - \bar W_{\theta_n}^*(X_{\tau_n}^{\varpi,\theta^*})\notag\,.
    \end{align}
    Using \eqref{eq: counterfactual removal bound}, and the uniform $L_W$-Lipschitzness of $(\bar W^*_{\theta_n})_{n\in\NN}$, we get for each $n\in\NN$
    \[\EE[\bar W_{\theta_n}^*(\tilde X_{\tau_{n+1}}^{\varpi,\theta_n})\vert \Fc_{\tau_n}] - \EE[ \bar W_{\theta_n}^*( X_{\tau_{n+1}}^{\varpi,\theta^*})\vert \Fc_{\tau_n}] \leq L_W \norm{\mu_{\theta_n}(X_{\tau_n}^{\varpi,\theta^*},\varpi_{\tau_n})-\mu_{\theta^*}(X_{\tau_n}^{\varpi,\theta^*},\varpi_{\tau_n})}\] 
    and thus the regret term \eqref{eq: regret remainder of first decomp} is bounded by
    \begin{align}
        \sum_{n=1}^{N_T}\EE[\bar W_{\theta_n}^*(\tilde X_{\tau_{n+1}}^{\varpi,\theta_n})\vert \Fc_{\tau_n}]& - \bar W_{\theta_n}^*(X_{\tau_n}^{\varpi,\theta^*})%
         \le R_3+R_4+R_5 \notag%
        \intertext{in which}
        R_3:&= L_W\sum_{n=\deb}^{N_T}\norm{\mu_{\theta_n}(X_{\tau_n}^{\varpi,\theta^*},\varpi_{\tau_n})-\mu_{\theta^*}(X_{\tau_n}^{\varpi,\theta^*},\varpi_{\tau_n})}\tag{$R_3$}\label{eq: regret pf R3} \\
        R_4:&=\sum_{n=\deb}^{N_T} \EE[ \bar W^*_{\theta_n}( X_{\tau_{n+1}}^{\varpi,\theta^*})-\bar W^*_{\theta_{n+1}}( X^{\varpi,\theta^*}_{\tau_{n+1}})\vert \Fc_{\tau_n}]\tag{$R_4$}\label{eq: regret pf R4} \\
        R_5:&=\sum_{n=\deb}^{N_T} \EE[\bar W^*_{\theta_{n+1}}( X^{\varpi,\theta^*}_{\tau_{n+1}})\vert \Fc_{\tau_{n}}] - \bar W^*_{\theta_{n}}(X_{\tau_n}^{\varpi,\theta^*})\tag{$R_5$}\label{eq: regret pf R5}\,.
    \end{align}
    At the end of this decomposition, we have constructed a true martingale in \eqref{eq: regret pf R5}, which we bound in \cref{subsec: regret decomp R5}. The first term \eqref{eq: regret pf R3} accumulates the fit error described in \cref{cor: conf set width no lazy updates}, up to the lazy updates, which we study in \cref{subsec: regret decomp R3}. The term \eqref{eq: regret pf R4} is bounded by the number of effective updates of $\theta_n$ (namely, $\sum_{n=\deb}^{N_T}\1_{\{\theta_{n+1}\neq\theta_n\}}$) in  \cref{subsec: regret decomp R4}.
    Finally, the bounds on \eqref{eq: regret pf R1} and \eqref{eq: regret pf R2} are given in \cref{subsec: regret decomp R1,subsec: regret decomp R2} respectively.

    To combine the high-probability events used to bound \eqref{eq: regret pf R1} and \eqref{eq: regret pf R5}, with the event of \cref{prop: conf sets OvR} used by the other terms, we will perform a union bound. This corresponds to the $\delta/3$ used in the definition of the confidence sets of \cref{alg: RL 1}.
    
\subsection{Bounding the Poisson clock variation term \texorpdfstring{\eqref{eq: regret pf R1}}{{(R1)}}}\label{subsec: regret decomp R1}

        We bound \eqref{eq: regret pf R1} using \cref{lemma: concentration of Poisson Measure} which is a standard sub-exponential concentration result, see e.g.~\citep[Lemma~4.1]{buldygin_metric_2000}. It implies 
        \[\PP\left(\abs{T-\ve N_T} \ge 2\sqrt{\ve T\log\left(\frac{6}{\delta}\right)}\vee 2\ve \log\left(\frac{6}{\delta}\right)\right) \le \frac\delta3\,.\]
        
    \begin{lemma}\label{lemma: concentration of Poisson Measure}
        For any $T\in\RR_+^*$ and $\delta\in(0,1)$,
        \[\PP\left(\abs{\ve N_T-T}>2\sqrt{\ve T\log\left(\frac2\delta\right)}\vee 2\ve \log\left(\frac2\delta\right)  \right)\le \delta\,.\]
    \end{lemma}
    
    \begin{proof}
        Let $\upsilon:=\ve^{-1}T$.
        For any $\lambda\in[-1,1]$, $\EE[e^{\lambda (N_T-\upsilon)}]= \exp(\upsilon(e^{\lambda}-1-\lambda))\le e^{\lambda^2\upsilon}$. Therefore, $N_T$ is $(\sqrt{2}\upsilon,1)$-subexponential (see e.g.~\citep{buldygin_metric_2000}) and therefore,
        \[ \PP\left(\abs{N_T-\upsilon} > \epsilon\right)\le \begin{cases} e^{-\frac{\epsilon^2}{4\upsilon}} \mbox{ for }\epsilon\in(0,2\upsilon]\\ e^{-\frac \epsilon2} \mbox{ for } \epsilon >2\upsilon\end{cases}\,,\]
        which implies
        \[ \PP\left(\abs{N_T-\upsilon} > 2\sqrt{\upsilon\log\left(\frac2\delta\right)}\1_{\{\delta\ge e^{-\upsilon}\}} + 2\log\left(\frac2\delta\right)\1_{\{\delta \le e^{-\upsilon}\}}\right)\le \delta\,.\] 
    \end{proof}

\subsection{Bounding the optimistic approximation term \texorpdfstring{\eqref{eq: regret pf R2}}{{(R2)}}}\label{subsec: regret decomp R2}

 There are two terms in \eqref{eq: regret pf R2}. The second is the most straightforward as it can be bounded by applying the bound on $e_{\theta^*}$ of \cref{prop: approx diff limit}, which yields
        \[ \sum_{n=\deb}^{N_T} e_{\theta^*}(X_{\tau_n}^{\varpi,\theta^*}) \le 2C_\gamma'N_T\ve^{1+\frac\gamma2}(1+\sup_{s\le T}\snorm{X_s^{\varpi,\theta^*}}^3)\,.\]
        We decompose the remaining term of \eqref{eq: regret pf R2} into
        \begin{align*}
            \ve \sum_{n=\deb}^{N_T} (\rho^*_{\theta^*} - \rho^{\bar\pi^*_{\theta_n}}_{\theta_n})
            &= \ve \sum_{n=\deb}^{N_T}\left(\rho^*_{\theta^*} - \bar\rho^*_{\theta^*} +\bar\rho^{\bar\pi^*_{\theta^*}}_{\theta^*}-\bar\rho^{\bar\pi^*_{\theta_n}}_{\theta_n}+\bar\rho^*_{\theta_n}-\rho^*_{\theta_n}+ \rho^*_{\theta_n} - \rho^{\bar\pi^*_{\theta_n}}_{\theta_n}\right)\\
            &\le 4 N_TC_\gamma \ve^{1+\frac\gamma2} +\ve \sum_{n=\deb}^{N_T}\left(\bar\rho^{\bar\pi^*_{\theta^*}}_{\theta^*}-\bar\rho^{\bar\pi^*_{\theta_n}}_{\theta_n}\right)
        \end{align*}
        by applying \cref{prop: approx diff limit} to all but the second pair of terms. 

        On the event of \cref{prop: conf sets OvR}, with $\delta/3$ in place of $\delta$, we have $\theta^*\in\cap_{n\in\NN^*}\Cc_n(\delta/3)$ and thus, by definition of \cref{alg: RL 1}, $\bar\rho^{\bar\pi^*_{\theta^*}}_{\theta^*}-\bar\rho^{\bar\pi^*_{\theta_n}}_{\theta_n}\le 0$ for all $n\in\NN^*$ . Thus, on this event we have
        \begin{align*}
            \ve \sum_{n=\deb}^{N_T} (\rho^*_{\theta^*} - \rho^{\bar\pi^*_{\theta_n}}_{\theta_n})
            &\le 4N_TC_\gamma\ve^{1+\frac\gamma2}\,.
        \end{align*}

\subsection{Bounding the prediction error term \texorpdfstring{\eqref{eq: regret pf R3}}{{(R3)}}}\label{subsec: regret decomp R3}
        Because of the lazy updates, $\mu_{\theta_n}=\mu_{\theta_{k(n)}}$ is chosen within $\Cc_{k(n)}(\delta/3)$ instead of $\Cc_{n}(\delta/3)$ preventing us from using directly \cref{prop: conf set width no lazy updates}.
        Nevertheless, the lazy update-scheme is designed not to degrade the overall learning performance by more than a constant factor. Leveraging \eqref{eq: def lazsy update}, 
        \begin{align}
            \sum_{i=\deb}^{n-1} \norm{\mu_{\theta_n}(X_{\tau_i}^{\varpi,\theta^*},\varpi_{\tau_i})-\mu_{\theta^*}(X_{\tau_i}^{\varpi,\theta^*},\varpi_{\tau_i})} &\le \begin{cases} 2\beta_n(\delta/3) \mbox{ if } n< n_k\\ \beta_n(\delta/3) \mbox{ if } n=n_k\end{cases}
        \end{align}
        As a result, $\mu_{\theta_n}$is chosen within an inflated version of $\Cc_n(\delta/3)$, defined as in \eqref{eq: def conf sets} but with $\beta_n(\delta/3)$  replaced by $2\beta_n(\delta/3)$. Thus, we can follow the same arguments as in the proof of \cref{cor: conf set width no lazy updates}, by applying \cref{prop: conf set width no lazy updates} to the inflated confidence sets, up to the constant factor $2$ in the bounds. And therefore on the event of \cref{prop: conf sets OvR}, we have
        \begin{align*}
            R_3&=L_W\sum_{n=1}^{N_T} \norm{\mu_{\theta_n}(X_{\tau_n}^{\varpi,\theta^*},\varpi_{\tau_n})-\mu_{\theta^*}(X_{\tau_n}^{\varpi,\theta^*},\varpi_{\tau_n})} \\
            &\le 6L_W\beta_{N_T}(\delta/3)\sqrt{\mathrm{d}_{\mathrm{E},N_t}}+ L_W\mathrm{d}_{\mathrm{E},N_t}H_{\delta/3}(N_T)\,.
        \end{align*}

\subsection{Bounding the lazy-update term \texorpdfstring{\eqref{eq: regret pf R4}}{{(R4)}}}\label{subsec: regret decomp R4}

        We observe that \eqref{eq: regret pf R4} is bounded by
        \begin{align*}
            R_4&=\sum_{n=\deb}^{N_T} \EE[ \bar W^*_{\theta_n}( X_{\tau_{n+1}}^{\varpi,\theta^*})-\bar W^*_{\theta_{n+1}}(X^{\varpi,\theta^*}_{\tau_{n+1}})\vert \Fc_{\tau_n}]\\
            &\le 2L_W\sum_{n=\deb}^{N_T} \EE\left[ \left(1+ \snorm{ X_{\tau_{n+1}}^{\varpi,\theta^*}}\right)\1_{\{\theta_{n}\neq\theta_{n+1}\}} \vert \Fc_{\tau_n}\right]\\
            &\le 2L_W\sum_{n=\deb}^{N_T}\left((1+\ve L_0)(1+\snorm{X_{\tau_{n}}^{\varpi,\theta^*}})+ \ve^{\frac12}\snorm{\bar\Sigma}_\op \EE\left[\norm{\xi_{n+1}} \vert \Fc_{\tau_n}\right]\right)\1_{\{\theta_{n}\neq\theta_{n+1}\}}\\
            &\le 2L_W(1+\ve L_0)\left(1+\sup_{s\le T}\snorm{X_s^{\varpi,\theta^*}} + \sqrt{d}\ve^{\frac12}\snorm{\bar\Sigma}_\op\right)\sum_{n=\deb}^{N_T}\1_{\{\theta_{n}\neq\theta_{n+1}\}}\,.
        \end{align*}
        Thus bounding the number of updates with \cref{lemma: number of updates} bounds \eqref{eq: regret pf R4}.
    \begin{lemma}\label{lemma: number of updates}
        Under \cref{asmp: basics,asmp: both asmp joint}, \cref{alg: RL 1} generates episodes which satisfy for all $T\in\RR_+$ and $\delta\in(0,1)$
    \begin{align*}
        \sum_{n=\deb}^{N_T}\1_{\{\theta_{n}\neq\theta_{n+1}\}} &\le 4\beta_{N_T}(\delta/3)^2\mathrm{d}_{\mathrm{E},{N_t}}\left(3 + \log\left(\frac{N_t\sup_{s\le t}\snorm{X_s^{\varpi,\theta^*}}}{16\beta_{N_t}(\delta/3)^4\mathrm{d}_{\mathrm{E},{N_t}}^2}\right)\right)\notag\\
        &\qquad + 2\mathrm{d}_{\mathrm{E},N_t}(1+2\beta_{N_t}(\delta/3)^2\mathrm{d}_{\mathrm{E},N_t})(1+\sup_{s\le t}\snorm{X_s^{\varpi,\theta^*}}^2)\,.
    \end{align*}
    \end{lemma}
    
    \begin{proof}
        Consider $k\in\NN^*$, by \eqref{eq: def lazsy update}, each time we trigger an update we have 
        \begin{align*}
            2\beta_{n_k}(\delta/3)^2 &< \sup_{\mu_\theta\in\Cc_{n_{k-1}}(\delta)} \norm{\mu_\theta-\mu_{\hat\theta{n_{k-1}}}}_{n_{k}}^2\\
            &\le \sup_{\mu_\theta\in\Cc_{n_{k-1}}(\delta)} \norm{\mu_\theta-\mu_{\hat\theta{n_{k-1}}}}_{n_{k-1}}^2 \\
            &\qquad+ \sup_{\mu_\theta\in\Cc_{n_{k-1}}(\delta)}\sum_{n=n_{k-1}+1}^{n_k} \norm{\mu_\theta(X_{\tau_n}^{\varpi,\theta},\varpi_{\tau_n})-\mu_{\hat\theta{n_{k-1}}}(X_{\tau_n}^{\varpi,\theta},\varpi_{\tau_n})}^2\\
            &\le\beta_{n_k}(\delta/3)^2+ \sum_{n=n_{k-1}+1}^{n_k} \Lambda(\Cc_{n_{k-1}}(\delta/3); X_{\tau_n}^{\varpi,\theta},\varpi_{\tau_n})^2\,.
        \end{align*}
    Summing over all episodes, since the sequence $(\beta_n(\delta/3))_{n\in\NN}$ is non-decreasing, we have that for all $T\in\RR_+$
    \[ \sum_{n=\deb}^{N_T} \Lambda(\Cc_{n_k}(\delta/3); (X_{\tau_n},\varpi_{\tau_n}))^2\ge \sum_{k=\deb}^{K_T}\beta_{n_k}(\delta/3)^2\ge K_T\beta_0(\delta/3)^2\,,\]
    in which $K_T:=k(N_T)\in\NN$ is the number of episodes by time $T$. An application of the second part of \cref{prop: conf set width no lazy updates}, i.e. \eqref{eq: conf set width order 2} now yields the desired result as $\beta_0(\delta/3)^2=\ve$.
    \end{proof}

\subsection{Bounding the martingale term \texorpdfstring{\eqref{eq: regret pf R5}}{{(R5)}}}\label{subsec: regret decomp R5}

    Let
    \[Z_n:=\EE[\bar W^{*}_{\theta_n}(X^{\alpha,\theta^*}_{\tau_n})\vert\Fc_{n-1}]-\bar W^{*}_{\theta_n}(X^{\alpha,\theta^*}_{\tau_n})\,.\] 
    By definition
    \[R_5= \EE[\bar W^*_{\theta_{N_T+1}}( X^{\varpi,\theta^*}_{\tau_{N_T+1}})\vert \Fc_{\tau_{N_T}}] + \bar W^*_{\theta_0}(x_0) + \sum_{n=\deb}^{N_T} Z_n\,.\]
    On the one hand, $Z_n$ is a $L_W\snorm{\Sigma}_\op$-Lipschitz function of $\xi_n$, which is Gaussian and of mean $0$. Therefore, by \citep[Thm~5.5]{boucheron_concentration_2013}, $Z_n$ is $L_W\norm{\Sigma}_\op$-sub-Gaussian and
    \begin{align}
        \PP\left(\sum_{n=\deb}^{N_T}Z_n> L_W\snorm{\bar\Sigma}_\op\sqrt{2\ve N_T\log\left(\frac1\delta\right)}\right)\le \delta\,.\label{eq: R5 Z_N martingale part}
    \end{align}

    On the other hand, by the uniform Lipschitzness of $(\bar W^*_\theta)_{\theta\in\Theta}$, $\bar W^*_{\theta_0}(x_0)\le L_W(1+\norm{x_0})$ and 
    \begin{align}
        \EE[\bar W^*_{\theta_{N_T+1}}( X^{\varpi,\theta^*}_{\tau_{N_T+1}})\vert \Fc_{\tau_{N_T}}] & \le L_W(1+\EE[\snorm{X^{\varpi,\theta^*}_{\tau_{N_T+1}}}\vert \Fc_{\tau_{N_T}}])\notag\\
        &\le L_W(1+\ve L_0+(1+\ve L_0)\snorm{X^{\varpi,\theta^*}_{\tau_{N_T}}} + \ve^{\frac12}\snorm{\bar\Sigma}_\op\EE[\norm{\xi_{N_T+1}}\vert\Fc_{\tau_{N_T}}])\notag\\
        &\le L_W(1+\ve L_0)\left(1+\sup_{s\le T}\snorm{X^{\varpi,\theta^*}_s}_2 + \ve^{\frac12}\snorm{\bar\Sigma}_\op \sqrt{d} L_W\right).\label{eq: R5 remainder terms}
    \end{align}
    Combining \eqref{eq: R5 Z_N martingale part} and \eqref{eq: R5 remainder terms} yields
    \begin{align}
        R_5 \le L_W\norm{\bar\Sigma}_\op\sqrt{2\ve N_T\log\left(\frac3\delta\right)}+ 2L_W(1+\ve L_0)(1+\sup_{s\le T}\snorm{X^{\varpi,\theta^*}_s} + \ve^{\frac12}\snorm{\bar\Sigma}_\op\sqrt d L_W)
    \end{align}
    with probability at least $1-\delta/3$.
    
\subsection{Collecting the bounds}
    
    We conclude the proof of \cref{thm: regret alg 1} by collecting all the terms from Appendices \ref{subsec: regret decomp R1}--\ref{subsec: regret decomp R5} and simplifying them. By a union bound over the events listed in steps \cref{subsec: regret decomp R1,subsec: regret decomp R3,subsec: regret decomp R5}, with probability at least $1-\delta$ 
    \begin{align*}
        \Rc_T(\varpi)&\le 2L_0\left(\sqrt{\ve T\log\left(\frac{6}{\delta}\right)}\vee 2\ve \log\left(\frac{6}{\delta}\right)\right)\\
        &\quad + 4N_TC_\gamma\ve^{1+\frac\gamma2}+2C_\gamma'N_T\ve^{1+\frac\gamma2}(1+H_{\delta/3}^3(N_T))\\
        &\quad +  6L_W\beta_{N_T}(\delta/3)\sqrt{\mathrm{d}_{\mathrm{E},N_T}}+ L_W\mathrm{d}_{\mathrm{E},N_T}H_{\delta/3}(N_T)\\
        &\quad +2L_W(1+\ve L_0)\left((1+H_{\delta/3}(N_T)+ d\ve^{\frac12}\snorm{\bar\Sigma}_\op\right) \vvast(4\beta_{N_T}(\delta/3)^2\mathrm{d}_{\mathrm{E},{N_t}}\Bigg(3\notag\\
        &\quad\; + \log\left(\frac{N_t H_{\delta/3}(N_T)}{16\beta_{N_T}(\delta/3)^4\mathrm{d}_{\mathrm{E},{N_T}}^2}\right)\!\!\Bigg)+ 2\mathrm{d}_{\mathrm{E},N_T}(1+2\beta_{N_t}(\delta/3)^2\mathrm{d}_{\mathrm{E},N_T})(1+H_{\delta/3}(N_T)^2)\!\vvast)\notag\\
        &\quad + L_W\norm{\bar\Sigma}_\op\sqrt{2\ve N_T\log\left(\frac3\delta\right)}+ 2L_W(1+\ve L_0)(1+H_{\delta/3}(N_T) + \ve^{\frac12}\snorm{\bar\Sigma}_\op\sqrt d L_W).
    \end{align*}
    This can be more simply expressed for some constants $C_\Rc^{(i)}\in\RR_+$, $i\in[5]$, as
    \begin{align*}
        \Rc_T(\varpi)&\le C_\Rc^{(1)}(C_\gamma+C_\gamma')\ve^{1+\frac\gamma2}N_T\log(N_T)^3 + C_\Rc^{(2)}\sqrt{\mathrm{d}_{\mathrm{E},{N_T}}\ve N_T\log\left( \frac{N_T(1+\ve\Ns_{N_T}^\ve)}\delta\right)} \\
        &\qquad+  C_\Rc^{(3)}\Big(1+\ve\mathrm{d}_{\mathrm{E},{N_T}}\log(N_T)\log(N_T(1+\ve\Ns^\ve_{N_T}))\Big)\mathrm{d}_{\mathrm{E},{N_T}}\log(N_T)^4\\
        &\qquad+ C_\Rc^{(4)} \sqrt{\ve T \log\left(\frac{1}\delta\right)} +C_\Rc^{(5)}\left(1+\log\left(\frac{1}\delta\right)\right)\,
    \end{align*}
    still with probability at least $1-\delta$.
    On this high-probability event we can write $\Rc_T(\varpi)$ (up rounding up $T\ve^{-1}$ where necessary and up to a change in the constants) as 
    \begin{align*}
        \Rc_T(\varpi)&\le C_\Rc^{(1)}(C_\gamma+C_\gamma')\ve^{\frac\gamma2}T\log\left(\frac T\ve\right) + C_\Rc^{(2)}\sqrt{\mathrm{d}_{\mathrm{E},{T\ve^{-1}}}T\log\left( \frac{T\ve^{-1}(1+\ve\Ns_{T\ve^{-1}})}\delta\right)} \\
        &\qquad+  C_\Rc^{(3)}\Big(1+\ve\mathrm{d}_{\mathrm{E},{T\ve^{-1}}}\log(T\ve^{-1})\log(T\ve^{-1}(1+\ve\Ns^\ve_{T\ve^{-1}}))\Big)\mathrm{d}_{\mathrm{E},{T\ve^{-1}}}\log(T\ve^{-1})^4\\
        &\qquad+ C_\Rc^{(4)} \sqrt{\ve T \log\left(\frac{1}\delta\right)} +C_\Rc^{(5)}\left(1+\log\left(\frac{1}\delta\right)\right)\,.
        \end{align*}
    Considering only the two dominant terms and ignoring logarithmic factors we get the claimed bound.